\documentclass[final,12pt]{colt2023} 
\usepackage{bbm}
\usepackage{nicefrac}
\usepackage{graphicx}
\usepackage{comment}
\usepackage{algorithmic}
\usepackage{algorithm}
\usepackage{physics}
\newcommand{\bR}{\mathbb{R}}
\newcommand{\hSigma}{\hat{\Sigma}}
\newcommand{\polylog}{\mathsf{polylog}}
\newcommand{\bP}{\mathbb{P}}

\newcommand{\bI}{\mathbb{I}}
\newcommand{\vC}{\mathbf{C}}
\newcommand{\kmax}{K^{\max}}
\newcommand{\vx}{\mathbf{x}}
\newcommand{\tvx}{\tilde{\mathbf{x}}}

\newcommand{\vy}{\mathbf{y}}
\newcommand{\vY}{\mathbf{Y}}

\newcommand{\vM}{\mathbf{M}}
\newcommand{\vH}{\mathbf{H}}
\newcommand{\vw}{\mathbf{w}}
\newcommand{\hvw}{\mathbf{\hat{w}}}
\newcommand{\cparam}{C^{\mathsf{par}}}
\newcommand{\vv}{\mathbf{v}}
\newcommand{\ve}{\mathbf{e}}
\newcommand{\vvu}{\mathbf{u}}
\newcommand{\vf}{\mathbf{f}}
\newcommand{\hvf}{\hat{\mathbf{f}}}
\newcommand{\vg}{\mathbf{g}}

\newcommand{\bE}{\mathbb{E}}
\newcommand{\dotp}[2]{\left\langle #1, #2 \right\rangle}

\newcommand{\cP}{\mathcal{P}}
\newcommand{\Otilde}{\Tilde{O}}
\newcommand{\Omegatilde}{\Tilde{\Omega}}
\newcommand{\Thetatilde}{\Tilde{\Theta}}
\newcommand{\vX}{\mathbf{X}}
\newcommand{\vA}{\mathbf{A}}
\newcommand{\vB}{\mathbf{B}}
\newcommand{\vI}{\mathbf{I}}
\newcommand{\vD}{\mathbf{D}}
\newcommand{\iidsim}{{\overset{\mathrm{iid}}{\sim}}}
\newcommand{\vS}{\mathbf{S}}
\newcommand{\hvS}{\hat{\mathbf{S}}}

\newcommand{\ccontr}{c^{\mathsf{con}}}
\newcommand{\wstar}{\vw^{*}}
\newcommand{\what}{\widehat{\vw}}
\newcommand{\fstar}{\vf^{*}}
\newcommand{\fhat}{\widehat{\vf}}
\newcommand{\fhatspec}{\fhat_\textrm{Spec}}
\newcommand{\fhatWLS}{\fhat_\textrm{WLS}}
\newcommand{\htheta}{\hat{\theta}}
\newcommand{\cdist}{C^{\mathsf{dist}}}
\newcommand{\wols}{\hvw_{\mathsf{OLS}}}
\newcommand{\wwls}{\hvw_{\mathsf{WLS}}}
\newcommand{\wmle}{\hvw_{\mathsf{MLE}}}

\newcommand{\lossmul}{\mathcal{L}^{\mathsf{mul}}}
\newcommand{\cN}{\mathcal{N}}
\newcommand{\cG}{\mathcal{G}}
\newcommand{\fspec}{\hvf_{\mathsf{S}}}
\newcommand{\mubar}{\bar{\mu}}
\newcommand{\KL}[2]{\mathsf{KL}\left(#1\bigr|\bigr|#2\right)}
\newcommand{\Law}{\mathrm{Law}}
\newcommand{\TV}{\mathsf{TV}}

\newcommand{\ouralg}{SymbLearn}
\newcommand{\ourmultalg}{Self-SymbLearn} 
\newcommand{\vwc}[1]{\vw^{\backslash #1}}

\newcommand{\argmin}{\textrm{argmin}}
\newcommand{\iprod}[2]{\langle #1, #2 \rangle}

 \newcommand{\ad}{}

\newcommand{\draftupdate}[1]{{#1}}
\title[Near-Optimal Heteroscedastic Regression with SymbLearn]{Near Optimal Heteroscedastic Regression with Symbiotic Learning}
\usepackage{times}

\coltauthor{%
 \Name{Dheeraj Baby} \Email{dheeraj@ucsb.edu}\\
 \addr University of California, Santa Barbara
 \AND
 \Name{Aniket Das} \Email{ketd@google.com}\\
 \addr Google Research, Bangalore
 \AND
 \Name{Dheeraj Nagaraj} \Email{dheerajnagaraj@google.com}\\
 \addr Google Research, Bangalore%
 \AND
 \Name{Praneeth Netrapalli} \Email{pnetrapalli@google.com}\\
 \addr Google Research, Bangalore%
}

\begin{document}

\maketitle

\begin{abstract}%
We consider the classical problem of heteroscedastic linear regression where, given $n$ i.i.d. samples $(\vx_i, y_i)$ drawn from the model $y_i = \dotp{\wstar}{\vx_i} + \epsilon_i \cdot \dotp{\fstar}{\vx_i}, \ \vx_i \sim \cN(0, \vI), \ \epsilon_i \sim \cN(0, 1)$, our aim is to \emph{estimate the regressor} $\wstar$ \emph{without prior knowledge of the noise parameter} $\fstar$. In addition to classical applications of such models in statistics~\citep{jobson1980least}, econometrics~\citep{harvey1976estimating}, time series analysis~\citep{engle1982autoregressive} etc., it is also particularly relevant in machine learning problems where data is collected from multiple sources of varying (but apriori unknown) quality, e.g., in the training of large models~\citep{devlin2019bert} on web-scale data. In this work, we develop an algorithm called \emph{\ouralg} (short for \emph{Symb}iotic \emph{Learn}ing) which estimates $\wstar$ in squared norm upto an error of $\Otilde(\| \fstar \|^2 \cdot (\nicefrac{1}{n} + (\nicefrac{d}{n})^2))$, and prove that this rate is minimax optimal modulo logarithmic factors. This represents a substantial improvement upon the previous best known upper bound of $\Otilde(\| \fstar \|^2 \cdot \nicefrac{d}{n})$. Our algorithm is essentially an alternating minimization procedure which comprises of two key subroutines 1. An adaptation of the classical weighted least squares heuristic to estimate $\wstar$ (dating back to at least~\citet{davidian1987variance}), for which our work presents the first non-asymptotic guarantee; 2. A novel non-convex pseudogradient descent procedure for estimating $\fstar$, which draws inspiration from the phase retrieval literature. As corollaries of our analysis, we obtain fast non-asymptotic rates for two important problems, linear regression with multiplicative noise, and phase retrieval with multiplicative noise, both of which could be of independent interest. Beyond this, the proof of our lower bound, which involves a novel adaptation of LeCam’s two point method for handling infinite mutual information quantities (thereby preventing a direct application of standard techniques such as Fano’s method), could also be of broader interest for establishing lower bounds for other heteroscedastic or heavy tailed statistical problems. \\
\end{abstract}

\section{Introduction}
A popular trend in machine learning (ML) in the recent years has been the shift from training models on carefully curated datasets such as ImageNet~\citep{deng2009imagenet}, PennTreeBank~\citep{marcus1993building} etc. to training on a much larger corpus of data collected from all over the web~\citep{devlin2019bert,brown2020language}. While this has enabled training of models using much larger amounts of data, the data, so collected from all over the web, also has large variations in quality. It is clear that one should consider the quality of different data points while training the model -- giving more importance to high quality data points (i.e., less noise) and vice versa, rather than giving equal importance to all the data points. However, the quality of a data point is not apriori known and needs to be learned from the data itself. This motivates the problem of learning with \emph{heteroscedastic} noise, which finds applications in several domains such as regression \citep{davidian1987variance}, large-scale classification \citep{collier2022massively} and deep learning \citep{patrini2017making}. \\

\noindent In this work, we consider a prototypical version of this problem: linear regression with heteroscedastic noise, where the noise variance is a rank 1 quadratic function of the covariates. That is, given $n$ independently and identically distributed (i.i.d) samples $(\vx_i, y_i) \in \bR^d \times \bR$ such that $\vx_i \sim \cN(0, \vI)$, $\epsilon_i \sim \cN(0,1)$, sampled independently of $\vx_i$\footnote{Our results can be extended to any well-conditioned covariance matrix $\Sigma$ of $\vx_i$ and any zero-mean subGaussian random variable $\epsilon_i$, but for ease of exposition, we only consider identity covariance and standard normal $\epsilon_i$.}, and
\begin{align}
\vy_i = \dotp{\wstar}{\vx_i} + \epsilon_i \dotp{\fstar}{\vx_i},\label{eqn:prob}
\end{align}
Our objective is to estimate $\wstar$ \emph{without apriori knowledge of the noise model $\fstar$}. The problem of linear regression with heteroscedastic noise has been widely studied in statistics~\citep{jobson1980least,davidian1987variance,carroll2017transformation}, econometrics~\citep{goldfeld1972nonlinear,harvey1976estimating}, and timeseries analysis~\citep{engle1982autoregressive,bollerslev1986generalized,nelson1991conditional}. Most of these classical works analyze maximum likelihood or other estimators and obtain asymptotic rates of convergence for a large class of noise models. The main message of these works is that, \emph{asymptotically}, it is possible to obtain dramatically improved rates of estimation of $\wstar$ in the heteroscedastic model~\eqref{eqn:prob}, compared to the homogeneous setting where $y_i = \iprod{\wstar}{\vx_i} + \norm{\fstar} \cdot \cN(0,1)$. While there have been empirical results suggesting that the rate of estimation for heteroscedastic setting can be improved over that of homogeneous setting even for small sample sizes $n$~\citep{jobson1980least}, it has been an open problem to obtain such improved \emph{non-asymptotic} rates of estimation.
In fact, even for the interesting special case of~\eqref{eqn:prob}, where $\fstar$ is parallel to $\wstar$, used to capture covariate uncertainty \citep{harvey1976estimating,xu2000least,xu2019improving}, no better results are known compared to the homogeneous setting.

\subsection{Contributions}
The main contribution of this work is to propose an algorithm -- \ouralg~(short for Symbiotic Learning) -- that obtains an estimate $\what$ for~\eqref{eqn:prob} satisfying $\norm{\what - \wstar}^2 = \Otilde\bigr(\norm{\fstar}^2\bigr(\frac{1}{n} + \left(\frac{d}{n}\right)^2\bigr)\bigr)$, and an information theoretic lower bound of $\tilde{\Omega}\bigr(\norm{\fstar}^2\bigr(\frac{1}{n} + \left(\frac{d}{n}\right)^2\bigr)\bigr)$, which matches the upper bound up to logarithmic factors. For $n > \Omegatilde\left(d\right)$, this is a strict improvement over the estimation rate for the homogeneous setting, which is $\order{\norm{\fstar}^2\cdot \frac{d}{n}}$. An informal version of our result is presented below.
\begin{theorem}[Main Result (Informal)]
Consider any $\delta \in (0, \tfrac{1}{2})$ and let $n \geq \Omegatilde(d)$. With probability at least $1 - \delta$, the output $\hvw$ of~\ouralg~(Algorithm \ref{alg:symbol}) satisfies $\norm{\hvw - \wstar}^2 = \Otilde\left(\norm{\fstar}^2\left(\nicefrac{1}{n} + \left(\nicefrac{d}{n}\right)^2\right)\right)$ for the heteroscedastic regression problem~\eqref{eqn:prob}. The rate achieved by \ouralg~is minimax optimal up-to poly-logarithmic factors. For the same problem, the Ordinary Least Squares estimator $\wols$ exhibits a sub-optimal error rate of $\norm{\wols - \wstar}^2 = \Thetatilde(\norm{\fstar}^2\nicefrac{d}{n})$.
\end{theorem}
Conceptually, our algorithm is an iterative re-weighted least squares algorithm, where data points are weighted according to their estimated noise. This idea has a long history~\citep{jobson1980least} with impressive practical performance. To the best of our knowledge, our work presents the first non-asymptotic analysis of such an approach. Our key technical contributions are listed below: 

\paragraph{Upper bound}: We first show that the mere presence of structure in the noise as in~\eqref{eqn:prob} does not automatically improve the convergence rate of Ordinary Least Squares (OLS) estimator, which is computed as $\wols = \argmin_{\vw} 1/n \sum_{i=1}^n \left(\iprod{\vw}{\vx_i} - y_i\right)^2 = \left(\sum_{i=1}^{n} \vx_i \vx^T_i \right)^{-1}\bigr(\sum_{i=1}^{n} \vx_i y_i \bigr)$. This exhibits a rate of $\norm{\wols - \wstar}^2 = \Thetatilde\bigr(\norm{\fstar}^2\nicefrac{d}{n}\bigr)$ for this problem (Theorem~\ref{thm:ols-error-bound}). This motivates us to consider a weighted least squares (WLS) objective: $\wwls = \argmin_{\vw} 1/n \sum_{i=1}^n \alpha_i \cdot \left(\iprod{\vw}{\vx_i} - y_i\right)^2$, where the weight of an example $\alpha_i$ is inversely proportional to the variance of noise in that example i.e., $\alpha_i \propto \iprod{\fstar}{\vx_i}^{-2}$. However, this cannot be implemented since $\fstar$ is unknown. Our first step is to compute an estimate $\fhatspec$ of $\fstar$, using a standard spectral method, which guarantees that $\|\fhatspec - \fstar\|^2 = \Otilde\bigr(\norm{\fstar}^2 \cdot \frac{d}{n}\bigr)$. After selecting the weights $\alpha_i$ appropriately using $\fhatspec$, and accounting for the uncertainty in this estimate compared to $\fstar$, we first show that the resulting estimate $\wwls$ achieves a rate of $\|\wwls - \wstar\|^2 = \Otilde\bigr(\norm{\fstar}^2\bigr(\frac{1}{n} + \left(\frac{d}{n}\right)^{1.5}\bigr)\bigr)$, which is already significantly better than the rate achieved by OLS. To show this, we perform a fine-grained analysis of the design matrix of WLS, which involves establishing a strict spectral gap for certain heavy-tailed random matrices that have infinite expectation.
 This rate can be improved further by obtaining better estimators for $\fstar$ than the spectral method, and using it to come up with better weights $\alpha_i$ in WLS. Concretely, given $\wwls$ and $\fhatspec$, we hope to obtain a better estimate of $\fstar$ with $\fhatWLS = \argmin_\vf \frac{1}{n} \sum_{i=1}^n \beta_i \cdot {\left(\left(y_i- \iprod{\wwls}{\vx_i}\right)^2 - \iprod{\vf}{\vx_i}^2\right)^2}$, for an appropriately chosen $\beta_i$ depending on $\fhatspec$. We analyze a \emph{pseudogradient} descent algorithm on this objective and show that $\|\fhatWLS - \fstar\|^2 = \Otilde\bigr(\norm{\fstar}^2\bigr(\frac{1}{n} + \left(\frac{d}{n}\right)^{1.5}\bigr)\bigr)$, improving over $\fhatspec$. The overall algorithm, \ouralg, alternates between WLS estimation of $\wstar$ and pseudogradient descent estimation of $\fstar$, achieving a convergence rate of $\Otilde\bigr(\norm{\fstar}^2\bigr(\frac{1}{n} + \bigr(\frac{d}{n}\bigr)^{2}\bigr)\bigr)$ for both $\norm{\what - \wstar}^2$ and $\|\fhat - \fstar\|^2$. 

\paragraph{Lower bound}: We show that \ouralg~is near optimal by proving a minimax lower bound of $\tilde{\Omega}(\|\fstar \|^2(\nicefrac{1}{n} + \nicefrac{d^2}{n^2}))$ for $\bE\norm{\what - \wstar}^2$, which holds even when $\fstar$ is known. The $\frac{\| \fstar\|^2}{n}$ term which arises due to uncertainty of $\wstar$ in the direction parallel to $\fstar$ is straightforward to obtain. However the $\tilde{\Omega}\left(\| \fstar\|^2 (d/n)^2\right)$ term arising due to uncertainty in directions perpendicular to $\fstar$, is challenging to obtain.
The key obstacle is that that for any two instances of the heteroscedastic regression model with regressors $\wstar_1$ and $\wstar_2$ and a common noise model $\fstar$, the KL divergence between them can be infinite, precluding the direct application of standard techniques such as Fano's method or Assouad's lemma \citep{Tsybakov2009}.
Via a refined version of Assoud's lemma, we exploit the symmetry in Gaussian random variables to consider lower bounds for fixed designs (i.e, the covariates $\vx_{1},\dots,\vx_n$ are fixed). We then reduce this to the problem of obtaining lower bounds for `typical' fixed designs, which is then solved via an intricate covariate-bucketing argument which bounds certain heavy tailed random variables whose expectation is infinite. Our methods might be useful in establishing lower bounds in other statistical estimation problems with heteroscedastic or heavy-tailed noise. \\

\noindent As a byproduct of our analysis, we also obtain improved rates of estimation for both linear regression as well as phase retrieval with \emph{multiplicative noise} as we describe below. 

\paragraph{Linear Regression with Multiplicative Noise}: A special case of the heteroscedastic regression problem, when $\wstar =  \fstar$, corresponds to linear regression with multiplicative noise. The task is to estimate $\wstar$ given $n$ i.i.d samples $(\vx_i, y_i) \in \bR^d \times \bR$ such that $\vx_i \sim \cN(0, \vI)$ and $\vy_i = \dotp{\wstar}{\vx_i}(1 + \epsilon_i)$, with $\epsilon_i \sim \cN(0,1)$ drawn independently of $\vx_i$. While this is a classical model for regression with covariate uncertainty \citep{xu2000least}, but the best known rate of estimation in the literature is still $\Thetatilde(\norm{\wstar}^2\nicefrac{d}{n})$. While \ouralg~improves this to $\tilde{O}(\|\wstar\|^2(1/n + d^2/n^2))$, we also design a simpler algorithm, called \ourmultalg, which achieves this improved rate.
 \ourmultalg~is similar to Iteratively Reweighted Least Squares (IRLS)~\citep{govind-irls}, and our analysis constitutes the first \emph{nonasymptotic} analysis of such an algorithm for this problem. 
    
\paragraph{Phase Retrieval with Multiplicative Noise} The estimation of $\fstar$ (up-to a sign) is an important sub-routine in \ouralg. This problem can be reduced to estimating $\fstar$ with data of the form $(\vx_i, \epsilon_i^2 (\langle \fstar, \vx_i \rangle)^2 + \delta_i)$, where $\epsilon_i \sim \mathcal{N}(0,1)$ and $\abs{\delta_i} \leq \delta$. This corresponds to phase retrieval with multiplicative noise and a (small) additive adversarial error. Phase retrieval is a well-studied non-convex optimization problem encountered in physical sciences \citep{candes2015phase,shechtman2015phase} and is typically studied without noise (see \cite{chen2019gradient} and references therein). While additive noise has been considered in this setting~\citep{cai2016optimal}, multiplicative noise has not been studied in the literature to the best of our knowledge. This could help model covariate uncertainty in this setting as explained in~\cite{xu2000least}.

\subsection{Related Work}
Our work is most closely related to~\cite{anava2016heteroscedastic,chaudhuri2017active}, which obtain non-asymptotic rates of estimation of $\wstar$ in the heteroscedastic regression model~\eqref{eqn:prob}.
While~\cite{anava2016heteroscedastic} considers this problem through the lens of online learning,~\cite{chaudhuri2017active} considers the active learning setting. In the offline setting considered in this paper, neither of these works improve upon the OLS rate of $\norm{\widehat{\vw}-\wstar}^2 \leq \mathcal{O}\bigr(\norm{\fstar}^2\cdot \nicefrac{d}{n}\bigr)$ \\

\noindent Our work is an instance of \emph{statistical learning with a nuisance component}, with the regressor $\wstar$ being the \emph{target parameter} and the noise model $\fstar$ being a \emph{nuisance parameter}.
 Thus, our approach bears some resemblance to well-established methods for this problem, such as Double/Debiased Machine Learning \citep{cherno2018a, cherno2018b, fingerhut2022coordinated} and Orthogonal Statistical Learning (or OSL) \citep{mackey2018orthogonal, foster2019orthogonal, liu2022orthogonal}, where an estimate of the nuisance parameter can improve the estimation of the target parameter. These methods partition the data into disjoint subsets (sample splitting), using the first subset to estimate the nuisance parameter, and then estimating the target parameter using the nuisance estimate and the second subset. Similarly, \ouralg~uses separate subsets for estimating $\wstar$ and $\fstar$ in each iteration. However, our approach is cyclic and iterative, where \emph{a rough initial estimate of the target parameter can lead to improved nuisance parameter estimation, which in turn allows refined estimation of the target parameter and so on}. This leads to our alternating minimization-based approach where each step alternates between $\wstar$ estimation and $\fstar$ estimation. Typical realizations of OSL do not use any apriori target parameter estimate for nuisance estimation and is not iterative and cyclic (e.g. Meta-Algorithm 1 of \cite{foster2019orthogonal} estimates the nuisance as $\hat{g} = \textrm{Alg}(\mathcal{G}, S_1)$ and target as $\hat{\theta} = \textrm{Alg}(\mathcal{G}, S_2; \hat{g})$). In our case, such a single stage estimation achieves a sub-optimal rate of $\Otilde(\|\fstar\|^2(1/n+(d/n)^{1.5}))$ (see discussion after Theorem~\ref{thm:wls-error-bound}). Moreover, the OSL meta-algorithm treats $\textrm{Alg}(\mathcal{G}, S_1)$ and $\textrm{Alg}(\mathcal{G}, S_2; \hat{g})$ as blackbox subroutines with certain high-probability convergence guarantees, whereas we design such and analyze such subroutines for heteroscedastic regression.

\subsection{Organization}
In Section~\ref{sec:setup}, we present the problem setting and preliminaries. In Section~\ref{sec:partial}, we present algorithms to solve partial versions of the heteroscedastic problem, with additional information. By using these algorithms as subroutines, we derive our main algorithm~\ouralg~for the full heteroscedastic problem in Section~\ref{subsec:alg}. We present our main results for this algorithm as well as a matching lower bound in Section~\ref{sec:results} with a high level proof idea in Section~\ref{sec:proof-ideas}. We present some experimental results in Section~\ref{sec:exp} and conclude in Section~\ref{sec:conc}.

\subsection{Notation}
The boldface lower letters (e.g. $\vx$) represent vectors in $\bR^d$ and boldface capital letters (e.g. $\vA$) represent matrices in $\bR^{m \times n}$. $\vA_i$  denotes $i^{\textrm{th}}$ row of matrix $\vA$ and $x_j$ denotes  $j^{\textrm{th}}$ element of vector $\vx$.  For an indexed vector $\vx_i$; we use $x_{i,j}$ to denote the $j^{\textrm{th}}$ element of  $\vx_i$. We use $[n]$ to denote the set $\{1,2,\ldots,n\}$. The $\ell_{p}$ norm of a vector $\vv$ is denoted using $\norm{\vv}_{p}$. For any matrix $\vA$, $\norm{\vA}_2$ and $\norm{\vA}_F$ denote the spectral and Frobenius norms of $\vA$ respectively. We let $\vI$ denote the identity matrix, whose dimension is clear from the context. Unless otherwise specified, $\norm{\vv} = \norm{\vv}_2$ for any vector $\vv$ and $\norm{\vA} = \norm{\vA}_2$ for any matrix $\vA$. We use the $O$ notation to characterize the dependence of our error bounds on the number of samples $n$, the covariate dimension $d$ and the confidence level $\delta$, suppressing numerical constants. The $\Otilde$ notation suppresses polylogarithmic factors in $n, d$ and $\nicefrac{1}{\delta}$. We also assume that $ \delta \leq 1/2$ wherever it appears so we can write $\log(C/\delta) \leq C\log(1/\delta)$.  By $\polylog(x)$ we refer to some fixed poly-logarithmic function evaluated at $x$ which can be different when invoked in different bounds. We also hide $\polylog(dn/\delta)$ terms within the $\tilde{O}$ notation.

\section{Problem Formulation and Preliminaries}
\label{sec:setup}
Our data set consists of i.i.d samples $(\vx_i,y_i) \in \mathbb{R}^{d}\times\mathbb{R}$ of $i \in [n]$. We assume that $\vx_i \sim \cN(0,\vI)$ and $y_i = \langle \wstar,\vx_i \rangle + \epsilon_i \langle \fstar, \vx_i\rangle$ where $\epsilon_i \sim \cN(0,1)$ and independent of $\vx_i$, and $\wstar,\fstar \in \mathbb{R}^d$ are unknown vectors. Our task is to estimate the regressor $\wstar$ without prior knowledge of $\fstar$. For the rest of this paper, we assume that the data is sampled from this model.
\begin{remark}
 We can relax our assumptions to allow sub-Gaussian $\epsilon_i$. We can also consider $\vx_i \sim \cN(0,\Sigma)$ and `whiten' it by estimating $\hat{\Sigma} \approx \Sigma$ and considering $\hat{\Sigma}^{-1/2}\vx_i$ as the co-variates. We do not consider these scenarios for the sake of simplicity. 
\end{remark}
\subsection{Estimation of $\wstar$ using ordinary least squares (OLS)}
\label{sec:lstsq}
\begin{minipage}{ 0.45\linewidth}
\begin{algorithm}[H]
\caption{\texttt{OLS} } \label{alg:ols}
\vspace{0.25cm}
\textbf{Input}: $(\vx_1,y_1),\dots,(\vx_{n},y_n) \in \bR^d \times \bR$.
\vspace{0.20cm}
\begin{algorithmic}[1] 
\STATE
$\wols = \left(\sum_{i=1}^{n} \vx_i \vx^T_i \right)^{-1}\left( \sum_{i=1}^{n} \vx_i y_i \right) $
\STATE Output $\wols$
\end{algorithmic}
\end{algorithm}
\end{minipage}
\hfill
\begin{minipage}{0.50\textwidth}
 \begin{algorithm}[H] 
\caption{\texttt{Spectral Method} } \label{alg:spectral}
\textbf{Input}:  $\hvw \in \bR^d$, $(\vx_i, y_i)_{i=1}^{n} \in \bR^d\times \bR $
\begin{algorithmic}[1] 
\STATE 
$\hvS = \frac{1}{n} \sum_{i=1}^{n} \left(y_i - \dotp{\hvw}{\vx_i} \right)^2 \vx_i \vx^T_i$
\STATE $\vvu = \text{Top Eigenvector}(\hvS)$, $R = \sqrt{\frac{\|\hvS\|}{3}}$
\STATE Output $R \cdot \vvu$.
\end{algorithmic}
\end{algorithm}
\end{minipage}
\vspace{0.1cm}

A classical approach for estimating $\wstar$ is through ordinary least squares (OLS) (pseudocode in Algorithm~\ref{alg:ols}).
The following result, proved in Appendix \ref{proof:ols-error-proof}, gives the convergence rate of OLS.
\begin{theorem}[Ordinary Least Squares(OLS)]
\label{thm:ols-error-bound}
For any $\delta \in (0,\tfrac{1}{2})$, and $n \geq d\polylog(\tfrac{d}{\delta})$, the OLS estimator for the heteroscedastic linear regression problem satisfies $\norm{\wols - \wstar}^2 \leq \frac{d\|\fstar\|^2}{n} \polylog(\tfrac{nd}{\delta})$ with probability at least $1 - \delta$. Furthermore, $\norm{\wols - \wstar}^2 = \Theta(\nicefrac{d\|\fstar\|^2}{n})$ with probability at least $1 - d/n^c$, $c \geq 1$.
\end{theorem}
\subsection{Estimation of $\fstar$ using spectral method}
\label{sec:spectral}

A standard approach to estimate $\fstar$, given an estimate $\hvw \approx \wstar$, is through the spectral method~\citep{chaudhuri2017active,chen2021spectral}, originally proposed in the context of phase retrieval~\citep{netrapalli2013phase}. A pseudocode is presented in Algorithm~\ref{alg:spectral}. The following theorem, which is proved in Appendix~\ref{sec:spec-method-proof}, gives a performance guarantee for the spectral method.
\begin{theorem}[Spectral Method]
\label{thm:spec-error-bound} 
Consider any $\delta \in (0, \nicefrac{1}{2})$. Suppose we have $\hvw$ satisfying $\norm{\hvw - \wstar}^2 \leq \epsilon$. Then, for $n \geq d \polylog(\tfrac{d}{\delta}) $, the output $\fspec$ of the spectral method satisfies:
\begin{align*}
    \|\fspec - \fstar\|^2 \leq \left( \left(\norm{\fstar}^2 + \epsilon\right)\frac{d}{n} + \frac{\epsilon^2}{\norm{\fstar}^2} \right)\polylog(\tfrac{nd}{\delta}).
\end{align*}
\end{theorem}
\textbf{Remark}: Note that $\fstar$ can be recovered only up to a sign i.e., both $\pm \fstar$ are valid solutions. For any estimate $\fhat$, by $\|\fhat-\fstar\|$ we mean $\min\left(\|\fhat - \fstar\|, \|\fhat + \fstar\|\right)$. Theorem \ref{thm:ols-error-bound} and \ref{thm:spec-error-bound} imply that running the spectral method with $\hvw = \wols$ gives us an estimate $\fspec$ that satisfies $\|\fspec - \fstar\|^2 \leq \Otilde(\| \fstar\|^2 \cdot \nicefrac{d}{n})$.
\section{Estimating with Partial Information}
In this section we discuss procedures to estimate $\wstar$ and $\fstar$ when partial information is known. We then combine these procedures to develop our Algorithm \ouralg~in Section~\ref{subsec:alg}. 
\label{sec:partial}
\subsection{Estimate $\wstar$ when $\fstar$ is approximately known}
 \label{sec:wls}
 \begin{algorithm}[H] 
\caption{\texttt{Weighted Least Squares (WLS)} } \label{alg:wls}
\textbf{Input}: $(\vx_i, y_i) \in \bR^d \times \bR$ for $i \in [n]$, noise model $\hvf \in \bR^d$, reg. parameter $\lambda$
\begin{algorithmic}[1] 
\STATE Output  $\hvw_{\hvf, \lambda}$ given by
$\hvw_{\hvf, \lambda} = \bigr(\sum_{i=1}^{n} \frac{\vx_i \vx^T_i}{{\dotp{\hvf}{\vx_i}}^2 + \lambda }\bigr)^{-1} \biggr(\sum_{i=1}^{n} \frac{\vx_i y_i}{{\dotp{\hvf}{\vx_i}}^2 + \lambda}\biggr)$.
\end{algorithmic}
\end{algorithm}
 Suppose $\hat{\vf}$ is known such that $\hat{\vf} \approx \fstar$. Then, the weighted least squares (WLS) estimator $\hvw_{\hvf, \lambda}$ is given in Algorithm~\ref{alg:wls},
where $\lambda \geq 0$ is a regularizer. The intuition for the weights $\bigr({\langle\hvf,\vx_i\rangle^2+\lambda}\bigr)^{-1}$ is that the variance of the observation $y_i$ conditioned on $\vx_i$ is $\langle\fstar,\vx_i\rangle^2$. The role of $\lambda$ is to ensure the weights don't become arbitrarily large when $\langle\fstar,\vx_i\rangle^2$ is very small as well as to account for the approximation error $\|\hvf - \fstar\|$. This is made precise in the following theorem.
\begin{theorem}[Weighted Least Squares (WLS)]
\label{thm:wls-error-bound} 
Consider any $\delta \in (0, \tfrac{1}{2})$.
Suppose we know $\hvf$ such that $\norm{\hvf - \fstar}^2 \leq \epsilon$. If $\lambda$ is chosen such that $\lambda \geq \max \{ \epsilon,  \nicefrac{\norm{\hvf}^2d^2}{n^2} \}\polylog(\tfrac{nd}{\delta})$, the weighted least squares estimator $\hvw_{\hvf, \lambda}$ satisfies the following with probability at least $1 - \delta$:
$$\bigr\|\hvw_{\hvf, \lambda} - \wstar\bigr\|^2 \leq \biggr[\frac{\norm{\fstar}^2}{n} + \frac{d \norm{\fstar} \sqrt{\lambda}}{n} + \frac{\epsilon}{n} + \frac{d \sqrt{\epsilon \lambda}}{n}\biggr]\polylog(\tfrac{nd}{\delta}).$$
\end{theorem}
The WLS estimator obtains a better estimate of $\wstar$ than OLS using an estimate $\hvf$ of $\fstar$. In particular, for $\epsilon = 0$ (i.e., $\hvf = \fstar$), WLS attains a rate of $\tilde{O}( \| \fstar\|^2 (\frac{1}{n}+ \frac{d^2}{n^2}))$. This is in contrast to OLS, which does not utilize the knowledge of $\fstar$, and consequently achieves a suboptimal rate of $\Otilde\bigr(\norm{\fstar}^2 \cdot \frac{d}{n}\bigr)$. We refer to Appendix~\ref{subsec:wls_pf} for a proof of Theorem~\ref{thm:wls-error-bound}. 

\draftupdate{\noindent \paragraph{Obtaining $\Otilde(\|\fstar\|^2(\tfrac{1}{n} + (\tfrac{d}{n})^{1.5}))$ Rates} From Theorems \ref{thm:ols-error-bound} and \ref{thm:spec-error-bound}, we know  that the output of the spectral method $\fspec$ (computed using an OLS estimate) satisfies  $\|\fspec - \fstar\|^2 = \Otilde(\|\fstar\|^2 d/n)$. Moreover, from Theorem \ref{thm:wls-error-bound} we note that the WLS estimate computed using $\fspec$ $\hvw = \hvw_{\fspec, \lambda}$ (where $\lambda = \Thetatilde(\|\fspec\|d/n)$) satisfies $\|\hvw - \wstar\|^2 = \Otilde(\|\fstar\|^2(\tfrac{1}{n} + (\tfrac{d}{n})^{1.5}))$. This rate is an improvement upon the OLS estimate but strictly worse than the minimax optimal rate attained by \ouralg~. We show in Section \ref{sec:exp}, the empirical performance of this strategy is also worse than that of \ouralg.}

\subsection{Estimate $\fstar$ when $\wstar$ is approximately known}
\label{sec:pr}
\begin{algorithm}[ht] 
\caption{\texttt{Phase Retrieval with Multiplicative Noise} } \label{alg:phase_retrieve}
\textbf{Input}: samples $(\vx_1, y_1), \dots, (\vx_n, y_n)$. Estimates $\hvw, \hvf$, relaxation parameter $\mubar$, step sizes $\alpha_0, \alpha_1$ and number of steps $K$. 
\begin{algorithmic}[1] 
\STATE Divide the data into $K$ batches of size $m = \left\lfloor{\nicefrac{n}{K}}\right\rfloor$ as $\left \{ (\vx^{(t)}_1, y^{(t)}_1), \dots, (\vx^{(t)}_m, y^{(t)}_m)   \right \}_{t=1}^{K}$
\STATE Set $\hvf_0 = \hvf$,  matrix valued step size $\vD = \alpha_0 \frac{\hvf \hvf^T}{\norm{\hvf}^2} + \alpha_1 \left(\vI - \frac{\hvf \hvf^T}{\norm{\hvf}^2} \right)$
\FOR{$t \in \{0,\ldots,K-1\}$}
    \STATE Set pseudo stochastic gradient $\vg_t$ as
    \small
        \begin{align*}
         \vg_t \leftarrow \frac{1}{m} \sum_{i=1}^{m} \langle\hvf,\vx^{(t)}_i\rangle\vx^{(t)}_i\mathbbm{1}\bigr(\bigr|\langle\hvf,\vx^{(t)}_i\rangle\bigr| \geq \mubar \bigr) \frac{{\langle\vf_k,\vx^{(t)}_i}\rangle^2 - \bigr(y^{(t)}_i - {\langle\hvw,\vx^{(t)}_i}\rangle\bigr)^2}{{\langle\hvf,\vx^{(t)}_i}\rangle^4}
        \end{align*}
        \normalsize
    \STATE  Perform the pseudo-SGD update $\vf_{t+1} = \vf_t - \vD \vg_t$
\ENDFOR
\STATE Output $\vf_K$
\end{algorithmic}
\end{algorithm}
\noindent We now wish to refine our estimate of $\fstar$ using estimates $\what$ and $\fhat$ of $\wstar$ and $\fstar$ respectively. For this, we design a weighted phase retrieval algorithm which adapts to the quality of $\what$ and $\fhat$. Suppose we are given $\what = \wstar$, then we observe that $y_i - \iprod{\what}{\vx_i} = \epsilon_i \cdot \iprod{\fstar}{\vx_i}$. Since $\epsilon_i$ is a zero-mean symmetric random variable, this is equivalent to observing $(y_i - \langle \wstar,\vx_i\rangle)^2 = \epsilon_i^2\langle\fstar,\vx_i\rangle^2 = \left(1 + \zeta_i \right)\langle\fstar,\vx_i\rangle^2$, where $\zeta_i = \epsilon_i^2-1$ is a zero mean subexponential random variable. Thus, estimating $\pm \fstar$ given $(\vx_i,\left(1 + \zeta_i \right)\langle\fstar,\vx_i\rangle^2)$ is phase retrieval with multiplicative noise.

\noindent When expressed as $\epsilon_i^2\langle \fstar (\fstar)^{\intercal},\vx_i(\vx_i)^{\intercal}\rangle_{\mathsf{HS}}$ ( $\iprod{\cdot}{\cdot}_\mathsf{HS}$ is the Hilbert-Schmidt inner product), this reduces to linear regression with multiplicative noise in the space of rank-1 matrices. Thus, we consider estimating $\fstar$ by (roughly) solving for the minimizer of the weighted loss:
\begin{equation}\label{eq:weighted_sq_loss_phase}
\lossmul(\vf) = \frac{1}{n}\sum_{i=1}^{n}\frac{((y_i - \langle\hvw,\vx_i\rangle)^2-\langle \vf,\vx_i\rangle^2)^2}{\langle \hat{\vf},\vx_i\rangle^4}\end{equation}

\noindent The term $\langle \hvf,\vx_i\rangle^4$ appears in the denominator since $(y_i - \langle\hvw,\vx_i\rangle)^2$ has a variance of the order $\langle \fstar,\vx_i\rangle^4$ conditioned on $\vx_i$ and we re-weight just like in Algorithm~\ref{alg:wls}. To this end, we design Algorithm~\ref{alg:phase_retrieve} to minimize $\lossmul(\vf)$ via an approximate gradient descent procedure
The following theorem presents the performance guarantee for Algorithm~\ref{alg:phase_retrieve}.
Its proof is presented in Appendix~\ref{sec:phs_rtr_pf}.
\begin{theorem}[Phase retrieval with multiplicative noise]
\label{thm:phase_r_mult}  
In Algorithm~\ref{alg:phase_retrieve},
Assume $\bar{\mu} < \|\hat{\vf}\|$, let $r := \bigr\lceil \log(\tfrac{\|\hat{\vf}\|_2}{\bar{\mu}})\bigr\rceil$, $\Delta_{\hvw} := \wstar - \hat{\vw}$, $\Delta_{\hvf}:= \fstar - \hat{\vf}$ and $\delta \in (0,1/2)$. Let the step sizes be $\alpha_0 = c \|\hvf\|^2$ and $\alpha_1 = c\bar{\mu}\|\hvf\|$ for some small enough constant $c$. Suppose for large enough constant $\cparam$:
\begin{equation*}
\bar{\mu} > \cparam\max(\|\Delta_{\hvf}\|,\|\Delta_{\hvw}\|)\log \tfrac{m}{\delta} ;\quad
 m \geq \cparam \max\bigr(\tfrac{\|\hat{\vf}\|}{\bar{\mu}}\log(\tfrac{r}{\delta}),\tfrac{\|\hat{\vf}\|d\log^4(\tfrac{m}{\delta})}{\bar{\mu}}, \tfrac{\|\hat{\vf}\|^2\log^4(\tfrac{m}{\delta})}{\bar{\mu}^2}\bigr).
\end{equation*}
Then, the output $\vf_K$ of K steps of Algorithm~\ref{alg:phase_retrieve} satisfies the following with probability at least $1 - K\delta$:
\begin{equation*}
{\|\vf_{K}-\fstar\| \leq \exp(-\gamma K)\|\Delta_{\hvf}\|} + C\left[\frac{\|\hat{\vf}\|\log(\tfrac{1}{\delta})}{\sqrt{m}} + \frac{\|\Delta_{\hvw}\|^2}{\bar{\mu}} + \log(\tfrac{d}{\delta})\sqrt{\frac{d\bar{\mu}\|\hat{\vf}\|}{m}} \right].
\end{equation*}
{where $\gamma > 0$} is a universal constant. In particular, setting {$K =\Theta(\log(m))$  and $\mubar = \Thetatilde\left(\max\{\|\Delta_{\hvf}\|, \|\Delta_{\hvw}\|\}\right)$}  we have:
\begin{align*}
    {\norm{\vf_K - \fstar}^2} &\leq \Otilde\biggr(\frac{\norm{\fstar}^2}{m} + \|\Delta_{\hvw}\|^2 + \|\Delta_{\hvf}\|^2 \cdot \frac{d }{m}  + \frac{d}{m}\norm{\fstar} \cdot \left( \|\Delta_{\hvw}\| + \|\Delta_{\hvf}\| \right)\biggr)
\end{align*}
\end{theorem}

\paragraph{Quality Adaptivity of Noisy Phase Retrieval} If $\hvf$ is the spectral method estimate of $\fstar$, and $\hvw = \hvw_{\hvf, \lambda}$ (output of WLS) for appropriately chosen $\lambda = \Tilde{\Theta}(\nicefrac{\norm{\hvf}^2 d}{n})$, then $\norm{\hvw - \wstar}^2 = \Otilde\bigr(\| \fstar\|^2 \bigr(\nicefrac{1}{n} + \left(\nicefrac{d}{n}\right)^{\nicefrac{3}{2}}\bigr)\bigr)$ is lower order compared to $\|\hvf - \fstar\|^2 = \Otilde\left(\| \fstar\|^2 \cdot \nicefrac{d}{n}\right)$. In this case, $\|\vf_K - \fstar\|^2 \leq \Otilde(\norm{\hvw - \wstar}^2)$. This \emph{transfers the error in estimation of $\wstar$ to the error in estimation of $\fstar$}. That is, a better estimate of $\wstar$ leads to a better estimate of $\fstar$, which in turn improves $\wstar$ estimation via WLS. Having developed quality-adaptive algorithms for estimating both $\wstar$ and $\fstar$, the \ouralg~algorithm naturally follows by iteratively alternating between the two. 

\section{Algorithm}
\label{subsec:alg}
Combining our observations in Section~\ref{sec:partial} we present the main algorithm of our work. In the general case, $\fstar$ is arbitrary and unknown. 
For the sake of simplicity, we will assume that we can divide the dataset into $2K$ disjoint parts. We illustrate our algorithm in Figure~\ref{fig:algo}
\begin{algorithm}
\caption{\texttt{\ouralg} } \label{alg:symbol}
\textbf{Require}: $(\vx_1,y_1),\dots,(\vx_{n},y_n) \in \bR^d \times \bR$.  Number of steps $K$. \draftupdate{Number of phase retrieval steps $K_p$}. Weights $\lambda_1, \dots, \lambda_K$, $\mubar_1,\ldots,\mubar_K$ and step sizes $(\alpha_0^1,\alpha_1^1),\ldots,(\alpha_0^K,\alpha_1^K)$.
\begin{algorithmic}[1] 
\STATE Divide the data into $2K$ disjoint parts
\STATE Compute $\hat{\vw}^{(1)}$ with OLS (Algorithm \ref{alg:ols}) and $1^{\text{st}}$ data partition
\STATE Compute $\hat{\vf}^{(1)}$ using the spectral method with input $\hat{\vw}^{(1)}$ (Algorithm \ref{alg:spectral}) and $2^{\text{nd}}$ data partition
\FOR{$k \in \{1,\ldots,K-1 \}$}
    \STATE Compute $\hat{\vw}^{(k+1)}$ using WLS (Algorithm \ref{alg:wls}) with input parameters $\lambda_k$ and $\hat{\vf}^{k}$ and data from $(2k+1)^{\text{th}}$ partition
    \STATE Compute $\hat{\vf}^{(k+1)}$ using Phase Retrieval (Algorithm \ref{alg:phase_retrieve}) \draftupdate{for $K_p$ steps} with inputs  $(\hat{\vw}^{(k)},\hat{\vf}^{(k)},\bar{\mu}_k,\alpha_0^k,\alpha_1^k)$ and data from $(2k+2)^{\text{th}}$ partition
\ENDFOR
\STATE Output $\hat{\vw}^{K}$.
\end{algorithmic}
\end{algorithm}

\begin{figure}[h]
\centering
\includegraphics[width=0.9\linewidth]{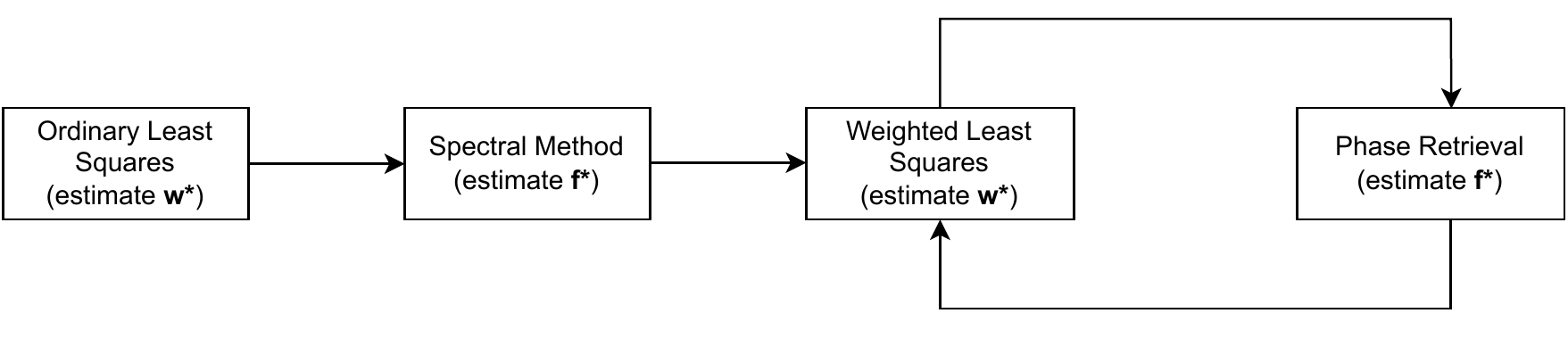}
\caption{An Illustration of \ouralg}
\label{fig:algo}
\end{figure}

\noindent \textbf{Regression with Multiplicative Noise} When $\wstar =   \pm \fstar$, the problem coincides with linear regression with multiplicative noise. For this special case, we design a much simpler algorithm (\ourmultalg~, see Appendix~\ref{app:ourmultalg}), which given an estimate of $\hvw \approx \wstar$, iteratively applies the weighted least squares procedure (Algorithm~\ref{alg:wls}) to get a better estimate with $\hvf = \hvw $.

\section{Main Results}
\label{sec:results}
Our first main result stated below establishes the estimation guarantee for $\wstar$ via \ouralg. The proof of this theorem is presented in Appendix~\ref{proof:yinyang-full-proof}.
\begin{theorem}[\ouralg]
\label{thm:yinyang-full-proof}
Consider any $\delta \in (0, \nicefrac{1}{2})$. Let {$S_k = \sum_{j=0}^{k} \nicefrac{1}{2^j}$}. Then, for $n \geq d\polylog(d/\delta)$, $K = \lceil\log_2(n)\rceil$ and \draftupdate{$K_p = \Theta(\log(n/K))$}, if we run \ouralg~(Algorithm~\ref{alg:symbol}) with $\mubar_{k} = \|\hvf_k\|$ $\cdot \sqrt{\max(\tfrac{k}{m},\tfrac{(k)d^2}{m^2}) + \left(\nicefrac{d}{m}\right)^{S_k}}\polylog(\tfrac{nd}{\delta})$, $\lambda_{k} =  \|\hvf_k\|^2 \left(\max(\tfrac{k}{m},\tfrac{(k)d^2}{m^2}) + \left(\nicefrac{d}{m}\right)^{S_k}\right)\polylog(\tfrac{nd}{\delta})$
then the output $\hvw_K$ satisfies the following with probability at least $1 - \delta$:
$$\norm{\hvw_K - \wstar}^2 \leq \Otilde\left(\norm{\fstar}^2 \left(\nicefrac{1}{n} + \nicefrac{d^2}{n^2}\right)\right),$$
\end{theorem}
Our second result below establishes the minimax optimality of \ouralg~with a lower bound construction. Let $P_{\vw, \vf}$ denote the heteroscedastic regression model parameterized by $\vw$ and  $\vf$, i.e., $P_{\vw, \vf}$ is a probability measure supported on $\bR^d \times \bR$ such that $(\vx, y) \sim P_{\vw, \vf}$ implies $\vx \sim \cN(0, \vI)$ and $y \ | \  \vx \ \sim \cN(\dotp{\vw}{\vx}, {\dotp{\vf}{\vx}}^2)$. We refer to Appendix~\ref{sec:lb_proof} for the proof.
\begin{theorem}[Minimax Lower Bound for Heteroscedastic Regression]
\label{thm:heteroscedastic-minimax-lb}
Consider any $\fstar \in \bR^d$, and let $\hvw : \left(\bR^d \times \bR\right)^n \to \bR^d$ denote any arbitrary measurable function that estimates $\wstar$ given i.i.d samples $(\vx_i, y_i)_{i \in [n]} \ \iidsim \ P_{\wstar, \fstar}$. Furthermore, let $\mathcal{W}$ denote the family of estimators $\hvw$. Then, 
\begin{align*}
    \inf_{\hvw \in \mathcal{W}} \sup_{\wstar \in \bR^d} \bE_{(\vx_i, y_i)_{i \in [n]} \ \iidsim \ P_{\wstar, \fstar}} \left[ \norm{\hvw - \wstar}^2 \right] \geq \Omegatilde\left(\tfrac{\norm{\fstar}^2}{n} + \tfrac{\norm{\fstar}^2 d^2}{n^2}\right).
\end{align*}
\end{theorem}




\section{Key Ideas behind Proofs}
\label{sec:proof-ideas}

\subsection{Weighted Least Squares (WLS) : Theorem~\ref{thm:wls-error-bound}}
For simplicity, consider the setting where $\hvf = \fstar$ and $\lambda = 0$ in Algorithm~\ref{alg:wls}. The empirical covariance matrix $\vH = \frac{1}{n}\sum_{i=1}^{n}\frac{\vx_i\vx_i^{\intercal}}{\langle \vx_i,\fstar\rangle^2}$ has infinite expectation. We show via a straight forward calculation that the expected squared error $\bE\left[\norm{\wwls - \wstar}^2\right] = \frac{1}{n}\mathbb{E} \Tr(\vH^{-1})$. The main technical difficulty is to precisely upper bound this trace. Note that $(\fstar)^{\intercal}\vH\fstar = 1$ and for $\vvu \perp \fstar$, we have that $\langle\vx_i,\vvu\rangle$ is independent of $\langle \vx_i,\fstar\rangle$. Note that $\vvu^{\intercal}\vH\vvu = \frac{1}{n}\sum_{i=1}^{n}\frac{\langle\vvu,\vx_i\rangle^2}{\langle\fstar,\vx_i\rangle^2} $. We show that, typically there exist $O(1)$ samples $i$ such that $|\langle\vx_i,\fstar\rangle|$ has the value $\tilde{\Theta}(\frac{1}{n})$ and the corresponding value of $\langle \vvu,\vx_i\rangle$ is $\tilde{\Theta}(1)$. We show that such extreme values contribute most to $\frac{1}{n}\sum_{i=1}^{n}\frac{\langle\vvu,\vx_i\rangle^2}{\langle\fstar,\vx_i\rangle^2} $, allowing us to show that $\vvu^{\intercal}\vH\vvu = \tilde{\Theta}(n)$ followed by a covering argument to show that $\vH \succeq \norm{\fstar}^{-2} \left[\ve \ve^T + \frac{n}{d}\left(\vI - \ve \ve^T\right)\right]$ with high probability ($\ve := \nicefrac{\fstar}{\norm{\fstar}}$, the extra $1/d$ appearing due to the covering argument). This concludes the proof of Theorem~\ref{thm:wls-error-bound} in this setting. Handling $\fhat \neq \fstar$ and $\lambda > 0$ are straight forward and is described in the full proof of Theorem~\ref{thm:wls-error-bound} in Appendix~\ref{subsec:wls_pf}.
\subsection{Phase retrieval with multiplicative noise : Theorem~\ref{thm:phase_r_mult}}
Algorithm~\ref{alg:phase_retrieve} estimates $\fstar$ given prior estimates of $\fstar$ and $\wstar$ via pseudogradient descent on a carefully designed nonconvex squared loss, which draws motivation from similar approaches in the noise-free phase retrieval literature~\citep{chen2019gradient, netrapalli2013phase}. Consider the hypothetical scenario in which we have apriori access to both $\fstar$ and $\wstar$, and iid samples $(\vx_i, y_i)_{i \in [n]}$. We observe that $(y_i - \langle \wstar,\vx_i\rangle)^2 = \epsilon_i^2 \langle \fstar,\vx_i\rangle^2$, where
$\mathbb{E}\left[\epsilon_i^2 \langle \fstar,\vx_i\rangle^2|\vx_i\right] = \langle \fstar,\vx_i\rangle^2$ and $\mathsf{var}(\epsilon_i^2 \langle \fstar,\vx_i\rangle^2|\vx_i) =  3\langle \fstar,\vx_i\rangle^4$. Hence, $\fstar$ is the (expected) minimizer of the following nonconvex quartic loss function which cannot be computed from samples.
\begin{align*}
    \lossmul(\vf) = \frac{1}{m}\sum_{i=1}^{m}\frac{\left[(y_i - \langle \wstar,\vx_i\rangle)^2 - \langle \vf,\vx_i\rangle^2\right]^2}{\langle \fstar,\vx_i\rangle^4}.
\end{align*}
The choice of this loss is also motivated by the effectiveness of the WLS procedure for estimating $\wstar$. In fact, writing $\vA=\vf \vf^{\intercal}$, $\lossmul$ can be interpreted as a weighted least squares objective on the manifold of rank 1 matrices, equipped with the Hilbert Schmidt inner product:
\begin{align*}
    \lossmul(\vA) = \frac{1}{m}\sum_{i=1}^{m}\frac{\left[(y_i - \langle \wstar,\vx_i\rangle)^2 - \dotp{\vA}{\vx_i \vx^T_i}_{\mathsf{HS}}\right]^2}{{\dotp{\vA}{\vx_i \vx^{T}_i}}^{2}_{{\mathsf{HS}}}}.
\end{align*}
In a practical setting, where we only have access to estimates $\hvf$ and $\hvw$, we consider the following computable approximation of $\lossmul$:
\begin{align*}
    \bar{\mathcal{L}}(\vf) &= \frac{1}{m}\sum_{i=1}^{m}\mathbbm{1}(|\langle\hat{\vf},\vx_i\rangle| \geq \bar{\mu})\frac{\left[(y_i - \langle \hvw,\vx_i\rangle)^2 - \langle \vf,\vx_i\rangle^2\right]^2}{\langle \hvf,\vx_i\rangle^4},
\end{align*}
where the `regularization' term $\mathbbm{1}(|\langle\hat{\vf},\vx_i\rangle| \geq \bar{\mu})$ ensures that the Hessian of this loss is well-defined. The \emph{pseudogradient} $\cG$, which is an approximation of $\nabla \bar{\mathcal{L}}(\vf)$, is then defined as follows: 
$$\cG(\vf) := \frac{1}{m}\sum_{i=1}^{m}\mathbbm{1}(|\langle\hat{\vf},\vx_i\rangle| \geq \bar{\mu})\frac{\left[\langle \vf,\vx_i\rangle^2 - (y_i - \langle \hat{\vw},\vx_i\rangle)^2 \right]\langle \hat{\vf},\vx_i\rangle \vx_i}{\langle \hat{\vf},\vx_i\rangle^4}$$
We show in Appendix \ref{sec:phs_rtr_pf} that $\mathcal{G}(\vf)$ behaves like the gradient of a strongly convex function whenever $\vf$ is initialized sufficiently close to $\fstar$. As a consequence, we show that the pseudogradient descent procedure $\vf_{t+1} = \vf_t - \vD \cG(\vf_t)$, initialized with $\vf_0 = \hvf$ under appropriately chosen matrix valued step-sizes $\vD$ (where the matrix valued step-size `pre-conditions' the gradient similar to the modified Newton method), converges exponentially fast to a local neighborhood of $\fstar$.
\subsection{\ouralg~Guarantees : Theorem~\ref{thm:yinyang-full-proof}}
In order to prove Theorem~\ref{thm:yinyang-full-proof} (full proof in Appendix~\ref{proof:yinyang-full-proof}), we use the fact that the OLS (Algorithm~\ref{alg:ols}) estimate $\hvw_0$ satisfies $\norm{\hvw_0-\wstar}^2 = \Otilde(\| \fstar \|^2 \cdot \nicefrac{d}{n})$. The spectral method (Algorithm~\ref{alg:spectral}) then takes this as the input and estimates $\|\hvf_0 - \fstar\|^2 = \Otilde(\| \fstar \|^2 \cdot \nicefrac{d}{n})$. Running WLS (Algorithm~\ref{alg:wls}) with input $\hvf_0$ and appropriately chosen $\lambda_0$ gives $\|\hvw_1 - \wstar\|^2 = \tilde{O}\bigr(\tfrac{1}{n} + (\tfrac{d}{n})^{\tfrac{3}{2}}\bigr)$. The phase retrieval algorithm with input $\hvw_1$ and $\hvf_0$ then transfers this error bound to $\hvf_1$, giving $\|\hvf_1-\fstar\|^2 = \tilde{O}\bigr(\tfrac{1}{n} + (\tfrac{d}{n})^{\nicefrac{3}{2}}\bigr)$. This iterative procedure gives us $\|\hvf_K-\fstar\|^2, \|\hvw_K-\wstar\|^2 = \tilde{O}\left(\tfrac{1}{n} + (\tfrac{d}{n})^{2}\right)$ for $K = \lceil\log_2(n)\rceil$.




\subsection{Lower Bound : Theorem~\ref{thm:heteroscedastic-minimax-lb}}
The key technical challenge in the lower bound analysis is that given any $\fstar \in \bR^d$ and two instances of the heteroscedastic regression problem, $P_{\vw_1, \fstar}$ and $P_{\vw_2, \fstar}$, $\KL{P_{\vw_1, \fstar}}{P_{\vw_2, \fstar}} = \infty $ unless $\vw_1 - \vw_2$ is parallel to $\fstar$. This precludes the direct use of standard techniques such as LeCam's method or Assoud's Lemma. We first obtain a coarse lower bound by considering instances $P_{\vw_1, \fstar}$ and $P_{\vw_2, \fstar}$ such that $\vw_1 - \vw_2$ is parallel to $\fstar$. Here, $\KL{P_{\vw_1, \fstar}}{P_{\vw_2, \fstar}} = \nicefrac{\norm{\vw_1 - \vw_2}^2}{\norm{\fstar}^2}$, and a direct application of LeCam's method obtains a lower bound of $\Omega(\nicefrac{\norm{\fstar}^2}{n})$. To obtain the finer $\Omegatilde(\nicefrac{\norm{\fstar}^2 d^2}{n^2})$ lower bound, we develop a refined version of Assoud's lemma for heavy tailed random variables. First, we lower bound the minimax risk by the Bayes risk over the uniform spherical distribution supported over the orthogonal subspace of $\fstar$ as follows.
\begin{align*}
    \sup_{\vw \in \bR^d} \bE_{(\vx_i, y_i)_{i \in [n]} \ \iidsim \ P_{\wstar, \vf}}\left[\norm{\hvw - \wstar}^2\right] &\geq \bE_{\vx_{1:n}} \left[ \bE_{\vw \sim S_{\alpha}(\ve_{1:d-1})}\left[\bE_{y_{1:n} | \vx_{1:n}}\left[\norm{\hvw - \wstar}^2\right]\right]\right]
\end{align*}
Here, $S_{\alpha}(\ve_{1:d-1})$ represents the uniform spherical distribution in the subspace orthogonal to $\fstar$. We now exploit the rotational symmetry of this distribution to perform a \emph{data-dependent change of basis}, i.e., we choose basis vectors $\vvu_1, \ldots, \vvu_{d}$ such that $\vvu_d = \nicefrac{\fstar}{\norm{\fstar}}$ and $\vvu_{1:d-1}$ are functions of $\vx_{1:n}$ to be chosen later. As a consequence of the rotational symmetry, $S_{\alpha}(\ve_{1:d-1})$ is equal in distribution to $S_{\alpha}(\vvu_{1:d-1})$. Combining this data-dependent change of basis with a careful coupling argument allows us to (approximately) lower bound the Bayes risk by a fixed-design hypothesis testing problem between $\wstar$ and $\wstar + \alpha \vvu_i$, defined as follows.\small
\begin{align}
\label{eqn:fixed-design-test-lower-bound}
    \bE_{\vw \sim S_{\alpha}(\ve_{1:d-1})}\left[\bE_{y_{1:n} | \vx_{1:n}}\left[\norm{\hvw - \wstar}^2\right]\right] &\geq \frac{\alpha^2}{d}\sum_{i=1}^{d-1} \bE_{\wstar \sim S_{\alpha}(\vvu_{1:d-1})}\left[1 - \sqrt{\KL{P_{\wstar + \alpha \vvu_i, \fstar}}{P_{\wstar, \fstar}|\vx_{1:n}}}\right]
\end{align}
\normalsize
We choose the vectors $\vvu_{1:d-1}$ such that $\sum_{i=1}^{d} \sqrt{\KL{P_{\wstar + \alpha \vvu_i, \fstar}}{P_{\wstar, \fstar}|\vx_1,\dots,\vx_n}}$ is as small as possible. The key challenge here lies in the fact that $\KL{P_{\wstar + \alpha \vvu_i, \fstar}}{P_{\wstar, \fstar}|\vx_1,\dots,\vx_n}$ has infinite expectation for fixed $\vvu_i$. Thus, when $\vvu_i$ are fixed basis vectors, we obtain a vacuous lower bound since all the KL terms can become moderately large. We resolve this with a careful bucketing argument to choose $\vvu_i$ such that in the sum $\sum_{i=1}^{d} \sqrt{\KL{P_{\wstar + \alpha \vvu_i, \fstar}}{P_{\wstar, \fstar}|\vx_1,\dots,\vx_n}}$, the $\mathsf{KL}$ terms corresponding to small $i$ are very large and those corresponding $i$ close to $d$ are small. The properties of the square root ensures that this gives a better upper bound. This is achieved by partitioning the data into buckets $B_k$ defined as,
\begin{align*}
    B_k = \{j \in [n]: \dotp{\fstar}{\vx_j}^2 \in [2^{k}\gamma,2^{k+1}\gamma)\}, k \leq K
\end{align*}
where $\gamma = c/n^2$ and $K = O\left(\log(nd)\right)$. The design matrix is accordingly split as $\vH_k = \sum_{j \in B_k} \vx_j \vx^{\intercal}_j$. The basis vectors $\vvu_{1:d}$ are then chosen via Gram-Schmidt orthogonalization on the eigenvectors of $\vH_k$ with non-zero eigenvalues starting from $k =0$ until we obtain $d-1$ such vectors. Under this choice of $\vvu_{1:d}$, each term of the summation in the hypothesis testing problem~\eqref{eqn:fixed-design-test-lower-bound} can be lower bounded as $1 - \frac{\alpha n}{d} \polylog(nd)$. Plugging this into~\eqref{eqn:fixed-design-test-lower-bound} and choosing $\alpha$ appropriately gives the desired $\Omegatilde\bigr(\norm{\fstar}^2\frac{d^2}{n^2}\bigr)$ lower bound. The detailed proof is presented in Appendix~\ref{sec:lb_proof}.

\section{Experimental Results} \label{sec:exp}
\begin{figure}
\centering
\label{fig:experiment}
    \subfigure{
    \centering
    \includegraphics[width =0.21 \linewidth]{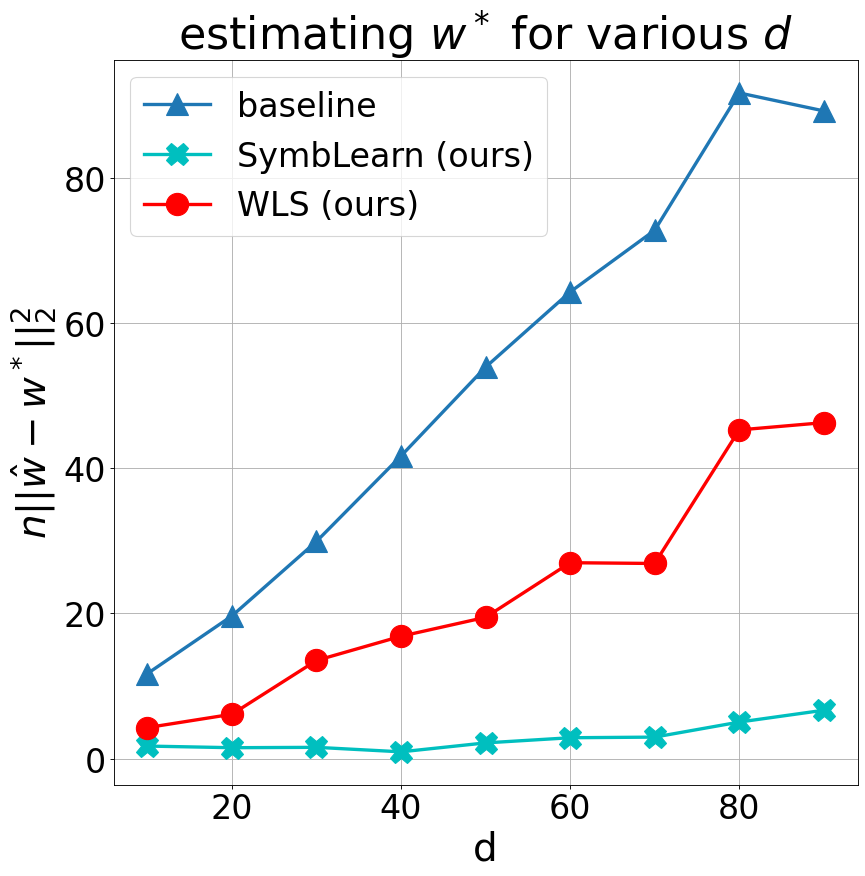}\label{fig:w_vs_d}}
    \hfill
    \subfigure{
    \centering
    \includegraphics[width =0.23 \linewidth]{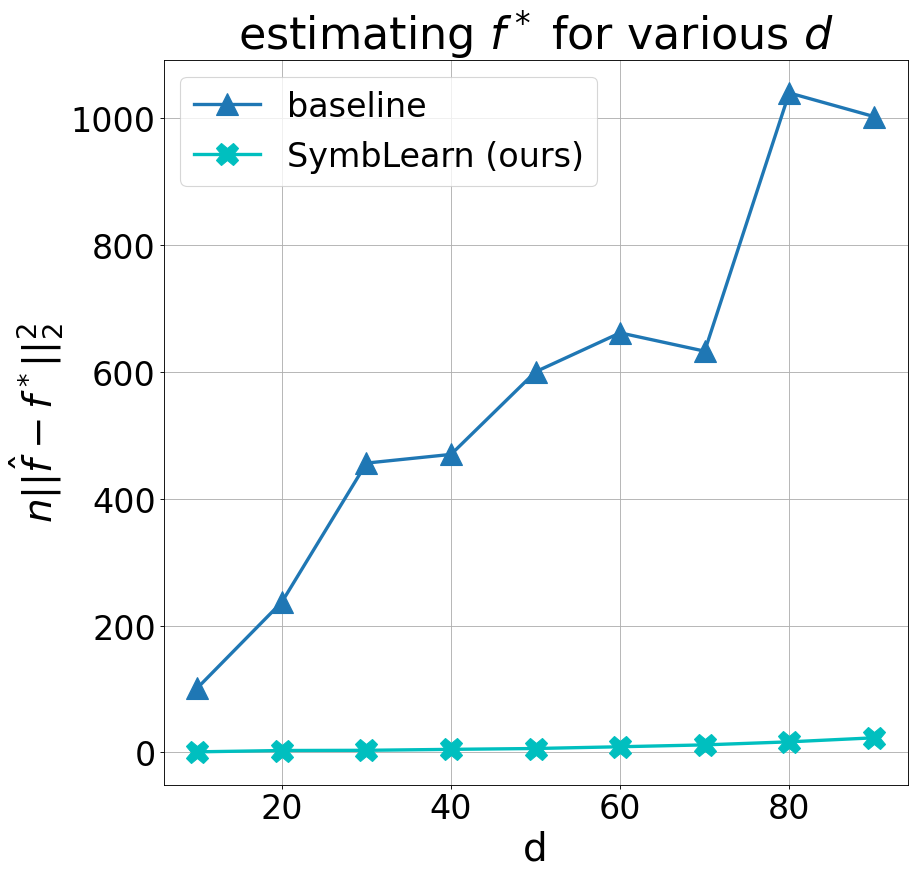}
    \label{fig:f_vs_d}
    }
    \hfill
   \subfigure{
    \centering
    \includegraphics[width =0.23\linewidth]{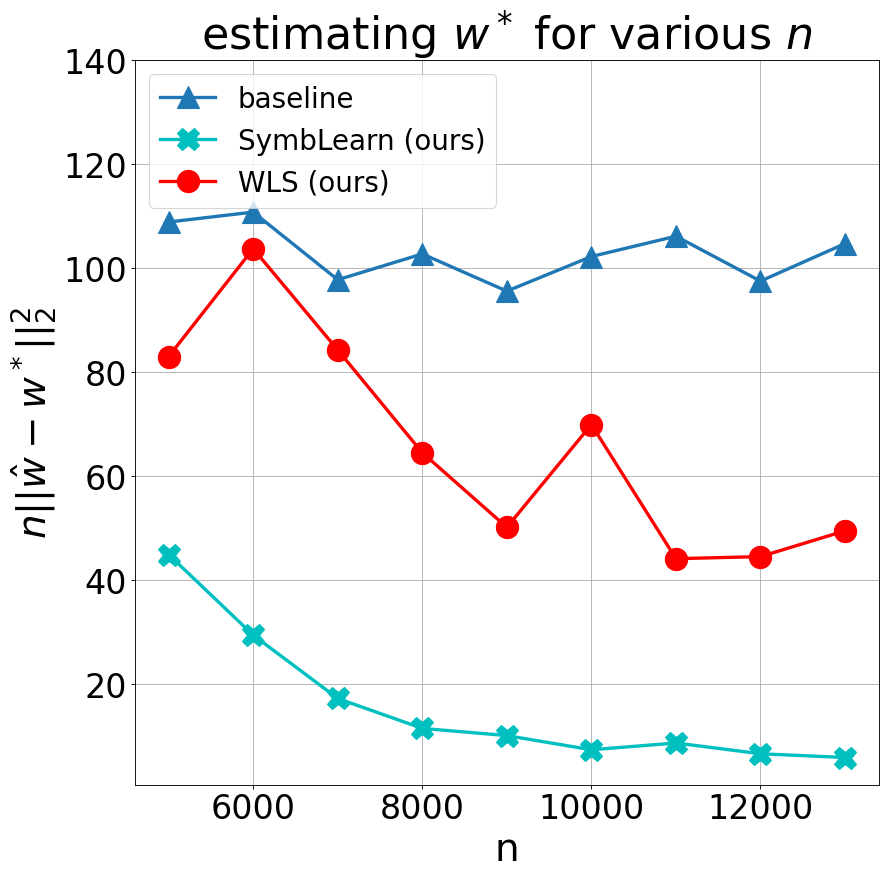} \label{fig:w_vs_n}}
    \hfill
    \subfigure{
    \centering
    \includegraphics[width =0.23 \linewidth]{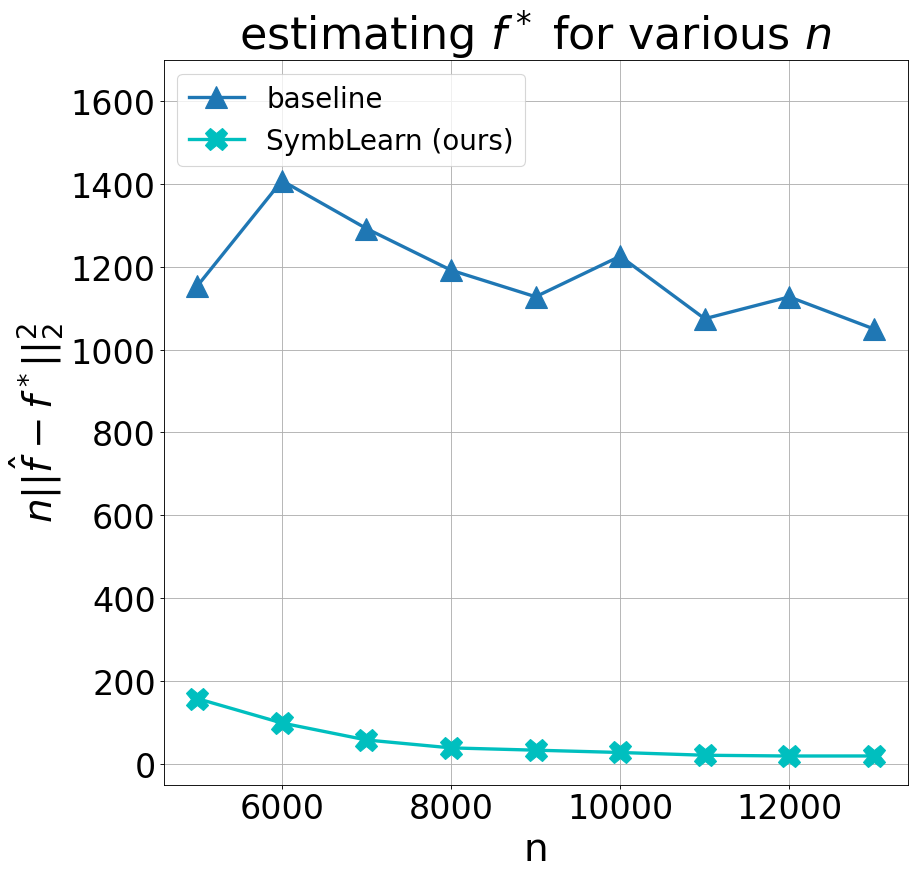} \label{fig:f_vs_n}
    }
    \caption{Estimation of $\wstar$ and $\fstar$ for various $d$ with $n = 10000$ (Figs.~\ref{fig:w_vs_d},~\ref{fig:f_vs_d}) and for various $n$ with $d = 100$ (Figs.~\ref{fig:w_vs_n},~\ref{fig:f_vs_n}). Baseline refers to OLS (Algorithm~\ref{alg:ols}) for $\wstar$ estimation and to the spectral method (Algorithm~\ref{alg:spectral}) for $\fstar$ estimation. WLS is the weighted least squares (Algorithm~\ref{alg:wls}) using spectral estimator of $\fstar$ (Algorithm~\ref{alg:spectral}). \ouralg~(Algorithm~\ref{alg:symbol}) significantly outperforms the baselines for both $\wstar$ and $\fstar$ estimation.}
\end{figure}
We generate data according to the model \eqref{eqn:prob} and estimate $\wstar$ and $\fstar$ with various algorithms and compare it to \ouralg. Figure~\ref{fig:experiment} shows the cumulative errors $n \|\hvw  - \wstar\|_2^2$ and $n\|\hvf -  \fstar\|_2^2$, averaged over $5$ trials for various algorithms. Since we can only recover $\fstar$ up to a sign, we consider the sign that is closest to the estimate $\hvf$.
For \ouralg~(Algorithm~\ref{alg:symbol}), we run multiple epochs of the weighted least squares and phase retrieval on the \emph{entire} data to get a final estimate of $\wstar$ and $\fstar$. \draftupdate{We observe that \ouralg~not only outperforms the baseline OLS estimator, but is also superior to the WLS estimator (computed using the spectral method) discussed in Section \ref{sec:wls}.} 

\section{Conclusion and Discussion}
\label{sec:conc}
In this work, we considered heteroscedastic linear regression with Gaussian design and obtained near sample optimal and computationally efficient algorithms. In future work, we plan to explore noise of the form $\epsilon_i \sigma(\langle \fstar,\vx_i\rangle)$ for more general functions $\sigma$. Developing kernel regression versions of this problem is also an interesting direction. We also hope to study the applicability of our methods in heteroscedastic versions of more challenging problems such as linear contextual bandits and reinforcement learning with linear function approximation.

\bibliography{references}
\clearpage
\appendix
\clearpage
\section{Analysis of MLE, WLS and \ourmultalg}
In this section, we present our analysis of MLE, WLS and \ourmultalg.
\subsection{Technical Lemmas}
\begin{lemma}[Gaussian Linear Combinations of Random Vectors]
\label{lem:normal-vec-norm}
Let $\vv = \sum_{i=1}^{n} \epsilon_i \vy_i$ where $\epsilon_i \iidsim \cN(0, 1)$ and $\vy_i$ are arbitrary random vectors independent of $\epsilon_i$. Then, for any $\delta \in (0, 1)$, the following holds with probability at least $1 - \delta$.
\begin{align*}
    \norm{\vv}^2 \leq \Tr(\vY)\left[1 + \max \left\{ 8\log(\nicefrac{d}{\delta}), \sqrt{8\log(\nicefrac{d}{\delta})}  \right\} \right]
\end{align*}
where $\vY = \sum_{i=1}^{n} \vy_i \vy^{T}_i$.
\end{lemma}
\begin{proof}
Let $\ve_j, \ j \in [d]$ denote the standard basis of $\bR^d$. Since $\epsilon_i \iidsim \cN(0,1)$, it follows that,
\begin{align*}
    \bE[ e^{\mu \dotp{\vv}{\ve_j}}  | \vy_{1:n}] &= \prod_{i=1}^{n} \bE\left[e^{\mu \epsilon_i \dotp{\vy_i}{\ve_j}} |  \vy_{1:n}\right] = \prod_{i=1}^{n} e^{\nicefrac{\mu^2{\dotp{\vy_i}{\ve_j}}^2}{2}} \\
    &= \exp\left( \frac{\mu^2 \ve^T_j \left( \sum_{i=1}^{n} \vy_i \vy^T_i \right) \ve_j}{2} \right) = e^{\nicefrac{\mu^2 \ve^T_j \vY \ve_j}{2}}
\end{align*}
Thus, we infer that $\dotp{\vv}{\ve_j}$ is a zero-mean Gaussian random variable conditioned on $\vy_{1:n}$. Thus, by $\chi^2$ concentration, the following holds with probability at least $1 - \nicefrac{\delta}{d}$
\begin{align*}
    \dotp{\vv}{\ve_j}^2 \leq (\ve^T_j \vY \ve_j) \left[1 + \max \left\{ 8\log(\nicefrac{d}{\delta}), \sqrt{8\log(\nicefrac{d}{\delta})}  \right\} \right]
\end{align*}
Taking a union bound, we conclude that
\begin{align*}
    \norm{\vv}^2 = \sum_{j=1}^d \dotp{\vv}{\ve_j}^2 &\leq ( \sum_{j=1}^{d}\ve^T_j \vY \ve_j) \left[1 + \max \left\{ 8\log(\nicefrac{d}{\delta}), \sqrt{8\log(\nicefrac{d}{\delta})}  \right\} \right] \\
    &\leq \Tr(\vY)\left[1 + \max \left\{ 8\log(\nicefrac{d}{\delta}), \sqrt{8\log(\nicefrac{d}{\delta})}  \right\} \right]
\end{align*}
\end{proof}
\begin{lemma}[Lower Tail Bound for Binomial Random Variables]
\label{lem:bin-rv-tail}
Let $X$ be a $\textrm{Binomial}(n, p)$ random variable. Then, 
\begin{align*}
    \bP \left \{ X \geq \nicefrac{np}{2} \right \} \geq 1 - e^{\nicefrac{-np}{8}} 
\end{align*}
\end{lemma}
\begin{proof}
Decomposing $X = \sum_{i=1}^{n} Y_i$ where $Y_i \ \iidsim \ \textrm{Bernoulli}(p)$ and observing that $Y_i$ are non-negative random variables, the result follows by applying the one-sided Bernstein inequality. 
\end{proof}\begin{lemma}[Tail Bound for Rational Functions of $\chi^2$ RVs - I]
\label{lem:chisq-rational-term1-bound}
Let $\zeta_i \ \iidsim \ \cN(0, 1)$ and $\sigma > 0$. Then, for any $\lambda > 0$,
\begin{align*}
    \bP \left \{ \sum_{i=1}^{n} \frac{\sigma^2 \zeta^2_i}{\lambda + \sigma^2 \zeta^2_i} \geq \frac{n \sigma^2 p_0}{2(\sigma^2 + \lambda)} \right \} \geq 1 - e^{\nicefrac{-np_0}{8}}
\end{align*}
where $p_0 = 1 - \sqrt{\nicefrac{2}{\pi}} \int_{0}^{1} e^{\nicefrac{-x^2}{2}} dx = \textrm{erfc}(\nicefrac{1}{\sqrt{2}})$
\end{lemma}
\begin{proof}
Let $\lambda = \beta \sigma^2$. Then, 
\begin{align*}
    \frac{\sigma^2 \zeta^2_i}{\lambda + \sigma^2 \zeta^2_i} &\geq \frac{1}{1 + \beta} \bI \left \{ \frac{\zeta^2_i}{\beta + \zeta^2_i} \geq \frac{1}{1 + \beta} \right \} = \frac{1}{1 + \beta} \bI \left \{ \zeta^2_i \geq 1 \right \},
\end{align*}
where the last equality follows from the fact that $t \rightarrow \nicefrac{t^2}{t^2 + \beta}$ is monotonic for $t \geq 0$. Hence, 
\begin{align*}
    \sum_{i=1}^{n} \frac{\zeta^2_i}{\beta + \zeta^2_i} &\geq \frac{1}{1 + \beta} \sum_{i=1}^{n} \bI \left \{ \zeta^2_i \geq 1 \right \} = \frac{1}{1 + \beta} \textrm{Bin}(n, p_0),
\end{align*}
where $p_0 = \bP \left \{ \zeta^2_i \geq 1 \right \} = \textrm{erfc}(\nicefrac{1}{\sqrt{2}}) $. Hence, by Lemma \ref{lem:bin-rv-tail}, it follows that, 
\begin{align*}
    \bP \left \{ \sum_{i=1}^{n} \frac{\zeta^2_i}{\lambda + \zeta^2_i} \geq \frac{n \sigma^2 p_0}{2(\sigma^2 + \lambda)} \right \} \geq 1 - e^{\nicefrac{-np_0}{8}}
\end{align*}
\end{proof}
\begin{lemma}[Tail Bound for Rational Functions of $\chi^2$ RVs - II]
\label{lem:chisq-rational-term3-bound}
Let $\gamma_i, \zeta_i \ \iidsim \ \cN(0, 1)$ and let $\sigma > 0$. Then, for any $\lambda \leq \sigma^2$
\begin{align*}
    \bP \left \{ \sum_{i=1}^{n} \frac{\zeta^2_i}{\lambda + \sigma^2 \gamma^2_i} \geq \frac{np_0}{16 \sigma \sqrt{ \lambda}} \right \} \geq 1 - \exp\left(-\frac{np_0 \sqrt{\lambda}}{32 \sigma}\right)
\end{align*}
\end{lemma}
\begin{proof}
Let $\lambda = \beta \sigma^2 $. Using the fact that $\gamma_i$ and $\zeta_i$ are independent.
\begin{align*}
    \frac{\zeta^2_i}{\lambda + \sigma^2 \gamma^2_i} &\geq \frac{1}{2\lambda} \bI \left \{ \zeta^2_i \geq 1 \right \} \bI \left \{ \frac{1}{\beta + \gamma^2_i} \geq \frac{1}{2\beta} \right \} = \frac{1}{2 \lambda} \bI \left \{ \zeta^2_i \geq 1 \right\} \bI \left \{ \gamma^2_i \leq \beta \right\}, \\
    \sum_{i=1}^{n} \frac{\zeta^2_i}{\lambda + \sigma^2 \gamma^2_i} &\geq 
    \frac{1}{2\lambda} \sum_{i=1}^{n} \bI \left \{ \zeta^2_i \geq 1 \right\} \bI \left \{ \gamma^2_i \leq \beta \right\} = \frac{1}{2\lambda} \textrm{Bin}(n, p_0 p_1) ,
\end{align*}
where,
\begin{align*}
    p_0 &= \bP \left \{ \zeta^2_i \geq 1 \right \} = \textrm{erfc}(\nicefrac{1}{\sqrt{2}}),\\
    p_1 &= \bP \left \{ \gamma^2_i \leq \beta \right \} = \sqrt{\frac{2}{\pi}}  \int_{0}^{\sqrt{\beta}} e^{\nicefrac{-x^2}{2}} dx \geq \sqrt{\frac{2\beta}{\pi}} e^{-\nicefrac{\sqrt{\beta}}{2}} \geq \frac{\sqrt{\beta}}{4} \geq \frac{\sqrt{\lambda}}{4 \sigma}
\end{align*}
where the last inequality follows since $\beta = \nicefrac{\lambda}{\sigma^2} \leq 1$. Thus, by Lemma \ref{lem:bin-rv-tail}, 
\begin{align*}
    \bP \left\{ \sum_{i=1}^{n} \frac{\zeta^2_i}{\lambda + \sigma^2 \gamma^2_i} \geq \frac{n p_0 p_1}{4\lambda}\right\} \geq 1 - e^{\nicefrac{-np_0 p_1}{8}}
\end{align*}
Since $p_1 \geq \frac{\sqrt{\lambda}}{4 \sigma}$, we conclude that, 
\begin{align*}
    \bP \left \{ \sum_{i=1}^{n} \frac{\zeta^2_i}{\lambda + \sigma^2 \gamma^2_i} \geq \frac{np_0}{16 \sigma \sqrt{ \lambda}} \right \} \geq 1 - \exp\left(-\frac{np_0 \sqrt{\lambda}}{32 \sigma}\right)
\end{align*}
\end{proof}
\begin{lemma}[Tail Bound for Rational Functions of Gaussian RVs]
\label{lem:chisq-rational-term2-bound}
Let $\gamma_i, \zeta_i \ \iidsim \ \cN(0, 1)$. Then, for any $\sigma > 0, \lambda > 0$ and $A \geq 0$,
\begin{align*}
    \bP \left\{ \sum_{i=1}^{n} \frac{\sigma \gamma_i \zeta_i}{\lambda + \sigma^2 \gamma^2_i} \geq -\frac{n \sqrt{2A}}{\lambda^{\nicefrac{1}{4}}} \right\} \geq 1 - e^{-An\sqrt{\lambda}}
\end{align*}
\end{lemma}
\begin{proof}
We show that $\sum_{i=1}^{n} \frac{\gamma_i \zeta_i}{\lambda + \gamma^2_i}$ is a subgaussian random variable. Since $\gamma_i, \zeta_i \ \iidsim \ \cN(0, 1)$, it follows that $\bE\left[\sum_{i=1}^{n} \frac{\sigma \gamma_i \zeta_i}{\lambda + \sigma^2 \gamma^2_i}\right] = 0$. Moreover, for any $\mu \in \bR$,
\begin{align*}
    \bE\left[ \exp\left( \mu \sum_{i=1}^{n} \frac{\sigma \gamma_i \zeta_i}{\lambda + \sigma^2 \gamma^2_i}\right) \right] &= \prod_{i=1}^{n} \bE\left[ \exp\left( \mu \frac{\sigma \gamma_i \zeta_i}{\lambda + \sigma^2 \gamma^2_i}\right) \right] \\
    &= \prod_{i=1}^{n} \bE\left[\bE\left[ \exp\left( \mu \frac{\sigma \gamma_i \zeta_i}{\lambda + \sigma^2 \gamma^2_i}\right) \ | \ \gamma_i \right]\right] \\
    &= \prod_{i=1}^{n} \bE\left[ \exp\left(\frac{\mu^2 \sigma^2 \gamma^2_i}{2(\lambda + \sigma^2 \gamma^2_i)^2}\right) \right] \\
    &\leq \prod_{i=1}^{n} \bE\left[ \exp\left(\frac{\mu^2}{2\lambda}\frac{\sigma^2 \gamma^2_i}{\lambda + \sigma^2 \gamma^2_i}\right) \right] \leq\prod_{i=1}^{n} \exp\left(\frac{\mu^2}{2\lambda}\right) = \exp\left(\frac{\mu^2}{2}\frac{n}{\lambda}\right)
\end{align*}
Hence, for any $t \geq 0$, 
\begin{align*}
    \bP\left\{ \sum_{i=1}^{n} \frac{\sigma \gamma_i \zeta_i}{\lambda + \sigma^2 \gamma^2_i} \geq -t \right\} \geq 1 - e^{-\nicefrac{\lambda t^2}{2n}}
\end{align*}
Setting $\nicefrac{\lambda t^2}{2n} = A n \sqrt{\lambda}$, we get,
\begin{align*}
    \bP \left\{ \sum_{i=1}^{n} \frac{\sigma \gamma_i \zeta_i}{\lambda + \sigma^2 \gamma^2_i} \geq -\frac{n \sqrt{2A}}{\lambda^{\nicefrac{1}{4}}} \right\} \geq 1 - e^{-An\sqrt{\lambda}}
\end{align*}
\end{proof}

\begin{lemma}
\label{lem:reweighted-design-conc}
Let $\vf$ be any arbitrary random vector and let $\vx_1, \dots, \vx_n$ be i.i.d samples from $\cN(0, \vI)$ that are independent of $\vf$. Then, for any $\delta \in (0, 1/2)$, $n \geq d \polylog(\nicefrac{d}{\delta})$ and any $\norm{\vf}^2 \geq \lambda \geq \Omega\left(\norm{\vf}^2 \frac{d^2}{n^2} \log(\nicefrac{nd}{\delta})^2\right)$, the following holds with probability at least $1 - \delta$,
\begin{align*}
    \Tr\left(\left[ \sum_{i=1}^{n} \frac{\vx_i \vx^T_i}{{\dotp{\vf}{\vx_i}}^2 + \lambda} \right]^{-1}\right) \leq O\left(\nicefrac{\norm{\vf}^2}{n} + \nicefrac{d\norm{\vf}\sqrt{\lambda}}{n}\right)
\end{align*}
\end{lemma}
\begin{proof}
Let $\vX_{\lambda} = \sum_{i=1}^{n} \frac{\vx_i \vx^T_i}{{\dotp{\vf}{\vx_i}}^2 + \lambda}$. Furthermore, define $\ve = \nicefrac{\vf}{\norm{\vf}}$ and the matrix $\vB_{\lambda}$ as follows,
\begin{align*}
    \vB_{\lambda} = \frac{C_1 n}{\norm{\vf}^2 + \lambda} \ve \ve^T + \frac{C_2 n}{\norm{\vf}\sqrt{\lambda}} \left(\vI - \ve \ve^T \right)
\end{align*}
where $C_1 \leq C_2$ are numerical constants to be chosen later. We note that, since $\vB_{\lambda}$ is symmetric and positive definite and $\lambda \leq \norm{\vf}^2$,
\begin{align*}
    \Tr\left(\vB^{-1}_{\lambda}\right) \leq O\left(\frac{\norm{\vf}^2 + \lambda}{n} + \frac{(d-1)\norm{\vf}\sqrt{\lambda}}{n}\right) \leq O\left(\frac{\norm{\vf}^2}{n} + \frac{d\norm{\vf}\sqrt{\lambda}}{n}\right)
\end{align*}
Hence, it suffices to prove that $\vX_{\lambda} \succeq \vB_{\lambda}$ holds with high probability. We establish this by means of a covering argument. To this end, consider any $\epsilon \in (0, 1)$ and let $C_{\epsilon}$ denote an $\epsilon$ cover of the unit ball in $\bR^d$. Applying a standard discretization argument, we get,
\begin{align*}
    \inf_{\norm{\vx} = 1} \vx^{T} \left(\vX_{\lambda} - \vB_{\lambda}\right) \vx &\geq \inf_{\vv \in C_\epsilon} \vv^{T} \left(\vX_{\lambda} - \vB_{\lambda}\right) \vv - 2 \epsilon \norm{\vX_{\lambda}} - 2 \epsilon \norm{\vB_{\lambda}}
\end{align*}
Since $\lambda \leq \norm{\vf}^2$, $\lambda + \norm{\vf}^2 \geq \norm{\vf}\sqrt{\lambda}$. Hence, $\norm{\vB_{\lambda}} = \frac{C_2 n}{\norm{\vf}\sqrt{\lambda}}$. Moreover,
\begin{align*}
    \norm{\vX_{\lambda}} &\leq \sum_{i=1}^{n} \frac{\norm{\vx_i \vx^T_i}}{\lambda + {\dotp{\vf}{\vx_i}}^2} \leq \frac{n}{\lambda} \left(\frac{1}{n} \sum_{i=1}^{n} \norm{\vx_i}^2\right)
\end{align*}
Since $\vx_i \ \iidsim \ \cN(0, \vI)$, the following holds with probability $1 - \nicefrac{\delta}{2}$ due to $\chi^2$ concentration,
\begin{align*}
    \norm{\vX_{\lambda}} &\leq \frac{n}{\lambda} \left[d + \max \left\{\frac{8 \log(\nicefrac{2}{\delta})}{n}, \sqrt{\frac{8 d \log(\nicefrac{2}{\delta})}{n}} \right\}\right] \leq \frac{n}{\lambda} (d+1)
\end{align*}
where the last inequality follows from the fact that $n \geq d \polylog(\nicefrac{d}{\delta})$. Hence, the following holds with probability at least $1 - \nicefrac{\delta}{2}$,
\begin{align}
\label{eqn:design-eigenvalue-bound}
    \inf_{\norm{\vx} = 1} \vx^{T} \left(\vX_{\lambda} - \vB_{\lambda}\right) \vx &\geq \inf_{\vv \in C_\epsilon} \vv^{T} \left(\vX_{\lambda} - \vB_{\lambda}\right) \vv - \frac{2\epsilon n}{\lambda}(d+1) -  \frac{2 \epsilon C_2 n}{\norm{\vf}\sqrt{\lambda}}
\end{align}
Now consider any $\vv \in C_{\epsilon}$. Since $\norm{\vv} = 1$, there exists a unit vector $\vg$ orthogonal to $\vf$ such that $\vv = \alpha \ve + \beta \vg$, with $\alpha^2 + \beta^2 = 1$. We note that,
\begin{align*}
    \vv^T \vB_\lambda \vv &= \frac{\alpha^2 n C_1}{\norm{\vf}^2 + \lambda} + \frac{\beta^2 n C_2}{\norm{\vf}\sqrt{\lambda}} \\
    \vv^T \vX_\lambda \vv &= \frac{\alpha^2}{\norm{\vf}^2} \sum_{i=1}^{n} \frac{{\dotp{\vf}{\vx_i}}^2}{\lambda + {\dotp{\vf}{\vx_i}}^2} + \beta^2 \sum_{i=1}^{n} \frac{{\dotp{\vg}{\vx_i}}^2}{\lambda + {\dotp{\vf}{\vx_i}}^2} + \frac{2\alpha \beta}{\norm{\vf}} \sum_{i=1}^{n} \frac{\dotp{\vf}{\vx_i}\dotp{\vg}{\vx_i}}{\lambda + {\dotp{\vf}{\vx_i}}^2}
\end{align*}
Since $\vf$ and $\vg$ are orthogonal, $\dotp{\vf}{\vx_i}$ and $\dotp{\vg}{\vx_i}$ are independent. From Lemmas \ref{lem:chisq-rational-term1-bound}, \ref{lem:chisq-rational-term3-bound} and \ref{lem:chisq-rational-term2-bound}, and the fact that $\lambda \leq \norm{\vf}^2$, we conclude that the following holds with probability at least $1 - e^{-\nicefrac{n p_0}{8}} - e^{-\frac{np_0 \sqrt{\lambda}}{32 \norm{\vf}}} - e^{-\frac{A_\alpha A_\beta n \sqrt{\lambda}}{\norm{\vf}}}$,
\begin{align*}
    \frac{\alpha^2}{\norm{\vf}^2} \sum_{i=1}^{n} \frac{{\dotp{\vf}{\vx_i}}^2}{\lambda + {\dotp{\vf}{\vx_i}}^2} &\geq \frac{\alpha^2 n p_0}{2\lambda + 2\norm{\vf}^2} \\
    \beta^2 \sum_{i=1}^{n} \frac{{\dotp{\vg}{\vx_i}}^2}{\lambda + {\dotp{\vf}{\vx_i}}^2} &\geq \frac{\beta^2 n p_0 \sqrt{\lambda}}{16 \norm{\vf}} \\
    \frac{2\alpha \beta}{\norm{\vf}} \sum_{i=1}^{n} \frac{\dotp{\vf}{\vx_i}\dotp{\vg}{\vx_i}}{\lambda + {\dotp{\vf}{\vx_i}}^2} &\geq -\frac{2\alpha \beta}{\norm{\vf}} \frac{n \sqrt{2 A_\alpha A_\beta}}{\norm{\vf}^{\nicefrac{1}{2}}\lambda^{\nicefrac{1}{4}}} \geq -\frac{2\sqrt{2} \alpha^2 n A_\alpha}{\lambda + \norm{\vf}^2} - \frac{\sqrt{2}\beta^2 n A_\beta}{\norm{\vf}\sqrt{\lambda}}
\end{align*}
Setting $C_1 = C_2 = \nicefrac{p_0}{64}$ and $A_\alpha = A_\beta = \nicefrac{p_0}{64 \sqrt{2}}$, we conclude that the following holds with probability at least $1 - 3 e^{-\nicefrac{np^2_0 \sqrt{\lambda}}{512 \norm{\vf}}}$
\begin{align*}
    \vv^T \left(\vX_\lambda - \vB_\lambda \right) \vv &\geq  \frac{n \alpha^2}{\lambda + \norm{\vf}^2} \left(\nicefrac{p_0}{2} - C_1 - 2\sqrt{2} A_\alpha\right) + \frac{n \beta^2}{\norm{\vf}\sqrt{\lambda}} \left(\nicefrac{p_0}{16} - C_2 - \sqrt{2} A_\beta\right) \\
    &\geq \frac{\nicefrac{n p_0}{64}}{\lambda + \norm{\vf}^2}
\end{align*}
The second inequality follows since $\alpha^2 + \beta^2 = 1$ and $\lambda \leq \norm{\vf}^2$ implies $\norm{\vf}\sqrt{\lambda} \leq \norm{\vf}^2 \leq \lambda + \norm{\vf}^2$. Taking a union bound over $C_\epsilon$, using the fact that $|C_\epsilon| \leq (\nicefrac{3}{\epsilon})^d$, we conclude that the following holds with probability at least $1 - \exp(\ln 3 + d \ln\left(\nicefrac{3}{\epsilon}\right) -\nicefrac{np^2_0 \sqrt{\lambda}}{512 \norm{\vf}})$,
\begin{align}
\label{eqn:covering-bound-hp}
    \inf_{\vv \in C_\epsilon} \vv^T \left(\vX_\lambda - \vB_\lambda \right) \vv &\geq \frac{\nicefrac{n p_0}{64}}{\lambda + \norm{\vf}^2}
\end{align}
To ensure that the above event holds with probability at least $1 - \nicefrac{\delta}{2}$, we require $\lambda$ to be lower bounded as,
\begin{align}
\label{eqn:wls-lambda-lower-bound}
\frac{\sqrt{\lambda}}{\norm{\vf}} \geq \frac{512}{np^2_0} \left(\ln(6) + d \ln\left(\nicefrac{3}{\epsilon}\right) + \ln\left(\nicefrac{2}{\delta}\right) \right)
\end{align}
Suppose $\lambda$ and $\epsilon$ appropriately chosen (to be specified later) such that equation \eqref{eqn:wls-lambda-lower-bound} is satisfied. Then, from equations \eqref{eqn:design-eigenvalue-bound} and \eqref{eqn:covering-bound-hp}, we conclude that the following holds with probability at least $1 - \delta$,
\begin{align*}
    \inf_{\norm{\vx} = 1} \vx^{T} \left(\vX_{\lambda} - \vB_{\lambda}\right) \vx &\geq \frac{\nicefrac{n p_0}{64}}{\lambda + \norm{\vf}^2} - \frac{2\epsilon n (d+1)}{\lambda} -  \frac{\epsilon n p_0}{32 \norm{\vf}\sqrt{\lambda}} \\
    &\geq \frac{\nicefrac{n p_0}{64}}{\lambda + \norm{\vf}^2} - \frac{\epsilon n}{\lambda} \left(2d + 2 + \nicefrac{p_0}{32}\right)
\end{align*}
To ensure that the RHS is non-negative, $\epsilon$ must satisfy the following,
\begin{align}
\label{eqn:wls-epsilon-lower-bound}
\nicefrac{1}{\epsilon} \geq \frac{64}{p_0}\left(1 + \nicefrac{\norm{\vf}^2}{\lambda}\right)\left(2d + 2 + \nicefrac{p_0}{32}\right)
\end{align}
Without loss of generality, assume $\epsilon \leq e^{-1}$. Then, we note that equation \eqref{eqn:wls-lambda-lower-bound} is satisfied if,
\begin{align}
\label{eqn:wls-lambda-finer-bound}
    \frac{\sqrt{\lambda}}{\norm{\vf}} \geq \frac{\tau_1 d}{n} \left[\ln(\nicefrac{d}{\delta}) + \ln(\nicefrac{2}{\epsilon})\right]
\end{align}
where $\tau_1 > 10^4$ is a universal constant. It follows that, $\nicefrac{\norm{\vf}^2}{\lambda} \leq \nicefrac{n^2}{d^2 \tau_1^2}$. We observe that, for this choice of $\lambda$, equation \eqref{eqn:wls-epsilon-lower-bound} is satisfied if $\nicefrac{1}{\epsilon} = \tau_2 d^2 n^2$ for some absolute constant $\tau_2 \geq 10^5$. Substituting this choice of $\nicefrac{1}{\epsilon}$ into equation \eqref{eqn:wls-lambda-finer-bound}, we note that equation \eqref{eqn:wls-lambda-finer-bound} is satisfied if,
\begin{align*}
    \frac{\sqrt{\lambda}}{\norm{\vf}} \geq \frac{\tau_1 d}{n} \left[\ln\left(\nicefrac{d}{\delta}\right) + 2 \tau_2 \ln(d) + 2 \tau_2 \ln(n) \right]
\end{align*}
Hence, there exists a universal constant $\tau_3$ such that the above is satisfied when,
\begin{align*}
    \sqrt{\lambda} \geq \tau_3 \norm{\vf} \nicefrac{d}{n} \log(\nicefrac{n d}{\delta})
\end{align*}
Hence, we conclude that, setting $\nicefrac{1}{\epsilon} = \Theta(n^2 d^2)$ and $\norm{\vf}^2 \geq \lambda \geq \Omega(\norm{\vf}^2\frac{d^2}{n^2} \log(\nicefrac{nd}{\delta})^2)$ is sufficient to ensure that $\inf_{\norm{\vx} = 1} \vx^{T} \left(\vX_{\lambda} - \vB_{\lambda}\right) \vx \geq 0$  with probability at least $1 - \delta$, i.e., $\vX_\lambda \succeq \vB_\lambda$ with probability at least $1 - \delta$. Furthermore, the condition $n \geq d \polylog(\nicefrac{d}{\delta})$ ensures that $\lambda \leq \norm{\vf}^2$ and $\lambda \geq \Omega(\norm{\vf}^2\frac{d^2}{n^2} \log(\nicefrac{nd}{\delta})^2)$ can be simultaneously satisfied. 
\end{proof}

\subsection{Proof of Theorem \ref{thm:wls-error-bound}}
\label{subsec:wls_pf}
\begin{proof}
We recall that, given any relaxation parameter $\lambda$ and approximate noise model $\hvf$, the WLS estimator $\hvw_{\hvf, \lambda}$ is defined as follows.
\begin{align*}
    \hvw_{\hvf, \lambda} = \left[\sum_{i=1}^{n} \frac{\vx_i \vx^T_i}{{\dotp{\hvf}{\vx_i}}^2 + \lambda }\right]^{-1} \left[\sum_{i=1}^{n} \frac{\vx_i y_i}{{\dotp{\hvf}{\vx_i}}^2 + \lambda}\right]
\end{align*}
We use $\vX_{\lambda, \hvf}$ to denote the design matrix of WLS, which is defined as,
\begin{align*}
    \vX_{\lambda, \hvf} = \left[\sum_{i=1}^{n} \frac{\vx_i \vx^T_i}{{\dotp{\hvf}{\vx_i}}^2 + \lambda }\right]
\end{align*}
Using $\vy_i = \dotp{\wstar}{\vx_i} + \dotp{\fstar}{\vx_i}\epsilon_i$, we observe that,
\begin{align*}
    \hvw_{\hvf, \lambda} - \wstar &= \vX^{-1}_{\lambda, \hvf} \left[\sum_{i=1}^{n} \frac{\epsilon_i \dotp{\fstar}{\vx_i} \vx_i}{{\dotp{\hvf}{\vx_i}}^2 + \lambda}\right] = \sum_{i=1}^{n} \epsilon_i \vv_i,
\end{align*}
where $\vv_i$ is defined as,
\begin{align*}
    \vv_i &= \left[\frac{\dotp{\fstar}{\vx_i} }{{\dotp{\hvf}{\vx_i}}^2 + \lambda}\right] \vX^{-1}_{\lambda, \hvf} \vx_i
\end{align*}
We now define the matrix $\vM$ as follows
\begin{align*}
    \vM &= \sum_{i=1}^{n} \vv_i \vv^T_i = \vX^{-1}_{\lambda, \hvf} \left[ \sum_{i=1}^{n} \frac{{\dotp{\fstar}{\vx_i}}^2 \vx_i \vx^T_i}{\left[{\dotp{\hvf}{\vx_i}}^2 + \lambda\right]^2} \right] \vX^{-1}_{\lambda, \hvf}
\end{align*}
We note that, since $\vX_{\lambda, \hvf}$ is a symmetric PSD matrix, so is $\vM$. Furthermore, since $\epsilon_i$ are independent of $\vv_i$, we conclude that the following holds with probability at least $1 - \nicefrac{\delta}{3}$,
\begin{equation}
\label{eqn:wls-proof-parameter-error-Tr-M}
    \norm{\hvw_{\hvf, \lambda} - \wstar}^2 \leq \Tr(\vM)\left[1 + \max \left\{ 8\log(\nicefrac{3d}{\delta}), \sqrt{8\log(\nicefrac{3d}{\delta})}  \right\} \right]
\end{equation}
Motivated by the fact that Lemma \ref{lem:reweighted-design-conc} allows us to upper bound the trace of $\vX_{\lambda, \hvf}^{-1}$, but directly controlling $\Tr(\vM)$ might be cumbersome, we now aim to establish that $\vM \preceq 2 \vX_{\lambda, \hvf}^{-1}$ holds with high probability, and consequently, so does $\Tr(\vM) \leq 2 \Tr\left(\vX_{\lambda, \hvf}^{-1}\right)$. To this end, we recall that $\norm{\hvf - \fstar}^2 \leq \epsilon$ and $\vx_1 \dots, \vx_n$ are independent of $\vx_i$. Thus,  $\left(\hvf - \fstar\right)^{T} \vx_i \ \iidsim \cN(0, \epsilon)$. Furthermore, by concentration of suprema of Gaussian random variables, the following holds with probability at least $1 - \nicefrac{\delta}{3}$,
\begin{align*}
    \abs{\left(\hvf - \fstar\right)^{T} \vx_i} \leq \beta = \sqrt{2 \epsilon \log\left(\nicefrac{6n}{\delta}\right)} \ \ \forall \ i \in [n]
\end{align*}
Conditioned on the above event occuring, the following inequalities hold uniformly for every $i \in [n]$ with probability 1.
\begin{align*}
    \abs{\left(\hvf^{T} \vx_i\right)^2 - \left((\fstar)^{T} \vx_i\right)^2} &= \abs{\left(\hvf + \fstar \right)^{T} \vx_i} \abs{\left(\hvf + \fstar \right)^{T} \vx_i} \\
    &\leq \beta \abs{2 \left({\fstar}^{T} \vx_i\right) + \beta} \\
    &\leq 2 \beta \abs{\left(({\fstar})^{T} \vx_i\right)} + \beta^2 \\
    &\leq \frac{\left(({\fstar})^{T} \vx_i\right)^2}{2} + 6 \epsilon \log\left(\nicefrac{2n}{\delta}\right) \ \ \forall \ i \in [n]
\end{align*}
Furthermore, 
\begin{align*}
    \left((\fstar)^T \vx_i\right)^2 &\leq \left((\hvf)^T \vx_i\right)^2 + \abs{\left((\hvf)^{T} \vx_i\right)^2 - \left(({\fstar})^{T} \vx_i\right)^2} \\
    &\leq \left((\hvf)^T \vx_i\right)^2 + \lambda + \frac{\left(({\fstar})^{T} \vx_i\right)^2}{2} + 6 \epsilon \log\left(\nicefrac{6n}{\delta}\right) \ \ \forall \ i \in [n]
\end{align*}
Setting $\lambda \geq \max \{ \epsilon,  \nicefrac{\norm{\hvf}^2d^2}{n^2} \} \polylog(\nicefrac{nd}{\delta}) \geq 6 \epsilon \log\left(\nicefrac{6n}{\delta}\right)$, we conclude that the following holds with probability at least $1 - \nicefrac{\delta}{3}$
\begin{align}
\label{eqn:wls-proof-algebraic-trick}
    \frac{\left((\fstar)^T \vx_i\right)^2}{\left((\hvf)^T \vx_i\right)^2 + \lambda} \leq 2 \ \ \forall \ i \in [n]
\end{align}
Since $\vX_{\lambda, \hvf}$ and $\vM$ are both PSD matrices, we conclude that the following holds with probability at least $1 - \nicefrac{\delta}{3}$,
\begin{align}
\label{eqn:wls-proof-M-psd-bound}
    \vM &= \vX^{-1}_{\lambda, \hvf} \left[ \sum_{i=1}^{n} \frac{{\dotp{\fstar}{\vx_i}}^2 \vx_i \vx^T_i}{\left[{\dotp{\hvf}{\vx_i}}^2 + \lambda\right]^2} \right] \vX^{-1}_{\lambda, \hvf} \nonumber \\
    &\preceq 2 \vX^{-1}_{\lambda, \hvf} \left[ \sum_{i=1}^{n} \frac{\vx_i \vx^T_i}{{\dotp{\hvf}{\vx_i}}^2 + \lambda} \right] \vX^{-1}_{\lambda, \hvf} \nonumber \\
    &\preceq 2 \vX^{-1}_{\lambda, \hvf}
\end{align}
where the second PSD inequality follows from \eqref{eqn:wls-proof-algebraic-trick}. Furthermore, applying Lemma \ref{lem:reweighted-design-conc} to $\vX_{\lambda, \hvf}$, we ensure that the following holds with probability $1 - \nicefrac{\delta}{3}$
\begin{align}
\label{eqn:wls-proof-design-trace-bound}
    \Tr\left(\vX^{-1}_{\lambda, \hvf}\right) \leq \left(\nicefrac{\norm{\hvf}^2}{n} + \nicefrac{d \norm{\hvf} \sqrt{\lambda}}{n} \right) \polylog(\nicefrac{nd}{\delta})
\end{align}
Finally, from \eqref{eqn:wls-proof-parameter-error-Tr-M}, \eqref{eqn:wls-proof-M-psd-bound} and \eqref{eqn:wls-proof-design-trace-bound}, we conclude that the following guarantee holds with probability at least $1 - \delta$,
\begin{align*}
    \norm{\hvw_{\hvf, \lambda} - \wstar}^2 &\leq \left(\nicefrac{\norm{\hvf}^2}{n} + \nicefrac{d \norm{\hvf} \sqrt{\lambda}}{n}\right )\polylog(\nicefrac{nd}{\delta}) \\
    &\leq \left(\nicefrac{\norm{\fstar}^2}{n} + \nicefrac{d \norm{\fstar} \sqrt{\lambda}}{n} + \nicefrac{\epsilon}{n} + \nicefrac{d \sqrt{\epsilon \lambda}}{n}\right)\polylog(\nicefrac{nd}{\delta})
\end{align*}
\end{proof}
\subsection{Algorithm and results for linear regression with multiplicative noise}
\label{app:ourmultalg}
In this section, we present a simpler algorithm, \ourmultalg, in Algorithm~\ref{alg:self_symbol} for linear regression with multiplicative noise, that achieves an improved rate compared to OLS.
\begin{algorithm}
\caption{\texttt{\ourmultalg} } \label{alg:self_symbol}
\textbf{Require}: $(\vx_1, y_1), \dots, (\vx_n, y_n)$.  Steps $K$. Weights $\lambda_1, \dots, \lambda_K$.
\begin{algorithmic}[1] 
\STATE Divide the data into $K + 1$ partitions of size $m = \left\lfloor{\nicefrac{n}{K}}\right\rfloor$
\STATE Let $\hvw_0$ be the OLS estimate $\wols$ computed on the first data partition
\FOR{$k \in \{1,\ldots,K \}$}
    \STATE Compute $\hvw_k$ to be the WLS estimator (Algorithm \ref{alg:wls}) computed on the $(k+1)^{\textrm{th}}$ data partition using regularization weight $\lambda_k$ and noise model estimate $\hvw_{k-1}$, i.e, $\hvw_k = \hvw_{\hvw_{k-1}, \lambda_k}$
\ENDFOR
\STATE Output $\hvw_K$
\end{algorithmic}
\end{algorithm}
\begin{theorem}[\ourmultalg] 
\label{thm:yinyin-mult-regression-bound}In the heteroscedastic regression model, we take $\wstar = \pm \fstar$. 
Consider any $\delta \in (0, \nicefrac{1}{2})$ and let $(\vx_1, y_1), \dots, (\vx_n, y_n) \in \bR^d \times \bR$ be i.i.d samples from the multiplicative linear regression model. Then, for $n \geq d \polylog(d)$ and $K = \Theta(\log(n))$, there exists an appropriate choice of weights $\lambda_1, \dots, \lambda_K$ such that the output $\hvw_K$ of the \ourmultalg~algorithm satisfies the following with probability at least $1 - \delta$:
$$\norm{\vw_K - \wstar}^2 \leq \Otilde\left(\nicefrac{\norm{\wstar}^2}{n} + \nicefrac{\norm{\wstar}^2 d^2}{n^2}\right)$$
\end{theorem}
\begin{proof}[Proof of Theorem \ref{thm:yinyin-mult-regression-bound}]
For ease of exposition, assume $n = mK$ where $K = \lceil\log_2(n)\rceil$ and $m \geq \Omegatilde(d)$. For $k \in \{0, \dots, K \}$, let $e_k = \norm{\vw_k - \wstar}^2$. Since $\hvw_0$ is the OLS estimate computed on $m$ samples, we know that with probability at least $1 - \nicefrac{\delta}{(K+1)}$, $e_0 = \Otilde(\norm{\wstar}^2\nicefrac{d}{m})$. Set $\lambda_1 = \Otilde(\norm{\hvw_0}^2 \nicefrac{d}{m})$. Since $\hvw_1 = \hvw_{\hvw_0, \lambda_1}$, it follows from Theorem \ref{thm:wls-error-bound} and a union bound that the following holds with probability at least $1 - \nicefrac{2\delta}{(K+1)}$,
\begin{align*}
    e_1 = \norm{\hvw_1 - \wstar}^2 &\leq \Otilde\left(\frac{\norm{\hvw_0}^2}{m} + \frac{d \norm{\hvw_0} \sqrt{\lambda_1}}{m} \right) \\
    &\leq \Otilde\left( \norm{\hvw_0}^2 \left(\nicefrac{1}{m} + \left(\nicefrac{d}{m}\right)^{1.5}\right) \right) \\
    &\leq \Otilde\left(\norm{\wstar}^2\left(1 + \nicefrac{d}{m}\right)\left(\nicefrac{1}{m} + \left(\nicefrac{d}{m}\right)^{1.5}\right)\right) \\
    &\leq \Otilde\left(\norm{\wstar}^2 \left(\nicefrac{1}{m} + \left(\nicefrac{d}{m}\right)^{1.5}\right)\right) \nonumber \\
    &= \left(\norm{\wstar}^2 \left(\nicefrac{1}{m} + \left(\nicefrac{d}{m}\right)^{1.5}\right)\right) \polylog(\tfrac{nd}{\delta})
\end{align*}
where the last inequality follows from the fact that $m \geq \Omegatilde(d)$. We now prove the required convergence guarantee via induction. To this end, we define $S_k = \sum_{j=0}^{k} \nicefrac{1}{2^j}$ Clearly, $1 \leq S_k \leq 2$ and $S_{k+1} = 1 + \nicefrac{S_k}{2}$. Let $L = \polylog(\nicefrac{nd}{\delta}) > 1$ be a large enough polylog factor independent of $k$. We now define the event $E_k(L)$ as follows: 
\begin{enumerate}
    \item $\|\hvw_l\| \leq 2 \|\wstar\|$ for every $1\leq l\leq k$
    \item $e_{l} \leq \norm{\wstar}^2 \left(\max(\frac{l}{m},\frac{ld^2}{m^2}) + (\frac{d}{m})^{S_l}\right)L$ for every $1\leq l\leq k$
\end{enumerate}

We set, $\lambda_{k+1} = \norm{\hvw_k}^2 \left(\max(\frac{k}{m},\frac{kd^2}{m^2}) + (\frac{d}{m})^{S_k}\right)L\polylog(\frac{nd}{\delta})$. Note that $\hvw_{k+1} = \hvw_{\hvw_k, \lambda_{k+1}}$, we conclude from Theorem \ref{thm:wls-error-bound} that whenever $\polylog$ is large enough (independent of $k,K$), when conditioned on the event $E_k(L)$ the following holds with probability at least $1 - \nicefrac{\delta}{(K+1)}$, with $\polylog$ factor being independent of $k$:

\begin{align*}
    e_{k+1} &\leq \left(\frac{\norm{\wstar}^2}{m} + \frac{d \norm{\wstar} \sqrt{\lambda_k}}{m}\right)\polylog(nd/\delta) \\
    &\leq \left(\frac{\norm{\wstar}^2}{m} + \norm{\wstar}^2\max\left(\frac{d\sqrt{k}}{m^{3/2}},\frac{d^2\sqrt{k}}{m^{2}}\right)+\norm{\wstar}^2\left(\frac{d }{m}\right)^{1+S_k/2}\right)\sqrt{L}\polylog(nd/\delta) \\
    &\leq \left(\norm{\wstar}^2\max\left(\frac{k+1}{m},\frac{d^2(k+1)}{m^{2}}\right)+\norm{\wstar}^2\left(\frac{d }{m}\right)^{1+S_k/2}\right)\sqrt{L}\polylog(nd/\delta)
\end{align*}

In the second step, we have used the fact that $\sqrt{x+y} \leq \sqrt{x} + \sqrt{y}$. In the third step, we use the fact that $\sqrt{k}\leq k$ and $\frac{d}{m^{3/2}} \leq \max(\frac{1}{m},\frac{d^2}{m^2})$, since it is the geometric mean of $1/m$ and $d^2/m^2$. Picking $L$ to be a large enough $\polylog(\tfrac{nd}{\delta})$ (independent of $k$), we conclude that conditioned on $E_k(L)$, we must have with probability at-least $1-\delta/(K+1)$:
$$e_{k+1} \leq \left(\norm{\wstar}^2\max\left(\frac{k+1}{m},\frac{d^2(k+1)}{m^{2}}\right)+\norm{\wstar}^2\left(\frac{d }{m}\right)^{1+S_k/2}\right)L$$

If we take $n > d \polylog(d/\delta)$ for large enough $\polylog()$ as in the statement of the theorem, the above equation when combined with the triangle inequality also implies that:

$$\|\hvw_{k+1}\| \leq 2\|\wstar\|\,.$$ 

Therefore, we conclude:
$\bP(E_{k+1}|E_{k}) \geq 1-\frac{\delta}{K+1}$. Applying a union bound on this and using the fact that $\bP(E_1) \geq 1-\frac{2\delta}{K+1}$, we conclude:

$$\bP(E_{K+1}) \geq 1-\delta$$

When $K \geq \log n$, we have $2\geq S_{K+1} \geq 2-1/n$. Thus, we have under the event $E_{K+1}$:

\begin{align*}
e_{K+1} &\leq \norm{\wstar}^2 \left(\max(\frac{1}{m},\frac{d^2}{m^2}) + \left(\frac{d}{m}\right)^{2-\tfrac{1}{n}}\right)\polylog(\tfrac{nd}{\delta}) \nonumber \\
&\leq \norm{\wstar}^2 \left(\max(\frac{1}{m},\frac{d^2}{m^2}) + \left(\frac{d}{m}\right)^{2}\right)\polylog(\tfrac{nd}{\delta}) \end{align*}

Which proves the result.
\end{proof}
\subsection{Proof of Theorem \ref{thm:ols-error-bound}}
\label{proof:ols-error-proof}
To derive the high probability upper bound, consider any $\delta \in (0, \nicefrac{1}{2})$. We note that the OLS estimator for the heteroscedastic regression problem satisfies
\begin{align*}
    \wols - \wstar = \left(\sum_{i=1}^{n} \vx_i \vx^T_i \right)^{-1} \left(\sum_{i=1}^{n} \dotp{\fstar}{\vx_i} \epsilon_i \vx_i \right)
\end{align*}
Applying Lemma \ref{lem:normal-vec-norm}, we note that the following holds with probability at least $1 - \nicefrac{\delta}{3}$
\begin{align*}
    \norm{\wols - \wstar} &\leq \Tr(\vM) \left(1 + 8 \log(\nicefrac{3d}{\delta})\right) \\
    \vM &= \left(\sum_{i=1}^{n} \vx_i \vx^T_i \right)^{-1} \left(\sum_{i=1}^{n} {\dotp{\fstar}{\vx_i}}^2 \vx_i \vx^T_i\right)\left(\sum_{i=1}^{n} \vx_i \vx^T_i \right)^{-1}
\end{align*}
By suprema of subgaussian random variables, the following holds with probability at least $1 - \nicefrac{\delta}{3}$
\begin{align*}
    {\dotp{\fstar}{\vx_i}}^2 \leq \norm{\fstar}^2 \log(\nicefrac{3n}{\delta}) \ \forall i \in [n]
\end{align*}
Furthermore, by concentration of Wishart matrices \citep{vershynin2010introduction}, the following holds with probability at least $1 - \nicefrac{\delta}{3}$
\begin{align*}
    \lambda_{\min}\left(\sum_{i=1}^{n} \vx_i \vx^T_i\right) \geq n\left(1 - \sqrt{\frac{d\log(\nicefrac{3d}{\delta})}{n}}\right) \geq \nicefrac{n}{2}
\end{align*}
where we use the fact that $n \geq \Otilde(d)$. 
Hence, by a union bound, it follows that,
\begin{align*}
    \Tr\left(\vM\right) &\leq \norm{\fstar}^2 \log(\nicefrac{3n}{\delta}) \Tr\left(\left(\sum_{i=1}^{n} \vx_i \vx^T_i \right)^{-1}\right) \\
    &\leq \norm{\fstar}^2 \log(\nicefrac{3n}{\delta}) \frac{d}{\lambda_{\min}\left(\sum_{i=1}^{n} \vx_i \vx^T_i\right)} \\
    &\leq \norm{\fstar}^2 \log(\nicefrac{3n}{\delta}) \frac{2d}{n}
\end{align*}
Then, it follows that,
\begin{align*}
    \norm{\wols - \wstar}^2 \leq \norm{\fstar}^2\frac{d}{n} \polylog(\nicefrac{nd}{\delta})
\end{align*}
Finally, the guarantee that $\norm{\wols - \wstar}^2 = \Theta(\nicefrac{d\|\fstar\|^2}{n})$ with probability at least $1 - d/n^c$, $c \geq 1$ follows from Theorem 1 of \cite{chaudhuri2017active}.

\section{Analysis of Spectral Method -- Proof of Theorem~\ref{thm:spec-error-bound}}
\label{sec:spec-method-proof}
We shall use the fact that for any positive semidefinite matrix $\vA$, and event $E$ that occurs with probability at least $1 - \delta$ for any $\delta \leq \nicefrac{1}{2}$, $\bE\left[\vA | E \right] \preceq \frac{\bE[\vA]}{1 - \delta} \preceq 2 \bE[\vA]$. $\hvw$ of $\wstar$ and let $\Delta = \wstar - \hvw$. Furthermore, let $\ve = \nicefrac{\fstar}{\norm{\fstar}}$ and define the matrices $\hvS$, $\vS$ and $\Sigma$ as follows,
\begin{align*}
    \hvS &= \frac{1}{n} \sum_{i=1}^{n} \left(y_i - \dotp{\hvw}{\vx_i}\right)^2 \vx_i \vx^T_i \\
    \vS &= \frac{1}{n} \sum_{i=1}^{n} {\dotp{\fstar}{\vx_i}}^2 \vx_i \vx^T_i  \\
    \Sigma &= 3 \norm{\fstar}^2 \ve \ve^T + \norm{\fstar}^2 (\vI - \ve \ve^T)
\end{align*}
Using $y_i = \dotp{\wstar}{\vx_i} + \epsilon_i \dotp{\fstar}{\vx_i}$, and writing $\epsilon^2_i = 1 + z_i$, we expand $\hvS$ as follows,
 \begin{align*}
     \hvS &= \frac{1}{n} \sum_{i=1}^{n} \left(\dotp{\Delta}{\vx_i} + \epsilon_i \right)^2 \vx_i \vx^T_i \\
     &= \vS +\frac{1}{n} \sum_{i=1}^{n} {\dotp{\Delta}{\vx_i}}^2 \vx_i \vx^T_i + \frac{2}{n} \sum_{i=1}^{n}\epsilon_i \dotp{\Delta}{\vx_i} \dotp{\fstar}{\vx_i} \vx_i \vx^T_i 
     + \frac{1}{n}\sum_{i=1}^{n} z_i {\dotp{\fstar}{\vx_i}}^2 \vx_i \vx^T_i 
 \end{align*}
We note that $\fstar$ is the top eigenvector of $\Sigma$ with eigenvalue $3\norm{\fstar}^2$. Moreover, by definition $\fspec$ is the top eigenvector of $\hSigma$ with norm $\sqrt{\nicefrac{\norm{\hSigma}}{3}}$. To this end, our proof is divided into three distinct parts, namely, controlling $\norm{\vS - \Sigma}$ via Matrix Bernstein, bounding $\norm{\hvS - \vS}$ as a function of $\norm{\Delta}$, and finally, bounding $\norm{\fspec - \fstar}$ as a function of $\norm{\hvS - \Sigma}$ via Davis-Kahan theorem. 

\subsection{Controlling $\norm{\vS - \Sigma}$}
 Since $\vx_i \ \iidsim \ \cN(0, \vI)$, we note that $\bE[\hSigma]= \Sigma$. Thus we control $\norm{\hSigma - \Sigma}$ via Matrix Bernstein's inequality. To this end, define the matrix $\vA_i = {\dotp{\fstar}{\vx_i}}^2\vx_i \vx^T_i$. Let $E_1$ and $E_2$ denote the events $E_1 = \left\{ {\dotp{\fstar}{\vx_i}}^2 \leq \norm{\fstar}^2 \log(\nicefrac{4n}{\delta}) \ \forall \ i \in [n] \right\}$, $E_2 = \left\{ \norm{\vx_i}^2 \leq d + \log(\nicefrac{4n}{\delta}) \right\}$, and let $E = E_1 \cap E_2$.  Gaussian concentration implies $\bP(E) \geq 1 - \delta$. Moreover, conditioned on the event $E$, the following holds,
\begin{align*}
     \norm{\vA_i} &\leq {\dotp{\fstar}{\vx_i}}^2 \norm{\vx_i}^2 \leq \norm{\fstar}^2 \log(\nicefrac{4n}{\delta}) \left[d + \log(\nicefrac{4n}{\delta})\right] \leq \norm{\fstar}^2 d \polylog(\nicefrac{nd}{\delta})\\
     \bE\left[\vA_i \vA^T_i | E\right] &= \bE\left[{\dotp{\fstar}{\vx_i}}^4 \norm{\vx_i}^2 \vx_i \vx^T_i | E \right] \preceq \norm{\fstar}^4 \log^{2}(\nicefrac{4n}{\delta}) \left(d + \log(\nicefrac{4n}{\delta})\right) \bE\left[\vx_i \vx_i^T | E\right] \\
     &\preceq \vI d \norm{\fstar}^4 \polylog(\nicefrac{nd}{\delta})
\end{align*}
Notice that conditioned on $E$, notice that $\vA_i$ are still i.i.d random matrices. Hence, by Matrix Bernstein inequality, $P(E_3 | E) \geq 1 - \delta$ where $E_3$ is given by,
\begin{align*}
    E_3 &= \left\{ \norm{\hSigma - \bE\left[\vA_i | E \right]} \leq  \norm{\fstar}^2 \sqrt{\frac{d}{n}}\polylog(\nicefrac{nd}{\delta}) \right\}
\end{align*}
It follows that
\begin{align}
    \bP \left\{ \norm{\hSigma - \bE\left[\vA_i | E \right]} \leq \norm{\fstar}^2 \sqrt{\frac{d}{n}}\polylog(\nicefrac{nd}{\delta}) \right\} \geq 1 - 2\delta \label{eq:second_term_bound}
\end{align}
We now need to show that $\mathbb{E}\left[\vA_i|E\right] \approx \mathbb{E}\left[\vA_i\right]$. Consider:
\begin{align*}
&\|\bE\left[\vA_i|E\right] - \bE\left[\vA_i\right]\| = \norm{\frac{\bE\left[\vA_i\mathbbm{1}(E)\right]}{\bP(E)} - \bE\left[\vA_i\right]} = \norm{\frac{\bE(\vA_i) -\bE\left[\vA_i\mathbbm{1}(E^{\complement})\right]}{\bP(E)} - \bE\left[\vA_i\right]} \\
&\leq \frac{\bP(E^{\complement})\norm{\bE[\vA_i]}}{\bP(E)} + \frac{\norm{\bE[\vA_i \mathbbm{1}(E^{\complement})]}}{\bP(E)} \leq \frac{\bP(E^{\complement})\norm{\bE[\vA_i]}}{\bP(E)} + \frac{\sqrt{\norm{\bE[\vA_i^2]}\bP(E^{\complement})}}{\bP(E)} \\
&\leq C\|\fstar\|^2[\delta + \sqrt{d\delta}]
\end{align*}
The second inequality follows from the Cauchy-Schwarz inequality. The last inequality follows from the fact that $\|\mathbb{E}[\vA_i]\| = 3\|\fstar\|^2$ and $\|\mathbb{E}\vA_i^2\| \leq C \|\fstar\|^4 d$ for some constant $C > 0$. Replacing $\delta$ above with $\delta^2/\mathsf{poly}(nd)$, we conclude from Equation~\eqref{eq:second_term_bound} that:

\begin{align}
\label{eqn:spec-conc-bound-1}
    \bP \left\{ \norm{\hSigma - \Sigma} \leq \norm{\fstar}^2\sqrt{\frac{d}{n}}\polylog(\nicefrac{nd}{\delta}) \right\} \geq 1 - \delta
\end{align}
\subsection{Controlling $\norm{\hvS - \vS}$}
We control each term in $\norm{\hvS - \vS}$ as follows,
\subsubsection{Bounding $\norm{\nicefrac{1}{n} \sum_{i=1}^{n} {\dotp{\Delta}{\vx_i}}^2 \vx_i \vx^T_i}$} 
Let $E_1$ denote the event $E_1 = \left\{ {\dotp{\Delta}{\vx_i}}^2 \leq \norm{\Delta}^2 \log(\nicefrac{4n}{\delta}) \right\}$. Then, by suprema of Gaussian random variables, we know that $\bP(E_1) \geq 1 - \nicefrac{\delta}{2}$. Furthermore, define the event $E_2$ as follows,
 \begin{align*}
     E_2 = \left\{ \norm{\frac{1}{n} \sum_{i=1}^{n} \vx_i \vx^T_i - \vI} \leq 2 \sqrt{\frac{d}{n}} + 2t + \left(\sqrt{\frac{d}{n}} + t\right)^2\right\}
 \end{align*}
 where $t = \sqrt{\frac{2 \log(\nicefrac{4}{\delta})}{n}}$. Taking a union bound over $E_1$ and $E_2$, and using the concentration properties of Wishart matrices, we conclude that the following must hold with probability at least $1 - \delta$.
 \begin{align}
 \label{eqn:spec-conc-bound-2}
     \frac{1}{n} \sum_{i=1}^{n} {\dotp{\Delta}{\vx_i}}^2 \vx_i \vx^T_i &\preceq \frac{\norm{\Delta}^2 \log(\nicefrac{4n}{\delta})}{n} \sum_{i=1}^{n} \vx_i \vx^T_i \nonumber \\
     \norm{\frac{1}{n} \sum_{i=1}^{n} {\dotp{\Delta}{\vx_i}}^2 \vx_i \vx^T_i} &\leq \norm{\Delta}^2 \log(\nicefrac{4n}{\delta}) \left[1 + 2 \sqrt{\frac{d}{n}} + 2\sqrt{\frac{2 \log(\nicefrac{4}{\delta})}{n}} + \left(\sqrt{\frac{d}{n}} + \sqrt{\frac{2 \log(\nicefrac{4}{\delta})}{n}}\right)^2\right] \nonumber \\
     &\leq \norm{\Delta}^2 \polylog(\nicefrac{nd}{\delta})
 \end{align}
 where the last inequality uses the fact that $n \geq d \polylog(\nicefrac{nd}{\delta})$
 
\subsubsection{Bounding $\norm{\frac{2}{n} \sum_{i=1}^{n}\epsilon_i \dotp{\Delta}{\vx_i} \dotp{\fstar}{\vx_i} \vx_i \vx^T_i}$}
Define the events $E_1, E_2, E_3, E_4$ as,
\begin{align*}
    E_1 &= \left\{ \abs{\epsilon_i} \leq \sqrt{\log(\nicefrac{8n}{\delta})} \right\} \\
    E_2 &= \left\{ \abs{\dotp{\Delta}{\vx_i}} \leq \norm{\Delta}\sqrt{\log(\nicefrac{8n}{\delta})} \right\} \\
    E_3 &= \left\{ \abs{\dotp{\fstar}{\vx_i}} \leq \norm{\fstar}\sqrt{\log(\nicefrac{8n}{\delta})} \right\} \\
    E_4 &= \left\{\norm{\vx_i}^2 \leq d + \log(\nicefrac{8n}{\delta}) \right\}
\end{align*}
Let $E = E_1 \cap E_2 \cap E_3 \cap E_4$. It follows from Gaussian concentration that $P(E) \geq 1 - \delta$. We now follow the same steps as the above. In particular, let $\vB_i = \epsilon_i \dotp{\Delta}{\vx_i} \dotp{\fstar}{\vx_i} \vx_i \vx^T_i$. Then, conditioned on $E$,
\begin{align*}
    \norm{\vB_i} &\leq \abs{\epsilon_i} \abs{\dotp{\Delta}{\vx_i}} \norm{\vx_i}^2 \\
    &\leq \norm{\Delta} \norm{\fstar} \log^{\nicefrac{3}{2}}\left(\nicefrac{8n}{\delta}\right) \left(d + \log(\nicefrac{8n}{\delta})\right) \leq \norm{\fstar} \norm{\Delta} d \polylog(\nicefrac{nd}{\delta}) \\
    \bE\left[\vB_i \vB^T_i | E\right] &= \bE\left[\epsilon^2_i {\dotp{\Delta}{\vx_i}}^2 {\dotp{\fstar}{\vx_i}}^2 \vx_i \vx^T_i | E\right] \\
    &\preceq \norm{\Delta}^2 \norm{\fstar}^2 \log^3(\nicefrac{8n}{\delta}) \left(d + \log(\nicefrac{8n}{\delta})\right) \bE\left[\vx_i \vx^T_i | E\right] \\
    &\preceq \norm{\Delta}^2 \norm{\fstar}^2 \polylog(\nicefrac{nd}{\delta}) \vI
\end{align*}
Notice that $B_i$ are still i.i.d. when conditioned on $E$. Hence, by the matrix Bernstein inequality, $\frac{2}{n} \sum_{i=1}^{n}\epsilon_i \dotp{\Delta}{\vx_i} \dotp{\fstar}{\vx_i} \vx_i \vx^T_i$
\begin{align*}
    \bP \left\{ \norm{\frac{1}{n} \sum_{i=1}^{n}\epsilon_i \dotp{\Delta}{\vx_i} \dotp{\fstar}{\vx_i} \vx_i \vx^T_i - \bE\left[\vB_i | E \right]} \leq \norm{\fstar}\norm{\Delta}\sqrt{\frac{d}{n} }\polylog(\nicefrac{d}{\delta}) \right\} \geq 1 - 2\delta
\end{align*}
Note that $\mathbb{E}[\vB_i|E] = 0$ since $\mathbb{E}[\epsilon_i|E] = 0$. This allows us to conclude:

\begin{align}
\label{eqn:spec-conc-bound-3}
    \bP \left\{ \norm{\frac{1}{n} \sum_{i=1}^{n}\epsilon_i \dotp{\Delta}{\vx_i} \dotp{\fstar}{\vx_i} \vx_i \vx^T_i } \leq \norm{\fstar}\norm{\Delta}\sqrt{\frac{d}{n} } \polylog(\nicefrac{nd}{\delta}) \right\} \geq 1 - 2\delta
\end{align}
\subsubsection{Bounding $\norm{\nicefrac{1}{n}\sum_{i=1}^{n}z_i {\dotp{\fstar}{\vx_i}}^2 \vx_i \vx^T_i}$}
Let $\vC_i = z_i {\dotp{\fstar}{\vx_i}}^2 \vx_i \vx^T_i$. Define the events $E_1, E_2, E_3$ as
\begin{align*}
    E_1 &= \{ \abs{z_i} \leq \log(\nicefrac{6n}{\delta}) \} \\
    E_2 &= \left\{ {\dotp{\fstar}{\vx_i}}^2 \leq \norm{\fstar}^2 \log(\nicefrac{6n}{\delta}) \right\} \\
    E_3 &= \left\{ \norm{\vx_i}^2 \leq d + \log(\nicefrac{6n}{\delta}) \right\} 
\end{align*}
Let $E = E_1 \cap E_2 \cap E_3$. Then $P(E) \geq 1 - 3\delta$. Conditioning on $E$ and following the same steps as before,
\begin{align*}
    \norm{\vC_i} &\leq \norm{\fstar}^2 \log^2(\nicefrac{6n}{\delta})\left[d + \log(\nicefrac{6n}{\delta})\right] \\
    \bE\left[\vC_i \vC_i | E\right] &= \bE\left[z^2_i {\dotp{\fstar}{\vx_i}}^4 \norm{\vx_i}^2 \vx_i \vx^T_i | E\right] \\
    &\preceq 2 \norm{\fstar}^4 \log^4(\nicefrac{6n}{\delta})\left[d + \log(\nicefrac{6n}{\delta})\right] \vI 
\end{align*}
Hence, applying the matrix Bernstein inequality in a similar way as above, we obtain:
\begin{align}
    \bP\left\{ \norm{\frac{1}{n} \sum_{i=1}^{n} z_i  {\dotp{\fstar}{\vx_i}}^2 - \bE\left[\vC_i|E\right]} \leq \norm{\fstar}^2 \sqrt{\frac{d}{n}} \polylog(\nicefrac{nd}{\delta})\right\} \geq 1 - 2\delta \label{eq:fourth_term_bound}
\end{align}

Now, it remains to bound $\|\mathbb{E}[\vC_i|E]\|$. Notice that, when conditioned on $E$, $z_i$ and ${\dotp{\fstar}{\vx_i}}^2 \vx_i \vx^T_i$ are independent. Therefore, we conclude:

\begin{align*}
&\|\mathbb{E}[\vC_i|E]\| \lesssim |\mathbb{E}[z_i|E]| \|\fstar\|^2d\log(\tfrac{6n}{\delta})^2 = \frac{\|\fstar\|^2d\log(\tfrac{6n}{\delta})^2}{\bP(E)}|\mathbb{E}[z_i\mathbbm{1}(E)]| \nonumber \\
&= \frac{\|\fstar\|^2d\log(\tfrac{6n}{\delta})^2}{\bP(E)}|\mathbb{E}[z_i\mathbbm{1}(E^{\complement})]| \leq \frac{\|\fstar\|^2d\log(\tfrac{6n}{\delta})^2}{\bP(E)}\sqrt{\mathbb{E}z_i^2\bP(E^{\complement})} \nonumber \\
&\lesssim \frac{\|\fstar\|^2d\log(\tfrac{6n}{\delta})^2}{\bP(E)}|\mathbb{E}[z_i\mathbbm{1}(E^{\complement})]| \leq \|\fstar\|^2d\log(\tfrac{6n}{\delta})^2\sqrt{\delta} 
\end{align*}
In the second line we have used the fact that $\bE[z_i] = 0$ and hence $\bE[z_i \mathbbm{1}(E)] = - \bE[z_i \mathbbm{1}(E^{\complement})]$. We have also used the Cauchy-Scwharz inequality. Therefore, replacing $\delta$ with $\delta^2/\mathsf{poly}(nd)$ in the above discussion and using Equation~\eqref{eq:fourth_term_bound}, we conclude:
\begin{align}
\label{eqn:spec-conc-bound-4}
    \bP\left\{ \norm{\frac{1}{n} \sum_{i=1}^{n} z_i \dotp{\Delta}{\vx_i} \dotp{\fstar}{\vx_i} } \leq \norm{\fstar}^2 \sqrt{\frac{d}{n}} \polylog(\nicefrac{nd}{\delta})\right\} \geq 1 - \delta
\end{align}
From \eqref{eqn:spec-conc-bound-1}, \eqref{eqn:spec-conc-bound-2}, \eqref{eqn:spec-conc-bound-3} and \eqref{eqn:spec-conc-bound-4}, we finally conclude that with probability at least $1 - \delta$,
\begin{align}
\label{eqn:spec-conc-bound-5}
    \norm{\hvS - \Sigma} \leq \norm{\hvS - \vS} + \norm{\vS - \Sigma} \leq \left(\norm{\Delta}^2 + \left(\norm{\fstar}^2 + \norm{\fstar}\norm{\Delta}\right) \sqrt{\frac{d}{n}}\right)\polylog(\nicefrac{nd}{\delta})
\end{align}
\subsection{Controlling $\norm{\hvf - \fstar}^2$}
Let $\fspec$ be the top eigenvector of $\hvS$ with $\norm{\fspec} = \sqrt{\nicefrac{\norm{\hvS}}{3}}$. Let $\theta$ be the angle between $\fspec$ and $\fstar$. We assume $\fspec$ is aligned with $\fstar$, i.e., $\theta \in [-\nicefrac{\pi}{2}, \nicefrac{\pi}{2}]$. This assumption is without loss of generality because the heteroscedastic regression model is invariant to the sign of $\fstar$. Let $\vg$ be a unit vector orthogonal to $\fstar$ such that $\fspec = \norm{\fspec} \cos{\theta} \ve + \norm{\fspec} \sin{\theta} \vg$. Moreover, let $\vv = \norm{\fstar} \cos{\theta} \ve + \norm{\fstar} \sin{\theta} \vg$. Since $\Sigma$ has a spectral gap of $2 \norm{\fstar}^2$, it follows from the Davis Kahan theorem that,
\begin{align}
\label{eqn:spec-conc-bound-6}
    \norm{\vv - \fstar} \leq \frac{\sqrt{2} \norm{\hvS - \Sigma}}{\norm{\fstar}^2}
\end{align}
Furthermore, since $\|\hvS\| = 3\|\fspec\|^2$, $\|\hSigma\| = 3\|\fstar\|^2$ and $\norm{\hvS} \leq \norm{\Sigma} + \norm{\hvS - \Sigma}$, we conclude that, 
\begin{align*}
    \norm{\fspec}^2 &\leq \norm{\fstar}^2 \left(1 + \frac{\norm{\hvS - \Sigma}}{3\norm{\fstar}^2}\right) \\
    \norm{\fspec} &\leq \norm{\fstar} \left(1 + \frac{\norm{\hvS - \Sigma}}{3\norm{\fstar}^2}\right)^{\nicefrac{1}{2}} \\
    &\leq \norm{\fstar} \left(1 + \frac{\norm{\hvS - \Sigma}}{3\norm{\fstar}^2}\right) \\
    \norm{\fspec} - \norm{\fstar} &\leq \frac{\norm{\hvS - \Sigma}}{3\norm{\fstar}}
\end{align*}
where we use the fact that $\sqrt{1+x} \leq 1+x$ for any $x \geq 0$. We follow similar steps to lower bound $\norm{\fspec} - \norm{\fstar}$. In particular, since $\norm{\hvS} \geq \norm{\Sigma} - \norm{\hvS - \Sigma}$, we infer that,
\begin{align*}
    \norm{\fspec}^2 &\geq \norm{\fstar}^2 \left(1 - \frac{\norm{\hvS - \Sigma}}{3\norm{\fstar}^2}\right) 
\end{align*}
Since $\sqrt{1-t} \geq 1-t \ \forall \ t \in [0, 1]$, and $\|\hSigma\|,\|\fspec\| \geq 0$, we conclude from the above inequality that the following must hold
\begin{align*}
    \norm{\fspec} &\geq \norm{\fstar} \left(1 - \frac{\norm{\hvS - \Sigma}}{3\norm{\fstar}^2}\right) \\
    \norm{\fspec} - \norm{\fstar} &\geq -\frac{\norm{\hvS - \Sigma}}{3\norm{\fstar}}
\end{align*}
Hence, $\norm{\hvf - \vv} = \abs{\norm{\fspec} - \norm{\fstar}} \leq \frac{\norm{\hvS - \Sigma}}{3\norm{\fstar}}$. From \eqref{eqn:spec-conc-bound-5} and \eqref{eqn:spec-conc-bound-6}, we obtain the following with probability at least $1 - \delta$
\begin{align*}
    \norm{\hvf - \fstar}^2 &\leq \frac{5 \norm{\hvS - \Sigma}^2}{\norm{\fstar}^2} \\
    &\leq \left(\frac{\norm{\Delta}^4}{\norm{\fstar}^2} + \left(\norm{\fstar}^2 + \norm{\Delta}^2\right) \frac{d}{n}\right)\polylog(\nicefrac{nd}{\delta}) \\
    &\leq \left(\frac{\epsilon^2}{\norm{\fstar}^2} + \left(\norm{\fstar}^2 + \epsilon\right) \frac{d}{n}\right)\polylog(\nicefrac{nd}{\delta})
\end{align*}

\section{Phase Retrieval with Multiplicative Noise}
\label{sec:phs_rtr_pf}
\subsection{Derivation of the Pseudo Gradient}

First, we will derive the Pseudo Graident $\cG$ from the fictitious square loss. Suppose we had access to $\wstar$ and $\fstar$. Then $(y_i - \langle \wstar,\vx_i\rangle)^2 = \epsilon_i^2 \langle \fstar,\vx_i\rangle^2$. 
$\mathbb{E}\left[\epsilon_i^2 \langle \fstar,\vx_i\rangle^2|\vx_i\right] = \langle \fstar,\vx_i\rangle^2$ and $\mathsf{var}(\epsilon_i^2 \langle \fstar,\vx_i\rangle^2|\vx_i) =  2\langle \fstar,\vx_i\rangle^4$. The (fictitious) square loss function which for recovering $f$ would be:

$$\lossmul(\vf) = \frac{1}{m}\sum_{i=1}^{m}\frac{\left[(y_i - \langle \wstar,\vx_i\rangle)^2 - \langle \vf,\vx_i\rangle^2\right]^2}{\langle \fstar,\vx_i\rangle^4}$$

Now, the actual gradient of this fictitious loss is: $$\nabla \lossmul(\vf) = \frac{2}{m}\sum_{i=1}^{m}\frac{\left[\langle \vf,\vx_i\rangle^2 - (y_i - \langle \wstar,\vx_i\rangle)^2 \right]\langle \vf,\vx_i\rangle \vx_i}{\langle \fstar,\vx_i\rangle^4}$$

Note that this loss function cannot be computed. Assuming we have a good enough estimate $\hat{\vf} \approx \fstar$ and $\hat{\vw}\approx \wstar$, and that $\vf \approx \hat{\vf}$, we make replace $\fstar$ with $\hat{\vf}$ and $\wstar$ with $\hat{\vw}$ and for the sake of convenience $\langle f,\vx_i\rangle$ with $\langle\hat{\vf},\vx_i\rangle$. With these approximation, we obtain:

$$\bar{\cG}(\vf) := \frac{1}{m}\sum_{i=1}^{m}\frac{\left[\langle \vf,\vx_i\rangle^2 - (y_i - \langle \hat{\vw},\vx_i\rangle)^2 \right]\langle \hat{\vf},\vx_i\rangle \vx_i}{\langle \hat{\vf},\vx_i\rangle^4}$$

In order to prevent the denominator from exploding, we add the `regularization' $\mathbbm{1}(|\langle\hat{\vf},\vx_i\rangle| \geq \bar{\mu})$ to obtain the defined pseudo-gradient:

$$\cG(\vf) := \frac{1}{m}\sum_{i=1}^{m}\mathbbm{1}(|\langle\hat{\vf},\vx_i\rangle| \geq \bar{\mu})\frac{\left[\langle \vf,\vx_i\rangle^2 - (y_i - \langle \hat{\vw},\vx_i\rangle)^2 \right]\langle \hat{\vf},\vx_i\rangle \vx_i}{\langle \hat{\vf},\vx_i\rangle^4}$$

\subsection{Noise-Contraction Decomposition of the Pseudo Gradient}
We now give the noise-contraction decomposition of the pseudo-gradient $\cG_t(\vf)$, where we write $\cG_t(\vf) = H_t(\vf)(\vf-\fstar) +N_t$ where $H_t(\vf)$ is a PSD matrix whenever $\vf$ is close to $\fstar$ and $N_t$ is the noise term. For the sake of clarity, only in this derivation, take $\gamma(\vx_i) := \frac{\mathbbm{1}(|\langle\hat{\vf},\vx_i\rangle|\geq \bar{\mu})}{\langle \hat{\vf},\vx_i\rangle^4}$.

\begin{align}
&\cG_t(\vf) = \frac{1}{m}\sum_{i=1}^{m}\gamma(\vx^{(t)}_i)(\langle \vf,\vx^{(t)}_i\rangle^2 - (Y_i-\langle\hat{\vw},\vx_i\rangle)^2)\vx_i \langle \hat{\vf},\vx^{(t)}_i \rangle \nonumber \\
&= \frac{1}{m}\sum_{i=1}^{m}\gamma(\vx^{(t)}_i)(\langle \vf,\vx^{(t)}_i\rangle^2 - \langle \fstar,\vx^{(t)}_i\rangle^2)\vx^{(t)}_i \langle \hat{\vf},\vx^{(t)}_i \rangle \nonumber \\&\quad + \frac{1}{m}\sum_{i=1}^{m}\gamma(\vx^{(t)}_i)(\langle \fstar,\vx^{(t)}_i\rangle^2 - (Y_i-\langle\hat{\vw},\vx^{(t)}_i\rangle)^2)\vx^{(t)}_i \langle \hat{\vf},\vx^{(t)}_i \rangle \label{eq:pseud_grad}
\end{align}

Define $N_t = \frac{1}{m}\sum_{i=1}^{m}\gamma(\vx^{(t)}_i)(\langle \fstar,\vx^{(t)}_i\rangle^2 - (Y_i-\langle\hat{\vw},\vx^{(t)}_i\rangle)^2 )\vx^{(t)}_i \langle \hat{\vf},\vx^{(t)}_i \rangle$. Now, consider the first term in Equation~\eqref{eq:pseud_grad}. We have:
\begin{align}
   &\frac{1}{m}\sum_{i=1}^{m}\gamma(\vx^{(t)}_i)(\langle \vf,\vx^{(t)}_i\rangle^2 - \langle \fstar,\vx^{(t)}_i\rangle^2)\vx^{(t)}_i \langle \hat{\vf},\vx^{(t)}_i \rangle \nonumber \\ &=  \frac{1}{m}\sum_{i=1}^{m}\gamma(\vx^{(t)}_i)\langle \vf - \fstar,\vx^{(t)}_i\rangle \langle f+ \fstar,\vx^{(t)}_i\rangle \vx^{(t)}_i \langle \hat{\vf},\vx^{(t)}_i \rangle\nonumber \\
 &=: H_t(\vf) (\vf - \fstar)
\end{align}
Where we define $H_t(\vf) := \frac{1}{m}\sum_{i=1}^{m}\gamma(\vx^{(t)}_i) \langle \vf+ \fstar,\vx^{(t)}_i\rangle \langle \hat{\vf},\vx^{(t)}_i \rangle \vx^{(t)}_i (\vx^{(t)}_i)^{\intercal}$. Plugging these into Equation~\eqref{eq:pseud_grad}, we conclude that:

\begin{equation}\label{eq:noise_contr}
    \cG_t(\vf) = H_t(\vf)(\vf-\fstar) + N_t
\end{equation}

\subsection{Bounding the Contraction Matrix and the Noise Vector}
In this section we will state results regarding the matrix $H_t(\vf)$ and the vector $N_t$ which will aid us in proving the convergence bounds. In 
this subsection only, for the sake of clarity, we drop the dependence on $t$ and let $(\vx_1,y_1),\dots,(\vx_m,y_m)$ to be derived from the model specified in Section~\ref{sec:setup}. We let $\Delta = \wstar - \hat{\vw}$ and $\Gamma = \max\left(\|\fstar - \hat{\vf}\|, \|\vf - \hat{\vf}\|\right)$. 

The following lemma, which we state without proof, follows from Gaussian concentration.
    
\begin{lemma}\label{lem:close_conc}
For any fixed vector $v \in \mathbb{R}^d$, we have $\bP(\sup_i |\langle v,\vx_i\rangle| > t\|v\|) \leq 2m \exp(-\tfrac{t^2}{2})$.
\end{lemma}

First, we turn our attention to the noise term. For the sake of clarity, we will take $\gamma(\vx_i) := \frac{\mathbbm{1}(|\langle\hat{\vf},\vx_i\rangle|\geq \bar{\mu})}{\langle \hat{\vf},\vx_i\rangle^4} $. The noise term can be written as:

\begin{align}
    N &= \frac{1}{m}\sum_{i=1}^{m}\gamma(\vx_i)(\langle \fstar,\vx_i\rangle^2(1 -\epsilon_i^2)\vx_i \langle \hat{\vf},\vx_i \rangle + \frac{2}{m} \sum_{i=1}^{m}\gamma(\vx_i)\epsilon_i\langle \fstar,\vx_i\rangle \langle\Delta,\vx_i\rangle \vx_i \langle \hat{\vf},\vx_i \rangle \nonumber \\
    &\quad - \frac{1}{m} \sum_{i=1}^{m}\gamma(\vx_i) \langle\Delta,\vx_i\rangle^2 \vx_i \langle \hat{\vf},\vx_i \rangle
\end{align}

Consider \begin{align}
\bar{N}_1 := \frac{1}{m}\sum_{i=1}^{m}\gamma(\vx_i)(\langle \fstar,\vx_i\rangle^2(1 -\epsilon_i^2)\vx_i \langle \hat{\vf},\vx_i \rangle\nonumber \\
\bar{N}_2 := \frac{2}{m} \sum_{i=1}^{m}\gamma(\vx_i)\epsilon_i\langle \fstar,\vx_i\rangle \langle\Delta,\vx_i\rangle \vx_i \langle \hat{\vf},\vx_i \rangle \nonumber \\
\bar{N}_3 := \frac{1}{m} \sum_{i=1}^{m}\gamma(\vx_i) \langle\Delta,\vx_i\rangle^2 \vx_i \langle \hat{\vf},\vx_i \rangle 
 \label{eq:noise_decomp}
 \end{align}

Lemmas~\ref{lem:N1_bound},~\ref{lem:N2_bound} and~\ref{lem:N3_bound} bound this quantity. We refer to Section~\ref{sec:tech_lem_proof} for their proof. 

\begin{lemma}\label{lem:N1_bound}
Suppose $u \in \mathbb{R}^d$ and $\|\hat{\vf}\| > \bar{\mu}$:
\begin{enumerate}
    \item Let $h_1(X,u) = \frac{1}{m}\sum_{i=1}^{m} \mathbbm{1}(|\langle \hat{\vf},\vx_i\rangle|>\bar{\mu})\frac{\langle \fstar,\vx_i\rangle^4 \langle \vx_i,u\rangle^2}{\langle \hat{\vf},\vx_i\rangle^6}$
    $$\bP\left(|\langle \bar{N}_1,u\rangle|>c_0 t \sqrt{\tfrac{h_1(X,u)}{m}}\biggr|\vx_1,\dots,\vx_m\right) \leq \exp(-t) $$
    \item $$\bP\left(h_1(X,\hat{\vf}) > C(1 +  \tfrac{\Gamma^4}{\bar{\mu}^4}\log^2(\tfrac{m}{\delta}))\right) \leq \delta$$
    \item Suppose $u \perp \hat{\vf}$ and $\|u\|=1$. Let $r = \bigr\lceil \log(\tfrac{\|\hat{\vf}\|_2}{\bar{\mu}})\bigr\rceil$.  Assume $m \geq C\frac{\|\hat{\vf}\|}{\bar{\mu}}\log(\tfrac{r}{\delta})$ for some large enough constant $C$. Then:
    $$\bP\left(h_1(X,u) > \tfrac{C}{\bar{\mu}\|\hat{\vf}\|} \bigr(1 +  \tfrac{\Gamma^4}{\bar{\mu}^4}\log^2(\tfrac{m}{\delta})\bigr)\right) \leq \delta$$
\end{enumerate}
\end{lemma}

\begin{lemma}\label{lem:N2_bound} Let $u \in \mathbb{R}^d$ be any arbitrary vector. Assume $\|\hat{\vf}\| > \bar{\mu}$. Let $r = \bigr\lceil \log(\tfrac{\|\hat{\vf}\|_2}{\bar{\mu}})\bigr\rceil$. Define $h_2(X,u) := \frac{1}{m}\sum_{i=1}^{m} \mathbbm{1}(|\langle \hat{\vf},\vx_i\rangle|>\bar{\mu})\frac{\langle f^{*},\vx_i\rangle^2 \langle \vx_i,u\rangle^2 \langle\Delta ,\vx_i\rangle^2}{\langle \hat{\vf},\vx_i\rangle^6}$
\begin{enumerate}
    \item $$\bP\left(|\langle \bar{N}_2,u\rangle|> C_1\sqrt{\tfrac{h_2(X,u)\log\tfrac{1}{\delta}}{m}} \biggr|\vx_1,\dots,\vx_m \right) \leq \delta$$
    \item Assume $m \geq C\frac{\|\hat{\vf}\|}{\bar{\mu}}\log(\tfrac{r}{\delta})$ for some large enough constant $C$. $$\bP\left(h_2(X,\hat{\vf}) > C_1 \frac{\|\Delta\|^2}{\bar{\mu}\|\hat{\vf}\|} \log(\tfrac{m}{\delta})\left(1+\tfrac{\Gamma^2\log(\tfrac{m}{\delta})}{\bar{\mu}^2}\right)\right) \leq \delta$$
    \item Let $u \perp \hat{\vf}$ and $\|u\|=1$. Assume $m \geq C\frac{\|\hat{\vf}\|}{\bar{\mu}}\log(\tfrac{r}{\delta})$ for some large enough constant $C$.
    $$\bP\left(h_2(X,u) >  C_1 \frac{\|\Delta\|^2}{\bar{\mu}^3\|\hat{\vf}\|} \log^2(\tfrac{m}{\delta})\left(1+\tfrac{\Gamma^2\log(\tfrac{m}{\delta})}{\bar{\mu}^2}\right)\right)\leq \delta$$
\end{enumerate}
\end{lemma}

\begin{lemma}\label{lem:N3_bound}
Suppose $\|\hat{\vf}\| >\bar{\mu}$. Let $r = \bigr\lceil \log(\tfrac{\|\hat{\vf}\|_2}{\bar{\mu}})\bigr\rceil$.  Assume $m \geq C\frac{\|\hat{\vf}\|}{\bar{\mu}}\log(\tfrac{r}{\delta})$,  for some large enough constant $C$.
\begin{enumerate}
    \item $$\bP\left(|\langle \bar{N}_3,\hat{\vf}\rangle| > C\frac{\|\Delta\|^2}{\bar{\mu}\|\hat{\vf}\|} \right) \leq \delta$$
    \item Let $Q$ be the projector onto the subspace perpendicular to $\hat{\vf}$. Then, we must have:
    $$\|Q\bar{N}_3\| \leq \frac{C\|\Delta\|^2}{\bar{\mu}\|\hat{\vf}\|^2} + \sqrt{\frac{C\|\Delta\|^4 \log^3(\tfrac{m}{\delta})(d+\log(\tfrac{m}{\delta}))}{m\bar{\mu}^5\|\hat{\vf}\|}}$$  
\end{enumerate}
\end{lemma}

While $H(\vf)$ need not be a PSD matrix almost surely, we show that whenever $\Gamma$ is small enough, $H(\vf)$ is a PSD matrix with high probability as shown below.

\begin{lemma}\label{lem:psd_bounds}
Suppose  with $\bar{\mu} > C\Gamma \sqrt{\log \tfrac{m}{\delta}} $ for some large enough universal constant $C$. 
With probability at-least $1-\delta$, we must have:
$$ \frac{1}{m}\sum_{i=1}^{m}\frac{\mathbbm{1}(|\langle\hat{\vf},\vx_i\rangle|\geq \bar{\mu})}{\langle \hat{\vf},\vx_i\rangle^4} |\langle \hat{\vf},\vx_i \rangle|^2 \vx_i \vx_i^{\intercal} \preceq H(f) \preceq \frac{4}{m}\sum_{i=1}^{m}\frac{\mathbbm{1}(|\langle\hat{\vf},\vx_i\rangle|\geq \bar{\mu})}{\langle \hat{\vf},\vx_i\rangle^4} |\langle \hat{\vf},\vx_i \rangle|^2 \vx_i \vx_i^{\intercal}$$
\end{lemma}

\begin{proof}
Only in this proof, we take $\gamma(\vx_i) := \frac{\mathbbm{1}(|\langle\hat{\vf},\vx_i\rangle|\geq \bar{\mu})}{\langle \hat{\vf},\vx_i\rangle^4} $. 
\begin{equation}
    H(\vf)= \frac{2}{m}\sum_{i=1}^{m}\gamma(\vx_i)  \langle \hat{\vf},\vx_i \rangle^2 \vx_i \vx_i^{\intercal} + \frac{1}{m}\sum_{i=1}^{m}\gamma(\vx_i) \langle \vf+ \fstar -2\hat{\vf},\vx_i\rangle \langle \hat{\vf},\vx_i \rangle \vx_i \vx_i^{\intercal}
\end{equation}

By an application of Lemma~\ref{lem:close_conc}, we conclude that with probability at-least $1-\delta$, we must have: $\sup_{i \in [m]}|\langle \vf + \fstar - \hat{\vf},\vx_i\rangle| \leq C_0\Gamma\sqrt{\log \tfrac{m}{\delta}}$ for some universal constant $C_0$. Therefore, by the definition of $H(\vf)$, we conclude that under this event:
\begin{align}
    H(\vf) &\preceq \frac{2}{m}\sum_{i=1}^{m}\gamma(\vx_i)  \langle \hat{\vf},\vx_i \rangle^2 \vx_i \vx_i^{\intercal} + \frac{1}{m}\sum_{i=1}^{m}C_0\gamma(\vx_i)\Gamma \sqrt{\log\tfrac{m}{\delta}}   |\langle\hat{\vf},\vx_i \rangle| \vx_i \vx_i^{\intercal} \nonumber \\
    &= \frac{2}{m}\sum_{i=1}^{m}\gamma(\vx_i)  |\langle \hat{\vf},\vx_i \rangle|\left(|\langle\hat{\vf},\vx_i \rangle| + C_1\Gamma \sqrt{\log(\tfrac{m}{\delta})}\right) \vx_i \vx_i^{\intercal} \nonumber \\
    &\preceq \frac{4}{m}\sum_{i=1}^{m}\gamma(\vx_i)  |\langle \hat{\vf},\vx_i \rangle|^2 \vx_i \vx_i^{\intercal}
\end{align}

In the third step, we have used the fact that $\gamma(\vx_i) \neq 0$ iff $|\langle\hat{\vf},\vx_i\rangle| \geq \bar{\mu}$ therefore, whenever $\gamma(\vx_i)\neq 0$, by the assumption in the statement of the lemma, we must have: $|\langle\hat{\vf},\vx_i \rangle| + C_1\Gamma \sqrt{\log(\tfrac{m}{\delta})} \leq 2|\langle\hat{\vf},\vx_i \rangle|$. The lower bounds follow in a similar way. The lower bound can be shown in a similar way by considering $ |\langle \hat{\vf},\vx_i\rangle| - C_1\Gamma \sqrt{\log(\tfrac{m}{\delta})} \geq\tfrac{1}{2} |\langle \hat{\vf},\vx_i\rangle| $. 
\end{proof}

The following result states some useful bounds for $H(\vf)$. We refer to Section~\ref{subsec:irritate_proof} for the proof of this lemma. 

\begin{lemma}\label{lem:irritating_bounds}
Assume $\bar{\mu} < \|\hat{\vf}\|$. Let $r := \bigr\lceil \log(\tfrac{\|\hat{\vf}\|_2}{\bar{\mu}})\bigr\rceil$, $\bar{\mu} > C\Gamma \sqrt{\log \tfrac{m}{\delta}} $ and $m \geq C\frac{\|\hat{\vf}\|}{\bar{\mu}}\log(\tfrac{r}{\delta})$ for some large enough universal constant $C$. Then, the following bounds hold:

\begin{enumerate}
    \item 

$$\bP\left( \tfrac{1}{2} \leq \langle \hat{\vf},H(\vf)\hat{\vf}\rangle \leq 8\right) \geq 1-2\exp(-c_0 m )$$

\item Let $u \perp \hat{\vf}$ be such that $\|u\| = 1$ 

$$ \bP\left(|\langle u,H(\vf)\hat{\vf}\rangle| \leq \frac{C\Gamma}{\bar{\mu}\|\hat{\vf}\|} + \sqrt{\frac{C\log(\tfrac{1}{\delta})}{m\bar{\mu}\|\hat{\vf}\|}} \right) \geq 1-\delta$$

\item Let $\hat{\vf}^{\perp} \in \mathsf{span}(\vf+\fstar-2\hat{\vf},\hat{\vf})$ and $\hat{\vf}^{\perp}\perp \hat{\vf}$, $\|\hat{\vf}^{\perp}\|=1$. Suppose $u \perp \hat{\vf}^{\perp}$ and $u\perp \hat{\vf}$.

$$ \bP\left(|\langle u,H(\vf)\hat{\vf}\rangle| \leq \Gamma\sqrt{\frac{C\log(\tfrac{1}{\delta})}{m\bar{\mu}^3\|\hat{\vf}\|}}   + \sqrt{\frac{C\log(1/\delta)}{m\bar{\mu}\|\hat{\vf}\|}} \right) \geq 1-\delta$$

\item

$$ \bP\left(\frac{c_0}{\|\hat{\vf}\|\bar{\mu}} \leq \langle u,H(\vf) u \rangle  \leq \frac{c_1}{\|\hat{\vf}\|\bar{\mu}}\right) \geq 1- \delta$$
\item Let $Q$ be the projector to the sup-space perpendicular to $\hat{\vf}$. With probability at-least $1-\delta$, we must have:
  $$\|QH(\vf) \hat{\vf}\|^2 \leq \frac{Cd\log(d/\delta)}{m\bar{\mu}\|\hat{\vf}\|} + \frac{Cd\Gamma^2\log^2(\tfrac{d}{\delta})}{m\bar{\mu}^3\|\hat{\vf}\|} + \frac{C\Gamma^2}{\bar{\mu}^2\|\hat{\vf}\|^2}  $$
  \item Let $Q$ be the projector to the sub-space perpendicular to $\hat{\vf}$ and let $u \perp \hat{\vf}$ such that $\|u\|=1$. With probability at-least $1-\delta$:

$$\| QH(\vf) u\| \leq  \frac{C}{\|\hat{\vf}\|\bar{\mu}} + \sqrt{\frac{Cd\log^2(\tfrac{4md}{\delta})}{m\bar{\mu}^3\|\hat{\vf}\|}} + \sqrt{\frac{Cd\Gamma^2 \log(\tfrac{4d}{\delta})\log^3(\tfrac{8m}{\delta})}{m\bar{\mu}^5\|\hat{\vf}\|}}$$

\end{enumerate}
\end{lemma}

\subsection{Proof of Theorem~\ref{thm:phase_r_mult}}

We are now ready to prove Theorem~\ref{thm:phase_r_mult}. Under the assumption of the theorem, $(\vx_i^{(t)},y_i^{(t)})$ are all i.i.d. Therefore, the iterate $\vf_t$ is independent of $(\vx_i^{(t)},y_i^{(t)})$ for $i \in [m]$. 

Now, consider the dynamics given in Algorithm~\ref{alg:phase_retrieve}. Using the noise-contraction decomposition (Equation~\eqref{eq:noise_contr}), we conclude:
$\vf_{t+1}-\fstar = (I-DH_t(\vf_t))(\vf_t-\fstar) + DN_t$.

Now, write $\vf_t = a_t \frac{\hat{\vf}}{\|\hat{\vf}\|} + b_t u_t$ where $u_t \perp \hat{\vf}$. Similarly, write $N_t = \bar{a}_t \frac{\hat{\vf}}{\|\hat{\vf}\|} + \bar{b}_t \bar{u}_t$ with $\bar{u}_t \perp \hat{\vf}$

Recall that $\hat{\vf} \in \mathsf{span}(P)$ and $u_t,\bar{u}_t \in \mathsf{span}(Q)$, where $P = \tfrac{\hvf \hvf^T}{\|\hvf\|^2}$ and $Q = I - P$. Now, let us track the evolution of $a_t,b_t$
Now, 
\begin{align}
    P(\vf_{t+1}-\fstar) &= a_{t+1}\frac{\hat{\vf}}{\|\hat{\vf}\|} = P(I-\vD H_t(\vf_t))(\vf_t - \fstar) + P\vD N_t \nonumber \\
    &= P(I-\alpha_0 H_t)(\vf_t - \fstar) + \alpha_0 PN_t \nonumber \\
    &= P(I-\alpha_0 H_t(\vf_t)) P (\vf_t - \fstar) + P(I-\alpha_0 H_t(\vf_t)) Q (\vf_t - \fstar)+  \alpha_0 \bar{a}_t \frac{\hat{\vf}}{\|\hat{\vf}\|}\nonumber \\
    &= a_t P(I-\alpha_0 H_t(\vf_t)) P \frac{\hat{\vf}}{\|\hat{\vf}\|} + b_t P(I-\alpha_0 H_t(\vf_t)) Qu_t+  \alpha_0 \bar{a}_t \frac{\hat{\vf}}{\|\hat{\vf}\|}\nonumber \\
    &= a_t P(I-\alpha_0 H_t(\vf_t)) P \frac{\hat{\vf}}{\|\hat{\vf}\|} - \alpha_0b_t P H_t(\vf_t) Qu_t+  \alpha_0 \bar{a}_t \frac{\hat{\vf}}{\|\hat{\vf}\|} \label{eq:proj_evol_1}
\end{align}

In the second line we have used the fact that $P\vD = \alpha_0 P$ by the definition of $\vD$. In the third line, we have used the fact that $P + Q = \vI $. In the fourth line, we have used the fact that $PQ = 0$. Similarly, interchanging $P$ and $Q$, we get

\begin{equation}\label{eq:proj_evol_2}
    Q(\vf_{t+1}-\fstar) = b_t Q(I-\alpha_1 H_t(\vf_t)) Q u_t -\alpha_1 a_t Q H_t(\vf_t) \frac{\hat{\vf}}{\|\hat{\vf}\|}+ \alpha_1 \bar{b}_t \bar{u}_t
\end{equation}

Now, define $\Gamma_t := \max(\|\fstar-\hat{\vf}\|,\|\vf_t - \hat{\vf}\|)$.

We state the following lemmas which are proved in Sections~\ref{subsec:P_contr_proof} and~\ref{subsec:Q_contr_proof}. 
\begin{lemma}\label{lem:P_contraction}
Assume $\bar{\mu} < \|\hat{\vf}\|$. Let $r := \bigr\lceil \log(\tfrac{\|\hat{\vf}\|_2}{\bar{\mu}})\bigr\rceil$, $\bar{\mu} > \cdist \Gamma_t \log \tfrac{m}{\delta} $, $m \geq \cparam \frac{\|\hat{\vf}\|}{\bar{\mu}}\log(\tfrac{r}{\delta})$, $\cparam\|\Delta\|\log(\tfrac{m}{\delta}) \leq \bar{\mu}$ for some large enough constants $\cparam,\cdist$, $\alpha_0 \leq c_1 \|\hat{\vf}\|^2$ for some small enough constant $c_1$. Then, conditioned on $\vf_t$ being such that the above conditions hold, we have with probability at-least $1-\delta$:
\begin{enumerate}
    \item For some $\ccontr$ which does not depend on $\cparam,\cdist$, we must have: 
    $$\biggr\|P(I-\alpha_0 H_t(\vf_t))P\frac{\hat{\vf}}{\|\hat{\vf}\|}\biggr\|^2 \leq \left(1- \tfrac{\ccontr\alpha_0 }{\|\hat{\vf}\|^2}\right)$$
    \item For some $C$ which does not depend on $\cparam,\cdist$
    
    $$\|PH_t(\vf_t)Q u_t\| \leq \frac{C\Gamma_t}{\bar{\mu}\|\hat{\vf}\|^2} + \sqrt{\frac{C\log(\tfrac{1}{\delta})}{m\bar{\mu}\|\hat{\vf}\|^3}}  $$
    
    \item For some $C$ which does not depend on $\cparam,\cdist$
    \begin{align}
       |\bar{a}_t| &\leq  C\left[\frac{\log(\tfrac{1}{\delta})}{\|\hat{\vf}\|\sqrt{m}} + \frac{\|\Delta\|^2}{\bar{\mu}\|\hat{\vf}\|^2}\right] 
    \end{align}
\end{enumerate}

\end{lemma}

\begin{lemma}\label{lem:Q_contraction}
Assume $\bar{\mu} < \|\hat{\vf}\|$ and let $r := \bigr\lceil \log(\tfrac{\|\hat{\vf}\|_2}{\bar{\mu}})\bigr\rceil$, $\bar{\mu} > \cdist\Gamma_t \log \tfrac{m}{\delta} $, $$m \geq \cparam \max\left(\frac{\|\hat{\vf}\|}{\bar{\mu}}\log(\tfrac{r}{\delta}),\frac{\|\hat{\vf}\|d\log^4(\tfrac{m}{\delta})}{\bar{\mu}}, \frac{\|\hat{\vf}\|^2\log^4(\tfrac{m}{\delta})}{\bar{\mu}^2}\right)\,,$$ $\cparam\|\Delta\|\log(\tfrac{m}{\delta}) \leq \bar{\mu}$ for some large enough universal constants $\cparam,\cdist$, $\alpha_1\leq c_1 \|\hat{\vf}\|\bar{\mu}$ for some small enough constant $c_1$. Then, conditioned on $\vf_t$ being such that the above conditions hold, we have with probability at-least $1-\delta$:

\begin{enumerate}
    \item For some $\ccontr$ which does not depend on $\cparam,\cdist$, we must have:   
    $$\|Q(I-\alpha_1H_t(\vf_t))Qu_t\|^2 \leq \left(1-\tfrac{\ccontr\alpha_1}{\|\hat{\vf}\|\bar{\mu}}\right)$$
    \item For some $C$ which does not depend on $\cparam,\cdist$
    $$\biggr\|QH_t(\vf_t)P\frac{\hat{\vf}}{\|\hat{\vf}\|}\biggr\|^2 \leq \frac{Cd\log(\tfrac{d}{\delta})}{m\bar{\mu}\|\hat{\vf}\|^3}  + \frac{C\Gamma^2_t}{\bar{\mu}^2\|\hat{\vf}\|^4} $$
    \item For some $C$ which does not depend on $\cparam,\cdist$
    $$|\bar{b}_t| \leq C\log(\tfrac{d}{\delta})\sqrt{\frac{d}{m\bar{\mu}\|\hvf\|}} +  \frac{C\|\Delta\|^2}{\bar{\mu}\|\hat{\vf}\|^2}$$
\end{enumerate}

\end{lemma}

Under the assumptions of Lemma~\ref{lem:P_contraction} and Lemma~\ref{lem:Q_contraction}, we apply the triangle inequality to Equation~\eqref{eq:proj_evol_1} and use the bounds in Lemma~\ref{lem:P_contraction} via the union bound. Given any $c_{ab} > 0$, we can take the constants $\cparam,\cdist$ relating $\Delta,\Gamma_t,m$ and $\bar{\mu}$ in Lemmata~\ref{lem:P_contraction} and~\ref{lem:Q_contraction} to be large enough, such that with probability $1-
\delta$:

$$|a_{t+1}| \leq |a_t|\sqrt{\left(1-\tfrac{\ccontr\alpha_0}{\|\hat{\vf}\|^2}\right)} + \frac{\alpha_0}{\|\hat{\vf}\|^2} c_{ab}|b_t| + C\alpha_0\left[\frac{\log(\tfrac{1}{\delta})}{\|\hat{\vf}\|\sqrt{m}} + \frac{\|\Delta\|^2}{\bar{\mu}\|\hat{\vf}\|^2}\right] $$

Similarly, we have with probability $1-\delta$:
$$|b_{t+1}| \leq |b_t|\sqrt{\left(1-\tfrac{\ccontr\alpha_1}{\bar{\mu}\|\hat{\vf}\|}\right)} +  \alpha_1 \frac{c_{ab}}{\|\hat{\vf}\|^2} |a_t| + C \alpha_1 \left[\log(\tfrac{d}{\delta})\sqrt{\frac{d}{m\bar{\mu}\|f\|}} +  \frac{\|\Delta\|^2}{\bar{\mu}\|\hat{\vf}\|^2}\right]$$

Now define:

$\kappa_t := 
\begin{bmatrix}
|a_t| \\ |b_t|
\end{bmatrix} \in \mathbb{R}^2$, $\beta := 
\begin{bmatrix}
C\alpha_0\left[\frac{\log(\tfrac{1}{\delta})}{\|\hat{\vf}\|\sqrt{m}} + \frac{\|\Delta\|^2}{\bar{\mu}\|\hat{\vf}\|^2}\right]
\\ C\alpha_1\left[\log(\tfrac{d}{\delta})\sqrt{\frac{d}{m\bar{\mu}\|\hat{\vf}\|}} +  \frac{\|\Delta\|^2}{\bar{\mu}\|\hat{\vf}\|^2}\right]
\end{bmatrix} \in \mathbb{R}^2$, $\Theta := 
\begin{bmatrix}
\sqrt{1-\tfrac{\alpha_0\ccontr}{\|\hat{\vf}\|^2}} & \frac{\alpha_0c_{ab}}{\|\hat{\vf}\|^2}\\
 \alpha_1 \frac{c_{ab}}{\|\hat{\vf}\|^2} &  \sqrt{1-\tfrac{\ccontr\alpha_1}{\|\hat{\vf}\|\bar{\mu}}}
\end{bmatrix}
 \in \mathbb{R}^{2\times 2}$

From the recursion for $|b_t|$ and $|a_t|$ above, we conclude that with probability $1-\delta$, we have the following evolution equation for $\kappa$ when the assumptions above are satisfied. 
\begin{equation}\label{eq:state_evolution}
\kappa_{t+1} \leq \Theta\kappa_t + \beta
\end{equation}
Where the vector inequalities are interpreted to be coordinate-wise. Notice that $\Gamma_t \leq  \|\vf_t -\fstar\| + \|\fstar -\hat{\vf}\|$ and since the initial condition $\vf_0 = \hat{\vf}$, we must have $\Gamma_t \leq \|\vf_t - \fstar\| + \Gamma_0$ and $\|\vf_t - \fstar\| = \|\kappa_t\|$. Therefore, we conclude $\Gamma_t \leq \|\kappa_t\| + \Gamma_0$.

Now define the event $\mathcal{C}_{t_0}$ to be the event that the Equation~\eqref{eq:state_evolution} is satisfied for $ 0\leq t\leq t_0-1$. From the discussions above, we must have: 
\begin{equation}\label{eq:prob_evol}
\bP\left(\mathcal{C}_{t+1}\biggr|\mathcal{C}_t, \cdist\Gamma_{t+1}\log(\tfrac{m}{\delta}) \leq \bar{\mu}\right) \geq 1-\delta
\end{equation}

Below, we will show that $\mathcal{C}_{t+1}$ has a high probability conditioned on $\mathcal{C}_t$ by showing that under the event $\mathcal{C}_t$, $\Gamma_{t+1}$ is small. 

Conditioned on $\mathcal{C}_t$, we unfurl the recursion in Equation~\eqref{eq:state_evolution} the following holds almost surely:
\begin{equation}\label{eq:unfurl}
\kappa_{t+1} \leq \Theta^{t+1}\kappa_0 + \sum_{s=0}^{t}\Theta^{t-s}\beta
\end{equation}

Take $\alpha_0 = c_0\|\hat{\vf}\|^2$ and $\alpha_1 = c_0 \bar{\mu}\|\hat{\vf}\|$ for some small enough constant $c_0$ as in the statement of the theorem (this can be done independently of $\cparam$ and $\cdist$ as per Lemmas~\ref{lem:P_contraction}
and~\ref{lem:Q_contraction}). We then pick $\cparam$ and $\cdist$ large enough to make $c_{ab}$ small enough to ensure $ 0 \preceq \Theta \preceq (1-\gamma)\vI$ for some $\gamma \in (0,1)$ for some constant $\gamma$ which does not depend on $\cparam,\cdist$ (by using diagonal dominance for instance).  Thus, the following follows from Equation~\eqref{eq:unfurl}:

\begin{equation}
\label{eq:state_conv}
\|\kappa_{t+1}\| \leq \|\Theta\|_{\mathsf{op}}^{t+1} \|\kappa_0\| + \frac{\|\beta\|}{1-\|\Theta\|_{\mathsf{op}}}
= \|\Theta\|_{\mathsf{op}}^{t+1} \Gamma_0 + \frac{\|\beta\|}{1-\|\Theta\|_{\mathsf{op}}}
\end{equation}

Using the fact that $\Gamma_t \leq \|\kappa_t\|+\Gamma_0$, we conclude that conditioned on $\mathcal{C}_t$, the following holds almost surely:

$$\Gamma_{t+1} \leq \left(1+\|\Theta\|_{\mathsf{op}}^{t+1} \right)\Gamma_0 + \frac{\|\beta\|}{1-\|\Theta\|_{\mathsf{op}}} $$

Therefore, conditioned on $\mathcal{C}_{t}$, we must have almost surely:

\begin{align}\Gamma_{t+1} &\leq 2\Gamma_0 + \frac{\|\beta\|}{\gamma} \leq 2\Gamma_0 + C\left[\frac{\|\hat{\vf}\|\log(\tfrac{1}{\delta})}{\sqrt{m}} + \frac{\|\Delta\|^2}{\bar{\mu}}\right]+ C\left[\log(\tfrac{d}{\delta})\sqrt{\frac{d\bar{\mu}\|\hat{\vf}\|}{m}} +  \frac{\|\Delta\|^2}{\|\hat{\vf}\|}\right]\nonumber \\
&\leq 2\Gamma_0 + C\left[\frac{\|\hat{\vf}\|\log(\tfrac{1}{\delta})}{\sqrt{m}} + \frac{\|\Delta\|^2}{\bar{\mu}} + \log(\tfrac{d}{\delta})\sqrt{\frac{d\bar{\mu}\|\hat{\vf}\|}{m}} \right] \nonumber \\
&\leq 2\Gamma_0 + \frac{C\bar{\mu}}{\sqrt{\cparam} \log(\tfrac{m}{\delta})} \leq \frac{C\bar{\mu}}{\sqrt{\cparam} \log(\tfrac{m}{\delta})}
\end{align}
In the second line we have used the fact that $\bar{\mu} <\|\hat{\vf}\|$. Note that $\Gamma_0 = \|\hvf -\fstar\|$ since $\vf_0 = \hvf$. 

Taking $\cdist = c \sqrt{\cparam}$ for some universal constant $c$, we check that we can make $\cdist$ as large as we wish by making $\cparam$ large enough. From the above discussion, we conclude that conditioned on $\mathcal{C}_t$, we must have almost surely:
$$\cdist\Gamma_{t+1} \log(\tfrac{m}{\delta}) \leq \bar{\mu}$$

Combining this with Equation~\eqref{eq:prob_evol} and unrolling the recursion, we conclude:
$\bP(\mathcal{C}_{K}) \geq 1-\delta K$. Combining this with Equation~\eqref{eq:state_conv}, with probability at-least $1-K\delta$, after time $K$, we must have:

\begin{equation}
\|\vf_{K}-\fstar\| \leq \exp(-\gamma K)\|\hat{\vf}-\fstar\| + C\left[\frac{\|\hat{\vf}\|\log(\tfrac{1}{\delta})}{\sqrt{m}} + \frac{\|\Delta\|^2}{\bar{\mu}} + \log(\tfrac{d}{\delta})\sqrt{\frac{d\bar{\mu}\|\hat{\vf}\|}{m}} \right] 
\end{equation}

Here, we have used the fact that $\|\kappa_t\| = \|\vf_t-\fstar\|$. 
\section{Proof of Theorem \ref{thm:yinyang-full-proof}}
\label{proof:yinyang-full-proof}
The proof structure is similar to that of Theorem \ref{thm:yinyin-mult-regression-bound}. For ease of exposition, assume $n = mK$ such that $K = \lceil \log_2(n) \rceil \Theta(\log(n))$ and $m \geq d\polylog(d/\delta)$. In the entire proof, we will use these facts freely.  Furthermore, denote $e_{\hvw_k} = \norm{\hvw_k - \wstar}^2$ and $e_{\hvf_k} = \norm{\hvf_k - \fstar}^2$. Since $\hvw_0$ is the OLS estimate computed on $m$ data points and $\hvf_0$ is the spectral method estimate computed on $m$ data points using $\hvw_0$, it follows that, with probability at least $1 - \nicefrac{2\delta}{2(K+1)}$, $e_{\hvw_0} =  \Otilde(\nicefrac{\norm{\fstar}^2 d}{m})$ and $e_{\hvf_0} = \Otilde(\nicefrac{\norm{\fstar}^2 d}{m})$.  Let $\lambda_1 = \Thetatilde(\nicefrac{\norm{\hvf_0}^2 d}{m})$. Then, from Theorem \ref{thm:wls-error-bound}, it follows that, with probability at least $1 - \nicefrac{3 \delta}{2(K+1)}$
\begin{align*}
    e_{\hvw_1} &\leq \Otilde\left(\norm{\hvf_0}^2\frac{1}{m} + \norm{\hvf_0} \frac{d \sqrt{\lambda_1}}{m} \right) \\
    &\leq \Otilde\left(\norm{\fstar}^2 \left(1 + \nicefrac{d}{m}\right)\left(\nicefrac{1}{m} + \left(\nicefrac{d}{m}\right)^{1.5}\right)\right)
\end{align*}
Now, running noisy phase retrieval on $m$ data points with $K_p = \Theta(\log(m))$ steps, using $\hvf_0, \hvw_1$ as estimates and $\mubar_1 = \Thetatilde(\norm{\hvf_0}\sqrt{\nicefrac{d}{m}})$, we conclude from Theorem \ref{thm:phase_r_mult} that the following holds with probability at least $1 - \nicefrac{2 \delta}{(K+1)}$
\begin{align*}
    e_{\hvf_1} &\leq \Otilde\left(\frac{e_{\hvf_0}}{m^2} + \frac{\norm{\hvf_0}^2}{m} + \frac{e^{2}_{\hvw_1}}{\mubar_1^2} + \frac{d \norm{\hvf_0} \mubar_1}{m} \right) \\
    &\leq \Otilde\left(\norm{\fstar}^2 \frac{d}{m^3} + \norm{\hvf_0}^2 \left(\nicefrac{1}{m} + (\nicefrac{d}{m})^{1.5}\right) + \left(\norm{\fstar}^2 \left(\nicefrac{1}{m} + (\nicefrac{d}{m})^{1.5}\right)\right)^2 \left(\norm{\hvf_0}^2 \frac{d}{m}\right)^{-1} \right) \\
    &\leq \Otilde\left(\frac{\norm{\fstar}^2}{m} + \norm{\fstar}^2 \left(2 + \nicefrac{d}{m}\right)\left(\nicefrac{1}{m} + \left(\nicefrac{d}{m}\right)^{1.5}\right)\right) \leq \Otilde\left(\norm{\fstar}^2 \left(\nicefrac{1}{m} + \left(\nicefrac{d}{m}\right)^{1.5}\right)\right)
\end{align*}
where we use the fact that $m \geq d \polylog(d)$. We shall now complete the proof via an inductive argument. To this end, define $S_k = \sum_{j=0}^{k} \nicefrac{1}{2^j}$ Clearly, $1 \leq S_k \leq 2$ and $S_{k+1} = 1 + \nicefrac{S_k}{2}$.  Define the event $E_k(L)$ for some $L > 1$ which does not depend on $k$ and every $1\leq l\leq k$

\begin{enumerate}
    \item $ \frac{1}{2}\|\fstar\|^2 \leq \|\hvf_l\|^2 \leq 2\|\fstar\|^2$
    \item $e_{\hvf_l} \leq L\|\fstar\|^2 \left( \max(\tfrac{l}{m},\tfrac{ld^2}{m^2}) + (\tfrac{d}{m})^{S_l}\right)$
\end{enumerate}

Similarly, for the some $L_1 > 1$ independent of $k$, we define $F_k(L)$ as the event such that for every $1\leq l \leq k$
\begin{enumerate}
    \item $e_{\hvw_l} \leq L_1\|\fstar\|^2 \left( \max(\tfrac{l}{m},\tfrac{ld^2}{m^2}) + (\tfrac{d}{m})^{S_l}\right)$
\end{enumerate}

For some $\polylog()$ independent of $k$, let
$$\mubar_{k} = \norm{\hvf_k}\sqrt{\max(L_1,L)\max(\tfrac{k}{m},\tfrac{(k)d^2}{m^2}) + \left(\nicefrac{d}{m}\right)^{S_k}}\polylog(\tfrac{nd}{\delta})$$ 

$$\lambda_{k} =  \norm{\hvf_k}^2 \left(\max(\tfrac{k}{m},\tfrac{(k)d^2}{m^2}) + \left(\nicefrac{d}{m}\right)^{S_k}\right)\polylog(\tfrac{nd}{\delta})L$$


Since $\hvw_{k+1} = \hvw_{\hvf_k, \lambda_{k}}$, it follows from Theorem \ref{thm:wls-error-bound} that when conditioned on $E_k$ the following holds with probability at least $1 - \frac{\delta}{2(K+1)}$, for $\polylog()$ which does not depend on $k$:
\begin{align}
    e_{\hvw_{k+1}} &\leq \left(\frac{\norm{\hvf_k}^2}{m} + \frac{d \norm{\hvf_k} \sqrt{\lambda_k}}{m}\right)\polylog(\tfrac{nd}{\delta}) \nonumber \\
    &\leq \left(\frac{\norm{\hvf_k}^2}{m} +  \norm{\hvf_k}^2\max\left(\frac{\sqrt{k}d}{m^{3/2}},\frac{\sqrt{k}d^2}{m^2}\right) +\norm{\hvf_k}^2\left(\frac{d  }{m}\right)^{1+\tfrac{S_k}{2}}\right)\polylog(\tfrac{nd}{\delta}) \sqrt{L}\nonumber \\
    &\leq \left(\norm{\fstar}^2\max\left(\frac{(k+1)}{m},\frac{(k+1)d^2}{m^2}\right) +\norm{\fstar}^2\left(\frac{d  }{m}\right)^{1+\tfrac{S_k}{2}}\right)\sqrt{L}\polylog(\tfrac{nd}{\delta}) \label{eq:w_star_recursion}
\end{align}

In the second step, we have used the definition of $\lambda_{k+1}$ and the fact that $\sqrt{x+y} \leq \sqrt{x} + \sqrt{y}$. In the third step, we use the fact that $\frac{d}{m^{3/2}} \leq \max(\frac{1}{m},\frac{d^2}{m^2})$ since it is their geometric mean. We also use the fact under the event $E_k(L)$, we have $\|\hvf_k\|\leq 2 \|\fstar\|$. Therefore for some fixed $\polylog$ independent of $k$:

 \begin{equation}\label{eq:f_recursion}
 L_1 \geq \sqrt{L}\polylog(nd/\delta) \implies \mathbb{P}(F_{k+1}(L_1)|E_k(L)) \geq 1-\frac{\delta}{2(K+1)}\,.\end{equation}

 Conditioned on $E_k\cap F_{k+1}$, running noisy phase retrieval on $m$ data points with $K_p = \Theta(\log(m))$ steps, using $\hvf_k, \hvw_{k+1}$ inputs, we conclude from Theorem \ref{thm:phase_r_mult} that the following holds with probability at least $1 - \nicefrac{ \delta}{2(K+1)}$
\begin{align*}
    &e_{\hvf_{k+1}} \leq \left(\frac{e_{\hvf_k}}{m^2} + \frac{\norm{\hvf_k}^2}{m} + \frac{e^{2}_{\hvw_{k+1}}}{\mubar_k^2} + \frac{d \norm{\hvf_k} \mubar_k}{m} \right)\polylog(\tfrac{nd}{\delta}) 
    \nonumber \\
&\leq \left(\frac{\|\fstar\|^2}{m} + e_{\hvw_{k+1}} + \frac{d\|\fstar\|\bar{\mu}_k}{m}\right) \polylog(\tfrac{nd}{\delta})\nonumber \\
&\leq \left(\norm{\fstar}^2\max\left(\frac{(k+1)}{m},\frac{(k+1)d^2}{m^2}\right) +\norm{\fstar}^2\left(\frac{d  }{m}\right)^{1+\tfrac{S_k}{2}}\right)(L_1+\sqrt{\max(L,L_1)})\polylog(\tfrac{nd}{\delta})) 
\end{align*}

In the second step, we have used the fact that $\|\hvf_k\| \leq 2 \|\fstar\|$ and the fact that $\bar{\mu}_k \geq e_{\hvw_k}$ as per the requirement of Theorem~\ref{thm:phase_r_mult} (this can be verified by the choice of parameters and the events $E_k,F_{k+1}$).  The last step follows from a similar calculation as in Equation~\eqref{eq:w_star_recursion} and substituting the bound on $e_{\hvw_{k+1}}$ implied by the event $F_{k+1}(L_1)$. This allows us to conclude:

\begin{equation} \label{eq:E_recursion}
L \geq (L_1+\sqrt{\max(L,L_1)})\polylog(\tfrac{nd}{\delta})) \implies
\mathbb{P}(E_{k+1}(L)|E_k(L)\cap F_{k+1}(L_1)) \geq 1-\frac{\delta}{2(K+1)}
\end{equation}

We note that whenever $L$ is a large enough, fixed poly-log factor $\polylog(\frac{nd}{\delta})$, both the conditions in Equations~\eqref{eq:f_recursion} and~\eqref{eq:E_recursion} are satisfied. Then, from Equations~\eqref{eq:f_recursion} and~\eqref{eq:E_recursion}, along with the union bound,

$$\mathbb{P}(E_{K+1}(L)\cap F_{K+1}(L_1)) \geq 1-\delta$$

Whenever $K = \lceil\log_2 n\rceil$, we conclude that $2\geq S_l \geq 2 - 1/n$. Thus, from the definition of $E_{K+1}$ and $F_{K+1}$ we conclude that with probability at-least $1-\delta$, both these inequalities hold:

$$\|\hvw_{K+1}-\wstar\|^2 \leq \|\fstar\|^2 \left( \tfrac{1}{m} + (\tfrac{d}{m})^{2}\right) \polylog(nd/\delta)$$

$$\|\hvf_{K+1}-\fstar\|^2 \leq \|\fstar\|^2 \left( \tfrac{1}{m} + (\tfrac{d}{m})^{2}\right) \polylog(nd/\delta)$$

\section{Proofs of Technical Lemmas}
\label{sec:tech_lem_proof}

The following lemmas will allow us to bound obtain a high-probability bound on certain random variables which we will encounter later. We give their proofs in Section~\ref{sec:concentration_lemmas}. 

\begin{lemma}\label{lem:partition_sets}
Consider i.i.d random vectors $\xi_1,\dots,\xi_n \in \mathbb{R}^k$ a sequence of sets $A_1,\dots,A_r \subseteq \mathbb{R}^k$ such that $A_i\cap A_j = \emptyset$ whenever $i\neq j$. Define $H_i = \sum_{j=1}^{n} \mathbbm{1}(\xi_j \in A_i)$. Let $p_i = \mathbb{P}(\xi_1 \in A_i)$ and $\mathcal{H}_i = \{j: \xi_j \in A_i\}$. Then, for some positive constants $c_0,c_1$:
$$\mathbb{P}(H_i > 2np_i) \leq \exp(-c_0 np_i)$$

$$\mathbb{P}(H_i < \tfrac{np_i}{2}) \leq \exp(-c_1 np_i)$$
\end{lemma}

The proof of Lemma~\ref{lem:partition_sets} follows from an application of Bernstein's inequality for binomial random variables (see \cite{boucheron2013concentration}).

\begin{lemma}\label{lem:chi_squared_conc}
Suppose $ z_i = 1-\epsilon_i^2$, where $\epsilon_i \sim \cN(0,1)$. Let $v_1,\dots,v_m \in \bR$ be fixed reals. Let $\|v\| := \sqrt{\sum_{i=1}^m v_i^2}$. Then, for some positive constants $C_0,C_1$: $$\bP(|\sum_{i=1}^m z_i v_i| > t\|v\|) \leq C_0\exp(-C_1t) \,.$$ 
\end{lemma}

\begin{lemma}\label{lem:conc_trunc_inv}
Assume that $\|\hat{\vf}\| > \bar{\mu}$ and define $ r = \lceil \log_2(\tfrac{\|\hat{\vf}\|}{\bar{\mu}})\rceil$. Assume that $m \geq C\frac{\|\hat{\vf}\|}{\bar{\mu}}\log(\tfrac{r}{\delta})$. for some large enough constant $C$.  Let $l \geq 1$ be a constant. Let $u \perp \hat{\vf}$ and $\|u\| = 1$. Consider the quantities:
$$L_1(X,u,l) := \frac{1}{m}\sum_{i=1}^{m} \mathbbm{1}(|\langle \hat{\vf},\vx_i\rangle|>\bar{\mu})\frac{ \langle \vx_i,u\rangle^2}{\langle \hat{\vf},\vx_i\rangle^{2l}}$$

$$L_2(X,l) := \frac{1}{m}\sum_{i=1}^{m} \mathbbm{1}(|\langle \hat{\vf},\vx_i\rangle|>\bar{\mu})\frac{1}{\langle \hat{\vf},\vx_i\rangle^{2l}}$$
Then we have the following concentration bounds (making the dependence on $l$ implicit):
\begin{enumerate}
    \item $\bP\left(L_1(X,u) > \frac{C}{\bar{\mu}^{2l-1}\|\hat{\vf}\|}\right)\leq \delta $
    \item $\bP\left(L_1(X,u) < \frac{C_1 2^{-2l}}{\bar{\mu}^{2l-1}\|\hat{\vf}\|}\right)\leq \delta$
    \item $\bP\left(L_2(X) > \frac{C}{\bar{\mu}^{2l-1}\|\hat{\vf}\|}\right)\leq \delta $
    \item $\bP\left(L_2(X) < \frac{C_1 2^{-2l}}{\bar{\mu}^{2l-1}\|\hat{\vf}\|}\right)\leq \delta$
\end{enumerate}
\end{lemma}

We consider symmetrization in the following sense. Draw $\zeta_1,\dots,\zeta_m$ i.i.d rademacher random variables (i.e, uniformly distributed over $\{-1,1\}$) and independent of $\vx_1,\dots,\vx_m$. Given any fixed projection matrix $R$, define $\vx_i^{\prime} = \zeta_i R \vx_i + (\vI-R)\vx_i$.
\begin{lemma}\label{lem:subspace_flip}
    $(X^{\prime}_1,\dots,X^{\prime}_m)$ are jointly distributed as i.i.d $\cN(0,\vI)$.
\end{lemma}

\subsection{Proof of Lemma~\ref{lem:N1_bound}}
\begin{proof}[Proof of Lemma~\ref{lem:N1_bound}]
\begin{enumerate}
    \item Consider $\langle\bar{N}_1,u\rangle = \frac{1}{m}\sum_{i=1}^{m}\gamma(\vx_i)\langle \fstar,\vx_i\rangle^2 \langle \hat{\vf},\vx_i\rangle \langle u,\vx_i\rangle (1-\epsilon_i^2) $. Invoking Lemma~\ref{lem:chi_squared_conc}, with $v_i = \gamma(\vx_i)\langle \fstar,\vx_i\rangle^2 \langle \hat{\vf},\vx_i\rangle \langle u,\vx_i\rangle (1-\epsilon_i^2) $ we  conclude the result. 

    \item Let $u \in \mathbb{R}^d$ be arbitrary.
    By Lemma~\ref{lem:close_conc}, we have with probability at-least $1-\delta$, for every $i \in [m]$, $|\langle \fstar,\vx_i\rangle| \leq \Gamma \sqrt{\log \tfrac{m}{\delta}} + |\langle\hat{\vf},\vx_i\rangle|$. Therefore, we conclude that for some constant $C > 0$, $$\mathbbm{1}(|\langle \hat{\vf},\vx_i\rangle|>\bar{\mu})\frac{\langle \fstar,\vx_i\rangle^4 \langle \vx_i,u\rangle^2}{\langle \hat{\vf},\vx_i\rangle^6} \leq C\mathbbm{1}(|\langle \hat{\vf},\vx_i\rangle|>\bar{\mu})\frac{\langle \vx_i,u\rangle^2}{\langle \hat{\vf},\vx_i\rangle^2}\left(1+\tfrac{\Gamma^4\log^2\tfrac{m}{\delta}}{\bar{\mu}^4}\right)$$
    Therefore, we conclude that with probability at-least $1-\delta$, we must have:
    \begin{equation}\label{eq:variance_proxy_1}
    h_1(X,u) \leq \frac{C}{m}\sum_{i=1}^{m} \mathbbm{1}(|\langle \hat{\vf},\vx_i\rangle|>\bar{\mu})\frac{ \langle \vx_i,u\rangle^2}{\langle \hat{\vf},\vx_i\rangle^2}\left(1+\tfrac{\Gamma^4\log^2\tfrac{m}{\delta}}{\bar{\mu}^4}\right)
    \end{equation}
    Taking $u = \hat{\vf}$, we conclude the result.
\item We begin with Equation~\eqref{eq:variance_proxy_1}, where we replace $\delta$ with $\tfrac{\delta}{2}$. Therefore, we conclude that with probability at-least $1-\tfrac{\delta}{2}$, we have:

    \begin{equation}\label{eq:variance_proxy_2}
    h_1(X,u) \leq \frac{C}{m}\sum_{i=1}^{m} \mathbbm{1}(|\langle \hat{\vf},\vx_i\rangle|>\bar{\mu})\frac{ \langle \vx_i,u\rangle^2}{\langle \hat{\vf},\vx_i\rangle^2}\left(1+\tfrac{\Gamma^4\log^2\tfrac{2m}{\delta}}{\bar{\mu}^4}\right)
    \end{equation}
    
Combining Lemma~\ref{lem:conc_trunc_inv} with Equation~\eqref{eq:variance_proxy_2} using the union bound,
  we conclude that with probability at-least $1-\delta$:

\begin{align}
    h_1(X,u)\leq \tfrac{C}{\bar{\mu}\|\hat{\vf}\|} \left(1+\tfrac{\Gamma^4\log^2\tfrac{2m}{\delta}}{\bar{\mu}^4}\right)
\end{align}

\end{enumerate}
\end{proof}

\subsection{Proof of Lemma~\ref{lem:N2_bound}}
\begin{proof}[Proof of Lemma~\ref{lem:N2_bound}]
\begin{enumerate}
    \item Consider $\langle\bar{N}_2,u\rangle = \frac{2}{m} \sum_{i=1}^{m}\gamma(\vx_i)\epsilon_i\langle \fstar,\vx_i\rangle \langle\Delta,\vx_i\rangle \langle \vx_i,u\rangle \langle \hat{\vf},\vx_i \rangle$. Conditioned on $\vx_1,\dots,\vx_m$, this is a one dimensional Gaussian random variable with mean $0$ and variance $\frac{4h_2(X,u)}{m}$. The concentration inequality follows from Gaussian concentration. 
    \item This proof is similar to the proof of Lemma~\ref{lem:N1_bound}. By similar considerations as in Equation~\eqref{eq:variance_proxy_1}, we conclude that with probability at-least $1-\delta/4$:
    \begin{equation}\label{eq:variance_proxy_2_1}
        h_2(X,u) \leq \frac{C}{m}\sum_{i=1}^{m}\mathbbm{1}(|\langle\hat{\vf},\vx_i\rangle|>\bar{\mu}) \frac{\langle \vx_i,u\rangle^2\langle \vx_i,\Delta\rangle^2}{\langle\hat{\vf},\vx_i\rangle^4}\left(1+\tfrac{\Gamma^2\log\tfrac{4m}{\delta}}{\bar{\mu}^2}\right)
    \end{equation}
    Now applying Lemma~\ref{lem:close_conc}, we have $\sup_{i\in [m]}|\langle\Delta,\vx_i\rangle|^2 \leq \|\Delta\|^2\log(\tfrac{4m}{\delta})$ with probability at-least $1-\tfrac{\delta}{4}$. Combining this with Equation~\eqref{eq:variance_proxy_2_1} via the union bound, we conclude that with probability at-least $1-\frac{\delta}{2}$, we have:
    \begin{equation}\label{eq:variance_proxy_2_2}
        h_2(X,u) \leq \frac{C}{m}\sum_{i=1}^{m}\mathbbm{1}(|\langle\hat{\vf},\vx_i\rangle|>\bar{\mu}) \frac{\langle \vx_i,u\rangle^2\|\Delta\|^2\log(\tfrac{4m}{\delta})}{\langle\hat{\vf},\vx_i\rangle^4}\left(1+\tfrac{\Gamma^2\log\tfrac{4m}{\delta}}{\bar{\mu}^2}\right)
    \end{equation}
    Setting $u = \hat{\vf}$, we have with probability at-least $1-\delta/2$:
    \begin{equation}\label{eq:variance_proxy_2_3_0}
        h_2(X,\hat{\vf}) \leq \|\Delta\|^2\log(\tfrac{4m}{\delta})\left(1+\tfrac{\Gamma^2\log\tfrac{4m}{\delta}}{\bar{\mu}^2}\right)\frac{C}{m}\sum_{i=1}^{m}\frac{\mathbbm{1}(|\langle\hat{\vf},\vx_i\rangle| >\bar{\mu})}{\langle\hat{\vf},\vx_i\rangle^2}
    \end{equation}
    Bounding the quantity $\frac{1}{m}\sum_{i=1}^{m}\frac{\mathbbm{1}(|\langle\hat{\vf},\vx_i\rangle| >\bar{\mu})}{\langle\hat{\vf},\vx_i\rangle^2}$ using Lemma~\ref{lem:conc_trunc_inv}.  Therefore, with probability at-least $1-\delta/2$:
    
    \begin{equation}
    \frac{1}{m}\sum_{i=1}^{m}\frac{\mathbbm{1}(|\langle\hat{\vf},\vx_i\rangle|>\bar{\mu})}{\langle\hat{\vf},\vx_i\rangle^2} \leq \frac{C}{\bar{\mu}\|\hat{\vf}\|}
    \end{equation}
    
    Combining this with Equation~\eqref{eq:variance_proxy_2_3_0} via the union bound, we conclude that with probability at-least $1-\delta$:
    \begin{equation}\label{eq:variance_proxy_2_3}
        h_2(X,\hat{\vf}) \leq C\frac{\|\Delta\|^2\log(\tfrac{4m}{\delta})}{\bar{\mu}\|\hat{\vf}\|}\left(1+\tfrac{\Gamma^2\log\tfrac{4m}{\delta}}{\bar{\mu}^2}\right)
    \end{equation}
    
    \item Now, assume that $u\perp \hat{\vf}$ and $\|u\|=1$. We begin with Equation~\eqref{eq:variance_proxy_2_2}. Using the union bound, we combine this with the high probability bound for the quantity $\sum_{i=1}^{m}\mathbbm{1}(|\langle\hat{\vf},\vx_i\rangle|>\bar{\mu}) \frac{\langle \vx_i,u\rangle^2}{\langle\hat{\vf},\vx_i\rangle^4}$ from Lemma~\ref{lem:conc_trunc_inv}. Thus, we conclude that with probability at-least $1-\delta$, we have:

    $$h_2(X,u) \leq C \frac{\|\Delta\|^2}{\bar{\mu}^3\|\hat{\vf}\|} \log(\tfrac{4m}{\delta})\left(1+\tfrac{\Gamma^2\log(\tfrac{4m}{\delta})}{\bar{\mu}^2}\right)$$
\end{enumerate}
\end{proof}

\subsection{Proof of Lemma~\ref{lem:N3_bound}}
\begin{proof}[Proof of Lemma~\ref{lem:N3_bound}]
\begin{enumerate}
    \item By definition, we note that: $$\langle\bar{N}_3,\hat{\vf}\rangle = \frac{1}{m} \sum_{i=1}^{m}\frac{\mathbbm{1}(|\langle\hat{\vf},\vx_i\rangle|>\bar{\mu})}{\langle \hat{\vf},\vx_i \rangle^2} \langle\Delta,\vx_i\rangle^2  $$
    Writing $\Delta = a \hat{\vf} + \Delta^{\perp}$ with $\Delta^{\perp}\perp f$, we note that:
\begin{align}
    \langle\bar{N}_3,\hat{\vf}\rangle &\leq \frac{1}{m} \sum_{i=1}^{m}\mathbbm{1}(|\langle\hat{\vf},\vx_i\rangle|>\bar{\mu})\left[a^2+\frac{\langle\Delta^{\perp},\vx_i\rangle^2}{\langle \hat{\vf},\vx_i \rangle^2}  \right]
\end{align}

Applying Lemma~\ref{lem:conc_trunc_inv}, using the fact that $\bar{\mu} \leq \|\hat{\vf}\|$ and $\max(a^2\|\hat{\vf}\|^2,\|\Delta^{\perp}\|^2) \leq \|\Delta\|^2$, we conclude the result. 
    \item 
    First, we will write $\Delta = a \hat{\vf} + \Delta^{\perp}$ where $\Delta^{\perp} \perp \hat{\vf}$. $\langle \Delta,\vx_i\rangle^2 = a^2 \langle \hat{\vf},\vx_i\rangle^2 + 2a \langle \hat{\vf},\vx_i \rangle \langle \Delta^{\perp},\vx_i\rangle + \langle \Delta^{\perp},\vx_i\rangle^2$.
    
Therefore, we can write:
\begin{align}
         &Q \bar{N}_3 = \nonumber \\
         &\frac{1}{m}\sum_{i=1}^{m}\mathbbm{1}(|\langle\hat{\vf},\vx_i\rangle|>\bar{\mu})\left[\frac{a^2 Q\vx_i}{\langle \hat{\vf},\vx_i\rangle} + \frac{2a Q\vx_i  \langle \Delta^{\perp},\vx_i\rangle}{\langle \hat{\vf},\vx_i\rangle^{2}} + \frac{ Q\vx_i \langle \Delta^{\perp},\vx_i\rangle^2}{\langle \hat{\vf},\vx_i\rangle^{3}}\right] \label{eq:ortho_decomp_N3}
\end{align}

Now consider the first term in Equation~\eqref{eq:ortho_decomp_N3}. When conditioned on $\langle \hat{\vf},\vx_i\rangle$ for $i\in [m]$, $Q\vx_i$ are distributed as i.i.d. standard Gaussian in the space orthogonal to $\hat{\vf}$. Therefore, by Gaussian concentration, we must have with probability at-least $1-\delta/6$:

$$ \biggr\|\frac{1}{m}\sum_{i=1}^{m}\mathbbm{1}(|\langle\hat{\vf},\vx_i\rangle|>\bar{\mu})\left[\frac{a^2 Q\vx_i}{\langle \hat{\vf},\vx_i\rangle}\right]\bigg\|^2 \leq  \frac{C}{m^2}\sum_{i=1}^{m}\mathbbm{1}(|\langle\hat{\vf},\vx_i\rangle|>\bar{\mu})\frac{a^4\left(d+\log(6/\delta)\right)}{\langle\hat{\vf},\vx_i\rangle^2}$$

Using Lemma~\ref{lem:conc_trunc_inv} with the equation above via the union bound, we conclude that with probability at-least $1-\delta/3$, we have:

\begin{equation}\label{eq:first_t_bound}
    \biggr\|\frac{1}{m}\sum_{i=1}^{m}\mathbbm{1}(|\langle\hat{\vf},\vx_i\rangle|>\bar{\mu})\left[\frac{a^2 Q\vx_i}{\langle \hat{\vf},\vx_i\rangle}\right]\bigg\| \leq \sqrt{\frac{C}{m}\frac{a^4\left(d+\log(6/\delta)\right)}{\bar{\mu}\|\hat{\vf}\|}}
\end{equation}
Only in this proof, we define $\bar{Q} = Q - \frac{\Delta^{\perp}(\Delta^{\perp})^{\intercal}}{\|\Delta^{\perp}\|^2}$. Now consider the second term in Equation~\eqref{eq:ortho_decomp_N3}:
\begin{align}
&\frac{1}{m}\sum_{i=1}^{m}\mathbbm{1}(|\langle\hat{\vf},\vx_i\rangle|>\bar{\mu})\left[ \frac{2a Q\vx_i  \langle \Delta^{\perp},\vx_i\rangle}{\langle \hat{\vf},\vx_i\rangle^{2}} \right] \nonumber \\&= \frac{1}{m}\sum_{i=1}^{m}\mathbbm{1}(|\langle\hat{\vf},\vx_i\rangle|>\bar{\mu})\left[ \frac{2a \bar{Q}\vx_i  \langle \Delta^{\perp},\vx_i\rangle}{\langle \hat{\vf},\vx_i\rangle^{2}} + \frac{2a \langle\Delta^{\perp},\vx_i\rangle^2 \Delta^{\perp}  }{\langle \hat{\vf},\vx_i\rangle^{2}\|\Delta^{\perp}\|^2} \right]
\label{eq:second_quantity}
\end{align}

The first term in Equation~\eqref{eq:second_quantity} can be bounded using Gaussian concentration and Lemma~\ref{lem:conc_trunc_inv} just like in Equation~\eqref{eq:first_t_bound}. With probability at-least $1-\delta/12$, we have:

\begin{align}
    \biggr\|\frac{1}{m}\sum_{i=1}^{m}\mathbbm{1}(|\langle\hat{\vf},\vx_i\rangle|>\bar{\mu})\left[ \frac{2a \bar{Q}\vx_i  \langle \Delta^{\perp},\vx_i\rangle}{\langle \hat{\vf},\vx_i\rangle^{2}}\right] \biggr\| \leq \sqrt{\frac{C}{m}\frac{a^2\|\Delta^{\perp}\|^2\left(d+\log(\tfrac{12}{\delta})\right)}{\bar{\mu}^3\|\hat{\vf}\|}} \label{lem:second_t_bound_prelim}
\end{align}

The second term in Equation~\eqref{eq:second_quantity} can be directly bounded using Lemma~\ref{lem:conc_trunc_inv}. Therefore, combining these results using the triangle inequality in Equation~\eqref{eq:second_quantity}, we conclude that with probability at-least $1-\delta/6$, we must have:

\begin{align}
    &\biggr\|\frac{1}{m}\sum_{i=1}^{m}\mathbbm{1}(|\langle\hat{\vf},\vx_i\rangle|>\bar{\mu})\left[ \frac{2a Q\vx_i  \langle \Delta^{\perp},\vx_i\rangle}{\langle \hat{\vf},\vx_i\rangle^{2}} \right]\biggr\| \nonumber \\
    &\leq \sqrt{\frac{C}{m}\frac{a^2\|\Delta^{\perp}\|^2\left(d+\log(\tfrac{12}{\delta})\right)}{\bar{\mu}^3\|\hat{\vf}\|}} + \frac{C|a|\|\Delta^{\perp}\|}{\bar{\mu}\|\hat{\vf}\|}
\label{eq:second_t_inequality}
\end{align}

Now consider the third term in Equation~\eqref{eq:ortho_decomp_N3}.

\begin{align}
      &\frac{1}{m}\sum_{i=1}^{m}\mathbbm{1}(|\langle\hat{\vf},\vx_i\rangle|>\bar{\mu})\left[\frac{ Q\vx_i \langle \Delta^{\perp},\vx_i\rangle^2}{\langle \hat{\vf},\vx_i\rangle^{3}}\right] \nonumber \\
      &=   \frac{1}{m}\sum_{i=1}^{m}\mathbbm{1}(|\langle\hat{\vf},\vx_i\rangle|>\bar{\mu})\left[\frac{ \bar{Q}\vx_i \langle \Delta^{\perp},\vx_i\rangle^2}{\langle \hat{\vf},\vx_i\rangle^{3}}+\frac{  \langle \Delta^{\perp},\vx_i\rangle^3\Delta^{\perp}}{\langle \hat{\vf},\vx_i\rangle^{3} \|\Delta^{\perp}\|^2}\right] \label{eq:third_quantity}
\end{align}

In Equation~\eqref{eq:third_quantity}, we can bound the first term using Gaussian concentration again and then conclude that with probability at-least $1-\delta/12$, we must have:

\begin{equation}
    \biggr\|\frac{1}{m}\sum_{i=1}^{m}\mathbbm{1}(|\langle\hat{\vf},\vx_i\rangle|>\bar{\mu})\frac{ \bar{Q}\vx_i \langle \Delta^{\perp},\vx_i\rangle^2}{\langle \hat{\vf},\vx_i\rangle^{3}}\biggr\| \leq \sqrt{\frac{C\|\Delta^{\perp}\|^4\log^3(\frac{12m}{\delta})d}{m\bar{\mu}^5\|\hat{\vf}\|}}
\end{equation}
Now consider the second term in Equation~\eqref{eq:third_quantity}. 
In Lemma~\ref{lem:subspace_flip}, we take the projector $R = Q$. Then, we conclude that for i.i.d. rademacher variables $\zeta_1,\dots,\zeta_m$ independent of everything else:
\begin{align}
    &\frac{1}{m}\sum_{i=1}^{m}\mathbbm{1}(|\langle\hat{\vf},\vx_i\rangle|>\bar{\mu})\left[\frac{  \langle \Delta^{\perp},\vx_i\rangle^3\Delta^{\perp}}{\langle \hat{\vf},\vx_i\rangle^{3}\|\Delta^{\perp}\|^2}\right] \nonumber \\
    &\stackrel{d}{=} \frac{1}{m}\sum_{i=1}^{m}\mathbbm{1}(|\langle\hat{\vf},\vx_i\rangle|>\bar{\mu})\zeta_i\left[\frac{  \langle \Delta^{\perp},\vx_i\rangle^3\Delta^{\perp}}{\langle \hat{\vf},\vx_i\rangle^{3}\|\Delta^{\perp}\|^2}\right]
\end{align}
Where $\stackrel{d}{=}$ denotes equality in distribution. By Azuma-Hoeffding inequality for rademacher random variables, we have with probability at-least $1-\delta/6$:

\begin{align}
    &\biggr\|\frac{1}{m}\sum_{i=1}^{m}\mathbbm{1}(|\langle\hat{\vf},\vx_i\rangle|>\bar{\mu})\zeta_i\left[\frac{  \langle \Delta^{\perp},\vx_i\rangle^3\Delta^{\perp}}{\langle \hat{\vf},\vx_i\rangle^{3}\|\Delta^{\perp}\|^2}\right]\biggr\| \nonumber \\&\leq \sqrt{\frac{C\log\tfrac{12}{\delta}}{m^2}\sum_{i=1}^{m} \mathbbm{1}(|\langle\hat{\vf},\vx_i\rangle|>\bar{\mu})\left[\frac{  \langle \Delta^{\perp},\vx_i\rangle^6}{\langle \hat{\vf},\vx_i\rangle^{6}\|\Delta^{\perp}\|^2}\right]} \leq \sqrt{\frac{C\|\Delta^{\perp}\|^4 \log^4(\tfrac{12m}{\delta})}{m\bar{\mu}^5\|\hat{\vf}\|}}
\end{align}
In the last step, we have used the high probability bound on $\sup_i |\langle  \vx_i, \Delta^{\perp}\rangle|$. Therefore, combining these inequalities, we conclude:

\begin{align}
      &\biggr\|\frac{1}{m}\sum_{i=1}^{m}\mathbbm{1}(|\langle\hat{\vf},\vx_i\rangle|>\bar{\mu})\left[\frac{ Q\vx_i \langle \Delta^{\perp},\vx_i\rangle^2}{\langle \hat{\vf},\vx_i\rangle^{3}}\right] \biggr\|\nonumber \\
      &= \sqrt{\frac{C\|\Delta^{\perp}\|^4 \log^3(\tfrac{12m}{\delta})(d+\log(\tfrac{12m}{\delta}))}{m\bar{\mu}^5\|\hat{\vf}\|}} \label{eq:third_t_inequality}
\end{align}

Combining Equations~\eqref{eq:first_t_bound},~\eqref{eq:second_t_inequality} and~\eqref{eq:third_t_inequality} with Equation~\eqref{eq:ortho_decomp_N3}, we conclude:
\begin{align}
    \|QN_3\| &\leq  \sqrt{\frac{C}{m}\frac{a^4\left(d+\log(6/\delta)\right)}{\bar{\mu}\|\hat{\vf}\|}} + \sqrt{\frac{C}{m}\frac{a^2\|\Delta^{\perp}\|^2\left(d+\log(\tfrac{12}{\delta})\right)}{\bar{\mu}^3\|\hat{\vf}\|}} \nonumber \\ &\quad + \frac{C|a|\|\Delta^{\perp}\|}{\bar{\mu}\|\hat{\vf}\|} + \sqrt{\frac{C\|\Delta^{\perp}\|^4 \log^3(\tfrac{12m}{\delta})(d+\log(\tfrac{12m}{\delta}))}{m\bar{\mu}^5\|\hat{\vf}\|}} \nonumber\\
    &\leq   \frac{C\|\Delta\|^2}{\bar{\mu}\|\hat{\vf}\|^2} + \sqrt{\frac{C\|\Delta\|^4 \log^3(\tfrac{12m}{\delta})(d+\log(\tfrac{12m}{\delta}))}{m\bar{\mu}^5\|\hat{\vf}\|}}
\end{align}

The last step follows by the fact that $|a| \leq \frac{\|\Delta\|}{\|\hat{\vf}\|}$ and $\|\Delta^{\perp}\|\leq \|\Delta\|$ and the fact that $\bar{\mu} \leq \|\hat{\vf}\|$ 
\end{enumerate}
\end{proof}

\subsection{Proof of Lemma~\ref{lem:irritating_bounds}}
\label{subsec:irritate_proof}
\begin{proof}[Proof of Lemma~\ref{lem:irritating_bounds}]
\begin{enumerate}
    \item From Lemma~\ref{lem:psd_bounds}, we conclude that with probability at-least $1-\tfrac{\delta}{2}$, we must have: 
    \begin{align} \frac{1}{m}\sum_{i=1}^{m} \mathbbm{1}(|\langle\hat{\vf},\vx_i\rangle| > \bar{\mu}) \leq \langle \hat{\vf},H(\vf) \hat{\vf}\rangle \leq \frac{4}{m}\sum_{i=1}^{m} \mathbbm{1}(|\langle\hat{\vf},\vx_i\rangle| > \bar{\mu})
    \end{align}
    
    Now, note that $\sum_{i=1}^{m} \mathbbm{1}(|\langle\hat{\vf},\vx_i\rangle| > \bar{\mu})$ is distributed as $\mathsf{Bin}(m,\mathbb{P}(|\langle\hat{\vf},\vx_i\rangle| > \bar{\mu}))$. Applying Lemma~\ref{lem:bin-rv-tail}, we conclude that with probability at-least $1-\exp(-cm\mathbb{P}(|\langle\hat{\vf},\vx_i\rangle| > \bar{\mu}))$, we must have:
    
    \begin{align} \tfrac{\mathbb{P}(|\langle\hat{\vf},\vx_i\rangle| > \bar{\mu})}{2}\leq \langle \hat{\vf},H(\vf) \hat{\vf}\rangle \leq 8\mathbb{P}(|\langle\hat{\vf},\vx_i\rangle| > \bar{\mu})
    \end{align}
    
    Now, using the fact that $ \tfrac{1}{4} \leq \mathbb{P}(|\langle\hat{\vf},\vx_i\rangle| > \bar{\mu}) \leq 1$ (since $\bar{\mu} < \|\hat{\vf}\|$), we conclude the result.

    \item Only in this proof, suppose that $\vf+\fstar - 2\hat{\vf} = a \frac{\hat{\vf}}{\|\hat{\vf}\|} + b \hat{\vf}^{\perp} $ such that $\|\hat{\vf}^{\perp}\| = 1$ and $\hat{\vf}^{\perp} \perp \hat{\vf}$.
    
    Note that \begin{align}
    &\langle u, H(\vf) \hat{\vf}\rangle =  \frac{1}{m}\sum_{i=1}^{m}\mathbbm{1}(|\langle\hat{\vf},\vx_i\rangle|\geq\bar{\mu}) \frac{\langle \vf+ \fstar,\vx_i\rangle \langle  \vx_i,u\rangle}{\langle \hat{\vf},\vx_i\rangle^2} \nonumber \\
    &= \frac{1}{m}\sum_{i=1}^{m}\mathbbm{1}(|\langle\hat{\vf},\vx_i\rangle|\geq\bar{\mu}) \left[\frac{\langle \vf+ \fstar-2\hat{\vf},\vx_i\rangle \langle  \vx_i,u\rangle}{\langle \hat{\vf},\vx_i\rangle^2} + 2\frac{ \langle  \vx_i,u\rangle}{\langle \hat{\vf},\vx_i\rangle}\right] \nonumber \\
    &= \frac{1}{m}\sum_{i=1}^{m}\mathbbm{1}(|\langle\hat{\vf},\vx_i\rangle|\geq\bar{\mu}) \left[b\frac{\langle \hat{\vf}^{\perp},\vx_i\rangle \langle  \vx_i,u\rangle}{\langle \hat{\vf},\vx_i\rangle^2} + \left(2+\tfrac{a}{\|\hat{\vf}\|}\right)\frac{ \langle  \vx_i,u\rangle}{\langle \hat{\vf},\vx_i\rangle}\right]
    \end{align}
    
     Now, conditioned on $\langle\hat{\vf},\vx_i\rangle$ for $i \in [m]$, we must have $\langle u,\vx_i\rangle$ for $i\in [m]$ to be i.i.d. standard Gaussians and the term $\frac{1}{m}\sum_{i=1}^{m}\mathbbm{1}(|\langle\hat{\vf},\vx_i\rangle|\geq\bar{\mu})\frac{ \langle  \vx_i,u\rangle}{\langle \hat{\vf},\vx_i\rangle}$ is a gaussian with variance $\frac{L_2(X)}{m}$ ( $L_2$ as defined in Lemma~\ref{lem:conc_trunc_inv} ). Therefore, applying Gaussian concentration and Lemma~\ref{lem:conc_trunc_inv}, we conclude
     that with probability at-least $1-\delta/2$, we must have:
     \begin{equation}
    \biggr|\frac{1}{m}\sum_{i=1}^{m}\mathbbm{1}(|\langle\hat{\vf},\vx_i\rangle|\geq\bar{\mu})\frac{ \langle  \vx_i,u\rangle}{\langle \hat{\vf},\vx_i\rangle}\biggr| \leq \sqrt{\frac{C\log(4/\delta)}{m\bar{\mu}\|\hat{\vf}\|}}
    \end{equation}
    
    Applying Cauchy Schwarz inequality and Lemma~\ref{lem:conc_trunc_inv}, we conclude that with probability $1-\frac{\delta}{2}$ (with $L_1$ as defined in Lemma~\ref{lem:conc_trunc_inv}): 
    \begin{align} 
        \biggr|\frac{1}{m}\sum_{i=1}^{m}\mathbbm{1}(|\langle\hat{\vf},\vx_i\rangle|\geq\bar{\mu}) \frac{\langle \hat{\vf}^{\perp},\vx_i\rangle \langle  \vx_i,u\rangle}{\langle \hat{\vf},\vx_i\rangle^2}\biggr| &\leq \sqrt{L_1(X,u)L_1(X,\hat{\vf}^{\perp})} \nonumber\\
        &\leq \frac{C}{\bar{\mu}\|\hat{\vf}\|} \label{eq:overlap_1}
    \end{align}
    The result follows by combining the above equations with the union bound along with the fact that $|a|, |b| \leq 2\Gamma $

    \item Now let $u \perp \hat{\vf}^{\perp}$ along with the condition $u\perp \hat{\vf}$. We will now find a finer bound than Equation~\eqref{eq:overlap_1}. Notice that conditioned on the random variables $\langle u,\vx_i\rangle,\langle \vx_i,\hat{\vf}\rangle$ for $i \in [m]$, the random variables $\langle \hat{\vf}^{\perp},\vx_i\rangle$ are i.i.d. standard Gaussians. Therefore, under this conditioning, we have $\frac{1}{m}\sum_{i=1}^{m}\mathbbm{1}(|\langle\hat{\vf},\vx_i\rangle|\geq\bar{\mu}) \frac{\langle \hat{\vf}^{\perp},\vx_i\rangle \langle  \vx_i,u\rangle}{\langle \hat{\vf},\vx_i\rangle^2}$ is a zero mean Gaussian with variance $\frac{L_1(X,u,2)}{m}$ (as defined in Lemma~\ref{lem:conc_trunc_inv})
    
    We conclude that with probability at-least $1-\delta$, we have:
$$\biggr|\frac{1}{m}\sum_{i=1}^{m}\mathbbm{1}(|\langle\hat{\vf},\vx_i\rangle|\geq\bar{\mu}) \frac{\langle \hat{\vf}^{\perp},\vx_i\rangle \langle  \vx_i,u\rangle}{\langle \hat{\vf},\vx_i\rangle^2}\biggr| \leq \sqrt{\frac{C\log(\tfrac{1}{\delta})}{m\bar{\mu}^3\|\hat{\vf}\|}}    $$

Rest of the proof follows similar to item 2.

    \item By similar observations as in item 1, we apply Lemma~\ref{lem:psd_bounds} to conclude that with probability at-least $1-\tfrac{\delta}{2}$:
    \begin{align} \frac{1}{m}\sum_{i=1}^{m} \mathbbm{1}(|\langle\hat{\vf},\vx_i\rangle| > \bar{\mu})\frac{\langle \vx_i,u\rangle^2}{\langle \vx_i,\hat{\vf}\rangle^2} \leq \langle u,H(\vf) u\rangle \leq \frac{4}{m}\sum_{i=1}^{m} \mathbbm{1}(|\langle\hat{\vf},\vx_i\rangle| > \bar{\mu})\frac{\langle \vx_i,u\rangle^2}{\langle \vx_i,\hat{\vf}\rangle^2}
    \end{align}
    The result follows by an application of Lemma~\ref{lem:conc_trunc_inv} along with the union bound.

    \item Only in this proof we consider an orthonormal basis $v_1,\dots,v_d$ for $\mathbb{R}^d$ such that $v_1 = \frac{\hat{\vf}}{\|\hvf\|}$ and $v_2 = \hat{\vf}^{\perp}$
    \begin{align}\label{eq:ell_2}
        \|QH(\vf)\hat{\vf}\|^2  = \sum_{i=2}^{d}\langle v_i,H(\vf)\hat{\vf}\rangle^2
    \end{align}
   
    From item 2, we conclude that with probability at-least $1-\delta/d$:
    $$\langle v_2,H(\vf)\hat{\vf}\rangle^2 \leq \frac{C\Gamma^2}{\bar{\mu}^2\|\hat{\vf}\|^2} + \left(1+\tfrac{\Gamma}{\|\hat{\vf}\|}\right)^2\frac{C\log(4d/\delta)}{m\bar{\mu}\|\hat{\vf}\|} $$
    
    From item 3, we conclude that for $i\geq 3$, with probability at-least $1-\delta/d$, we must have:
    $$\langle v_i,H(\vf)\hat{\vf}\rangle^2 \leq \frac{C\Gamma^2\log(\tfrac{4d}{\delta})}{m\bar{\mu}^3\|\hat{\vf}\|} + \left(1+\tfrac{\Gamma}{\|\hat{\vf}\|}\right)^2\frac{C\log(4d/\delta)}{m\bar{\mu}\|\hat{\vf}\|}$$
    Using these in Equation~\eqref{eq:ell_2} with the union bound, we conclude with probability at least $1-\delta$, we have:
    $$\langle \hat{\vf},H(\vf)^2 \hat{\vf}\rangle \leq \frac{C_1}{\|\hat{\vf}\|^2} + \left(1+\tfrac{\Gamma}{\|\hat{\vf}\|}\right)^2\frac{Cd\log(4d/\delta)}{m\bar{\mu}\|\hat{\vf}\|} + \frac{Cd\Gamma^2\log^2(\tfrac{4m}{\delta})}{m\bar{\mu}^3\|\hat{\vf}\|} + \frac{C\Gamma^2}{\bar{\mu}^2\|\hat{\vf}\|^2}  $$

    \item       Consider $QH_tu$. Similar to item 2, take $\vf+\fstar - 2\hat{\vf} = a \frac{\hat{\vf}}{\|\hat{\vf}\|} + b \hat{\vf}^{\perp} $:
    \begin{align}
        &QH(f)u = \frac{1}{m}\sum_{i=1}^{m}\frac{\mathbbm{1}(|\langle\hat{\vf},\vx_i\rangle|\geq \bar{\mu})}{\langle \vx_i,\hat{\vf}\rangle^3}\langle u,\vx_i\rangle Q\vx_i \langle \vf+ \fstar,\vx_i\rangle\nonumber \\
        &= \frac{2}{m}\sum_{i=1}^{m}\frac{\mathbbm{1}(|\langle\hat{\vf},\vx_i\rangle|\geq \bar{\mu})}{\langle \vx_i,\hat{\vf}\rangle^2}\langle u,\vx_i\rangle Q\vx_i \nonumber\\
        &\quad + \frac{1}{m}\sum_{i=1}^{m}\frac{\mathbbm{1}(|\langle\hat{\vf},\vx_i\rangle|\geq \bar{\mu})}{\langle \vx_i,\hat{\vf}\rangle^3}\langle u,\vx_i\rangle \langle f+ \fstar - 2\hat{\vf},\vx_i\rangle Q\vx_i \nonumber \\
        &= \frac{1}{m}\sum_{i=1}^{m}(2+\tfrac{a}{\|\hat{\vf}\|})\frac{\mathbbm{1}(|\langle\hat{\vf},\vx_i\rangle|\geq \bar{\mu})}{\langle \vx_i,\hat{\vf}\rangle^2}\langle u,\vx_i\rangle Q\vx_i \nonumber\\
        &\quad + \frac{b}{m}\sum_{i=1}^{m}\frac{\mathbbm{1}(|\langle\hat{\vf},\vx_i\rangle|\geq \bar{\mu})}{\langle \vx_i,\hat{\vf}\rangle^3}\langle u,\vx_i\rangle \langle \hat{\vf}^{\perp},\vx_i\rangle Q\vx_i  \label{eq:decomp_1}
    \end{align}
    
    In Lemma~\ref{lem:subspace_flip}, let $R$ be the projector onto the orthogonal space to $\hat{\vf}$. Then for $\zeta_1,\dots,\zeta_m$ i.i.d. rademacher variables independent of everything else, we have:
    
    \begin{align}
        &\frac{1}{m}\sum_{i=1}^{m}\frac{\mathbbm{1}(|\langle\hat{\vf},\vx_i\rangle|\geq \bar{\mu})}{\langle \vx_i,\hat{\vf}\rangle^3}\langle u,\vx_i\rangle \langle \hat{\vf}^{\perp},\vx_i\rangle Q\vx_i \nonumber \\ &\stackrel{d}{=} 
        \frac{1}{m}\sum_{i=1}^{m}\zeta_i\frac{\mathbbm{1}(|\langle\hat{\vf},\vx_i\rangle|\geq \bar{\mu})}{\langle \vx_i,\hat{\vf}\rangle^3}\langle u,\vx_i\rangle \langle \hat{\vf}^{\perp},\vx_i\rangle Q\vx_i \label{eq:dist_eq}
    \end{align}

Where $\stackrel{d}{=}$ denotes equality in distribution. By rademacher concentration, we have that with probability at-least $1-\delta/2$, we have:

\begin{align}
    &\biggr\|\frac{1}{m}\sum_{i=1}^{m}\zeta_i\frac{\mathbbm{1}(|\langle\hat{\vf},\vx_i\rangle|\geq \bar{\mu})}{\langle \vx_i,\hat{\vf}\rangle^3}\langle u,\vx_i\rangle \langle \hat{\vf}^{\perp},\vx_i\rangle Q\vx_i\biggr\|^2 \nonumber \\
    &\leq \frac{C\log(\tfrac{2d}{\delta})}{m^2}\sum_{i=1}^{m}\frac{\mathbbm{1}(|\langle\hat{\vf},\vx_i\rangle|\geq \bar{\mu})}{\langle \vx_i,\hat{\vf}\rangle^6}\langle u,\vx_i\rangle^2 \langle \hat{\vf}^{\perp},\vx_i\rangle^2 \|Q\vx_i\|^2
\end{align}

With probability at-least $1-\frac{3\delta}{8}$, we have:
$$\sup_{i \in [m]} \langle u,\vx_i\rangle^2 \langle \hat{\vf}^{\perp},\vx_i\rangle^2 \|Q\vx_i\|^2 \leq C\log(\tfrac{8m}{\delta})^3d$$ Using this, along with an application of Lemma~\ref{lem:conc_trunc_inv}, we conclude that with probability at-least $1-\delta/2$:

\begin{align}
    &\biggr\|\frac{1}{m}\sum_{i=1}^{m}\zeta_i\frac{\mathbbm{1}(|\langle\hat{\vf},\vx_i\rangle|\geq \bar{\mu})}{\langle \vx_i,\hat{\vf}\rangle^3}\langle u,\vx_i\rangle \langle \hat{\vf}^{\perp},\vx_i\rangle Q\vx_i\biggr\|^2 \nonumber \\
    &\leq \frac{Cd \log(\tfrac{4d}{\delta})\log^3(\tfrac{8m}{\delta})}{m\bar{\mu}^5\|\hat{\vf}\|}\label{eq:sym_conc_1}
\end{align}

Now consider $v_1,\dots,v_{d-1}$ to be any orthonormal basis for the span of $Q$ such that $v_1 = u$.
\begin{align}
    &\|\frac{1}{m}\sum_{i=1}^{m}\frac{\mathbbm{1}(|\langle\hat{\vf},\vx_i\rangle|\geq \bar{\mu})}{\langle \vx_i,\hat{\vf}\rangle^2}\langle u,\vx_i\rangle Q\vx_i\|^2 \nonumber \\
    &= \biggr\|\frac{1}{m}\sum_{i=1}^{m}\frac{\mathbbm{1}(|\langle\hat{\vf},\vx_i\rangle|\geq \bar{\mu})}{\langle \vx_i,\hat{\vf}\rangle^2}\langle u,\vx_i\rangle^2\biggr\|^2 
    \nonumber \\ &\quad + \sum_{j=2}^{d-1}\biggr\|\frac{1}{m}\sum_{i=1}^{m}\frac{\mathbbm{1}(|\langle\hat{\vf},\vx_i\rangle|\geq \bar{\mu})}{\langle \vx_i,\hat{\vf}\rangle^2}\langle u,\vx_i\rangle \langle v_j,\vx_i\rangle\biggr\|^2  \nonumber \\
    &\leq \frac{C}{\bar{\mu}^2\|\hat{\vf}\|^2} + \sum_{j=2}^{d-1}\biggr\|\frac{1}{m}\sum_{i=1}^{m}\frac{\mathbbm{1}(|\langle\hat{\vf},\vx_i\rangle|\geq \bar{\mu})}{\langle \vx_i,\hat{\vf}\rangle^2}\langle u,\vx_i\rangle \langle v_j,\vx_i\rangle\biggr\|^2
\end{align}

We have used Lemma~\ref{lem:conc_trunc_inv} in the second step. The second term is exactly same as that used in item 3 (by replacing $\hat{\vf}^{\perp}$ with $v_j$). Therefore, with probability at-least $1-\delta$, we must have:
\begin{align}
    \biggr\|\frac{1}{m}\sum_{i=1}^{m}\frac{\mathbbm{1}(|\langle\hat{\vf},\vx_i\rangle|\geq \bar{\mu})}{\langle \vx_i,\hat{\vf}\rangle^2}\langle u,\vx_i\rangle Q\vx_i\biggr\|^2 \leq \frac{C}{\bar{\mu}^2\|\hat{\vf}\|^2}+ \frac{Cd\log^2(\tfrac{4md}{\delta})}{m\bar{\mu}^3\|\hat{\vf}\|} \label{eq:sec_term}
\end{align}

Therefore, using Equations~\eqref{eq:dist_eq}~\eqref{eq:sym_conc_1}, we bound the first term in Equation~\eqref{eq:decomp_1}. We bound the second term with Equation~\eqref{eq:sec_term}. Combining these two using the union bound, we conclude:

$$\|QH(f)u\| \leq \left(1+\tfrac{\Gamma}{\|\hat{\vf}\|}\right)\frac{C}{\|\hat{\vf}\|\bar{\mu}} + \sqrt{\frac{Cd\log^2(\tfrac{4md}{\delta})}{m\bar{\mu}^3\|\hat{\vf}\|}} + \sqrt{\frac{Cd\Gamma^2 \log(\tfrac{4d}{\delta})\log^3(\tfrac{8m}{\delta})}{m\bar{\mu}^5\|\hat{\vf}\|}}$$
\end{enumerate}

\end{proof}

\subsection{Proof of Lemma~\ref{lem:P_contraction}}
\label{subsec:P_contr_proof}
\begin{proof}[Proof of Lemma~\ref{lem:P_contraction}]
\begin{enumerate}
    \item Consider the following sequence of inequalities:
\begin{align}
    \biggr\|P(I-\alpha_0 H_t(\vf_t))P\frac{\hat{\vf}}{\|\hat{\vf}\|}\biggr\|^2 &= \frac{1}{\|\hat{\vf}\|^2}\langle P(I-\alpha_0H_t(\vf_t))\hat{\vf},P(I-\alpha_0H_t(\vf_t))\hat{\vf} \rangle \nonumber \\
    &= 1 - 2\frac{\alpha_0}{\|\hat{\vf}\|^2}\langle\hat{\vf},H_t(\vf_t)\hat{\vf}\rangle + \frac{\alpha_0^2}{\|\hat{\vf}\|^2} \hat{\vf}^{\intercal}H_t P^2H_t(\vf_t) \hat{\vf} \nonumber \\
    &\leq 1 - 2\frac{\alpha_0}{\|\hat{\vf}\|^2}\langle\hat{\vf},H_t(\vf_t)\hat{\vf}\rangle + \frac{\alpha_0^2}{\|\hat{\vf}\|^4} \langle\hat{\vf}^{\intercal},H_t(\vf_t) \hat{\vf}\rangle^2 \label{eq:contract_1}
\end{align}
In the last step, we use the fact that $P = \frac{\hat{\vf}\hat{\vf}^{\intercal}}{\|\hat{\vf}\|^2}$.
By  Lemma~\ref{lem:irritating_bounds} items 1 and 5, we must have with probability at-least $1-\delta$, the following hold:
$C_0 \leq \hat{\vf}^{\intercal}H_t(\vf_t) \hat{\vf} \leq C_1$.
 By the assumption on the step size and Equation~\eqref{eq:contract_1}, we conclude:
\begin{equation}
    \biggr\|P(I-\alpha_0 H_t(\vf_t))P\frac{\hat{\vf}}{\|\hat{\vf}\|}\biggr\|^2 \leq \left(1- \tfrac{c\alpha_0 }{\|\hat{\vf}\|^2}\right)
\end{equation}

\item Note that $\|PH_t(\vf_t)Qu_t\| = \bigr|\langle\frac{\hat{\vf}}{\|\hat{\vf}\|},H_t(\vf_t)u_t\rangle\bigr|$. From item 2 of Lemma~\ref{lem:irritating_bounds}, we conclude that with probability at-least $1-\delta$:

\begin{align}
    |\langle\tfrac{\hat{\vf}}{\|\hat{\vf}\|},H_t(\vf_t)u_t\rangle| &\leq \frac{C\Gamma_t}{\bar{\mu}\|\hat{\vf}\|^2} + \left(1+\tfrac{\Gamma_t}{\|\hat{\vf}\|}\right)\sqrt{\frac{C\log(1/\delta)}{m\bar{\mu}\|\hat{\vf}\|^3}} \nonumber \\
    &\leq C\left[\frac{\Gamma_t}{\bar{\mu}\|\hat{\vf}\|^2} + \sqrt{\frac{\log(1/\delta)}{m\bar{\mu}\|\hat{\vf}\|^3}}\right]
\end{align}
In the last step, we have used the assumption to bound $\Gamma_t \leq \bar{\mu} < \|\hat{\vf}\|$. 
\item We decompose the noise $N_t$ into $N_{t,1},N_{t,2}$ and $N_{t,3}$ as in Equation~\eqref{eq:noise_decomp} to give us $N_t = N_{t,1}+N_{t,2}+N_{t,3}$. Then, by definition of $\bar{a}_t$, and the bounds in Lemmata~\ref{lem:N1_bound},~\ref{lem:N2_bound} and ~\ref{lem:N3_bound} we must have: 
\begin{align}
    |\bar{a}_t| &\leq \frac{1}{\|\hat{\vf}\|}|\langle N_{t,1} , \hat{\vf}\rangle| + \frac{1}{\|\hat{\vf}\|}|\langle N_{t,2} , \hat{\vf}\rangle|+\frac{1}{\|\hat{\vf}\|}|\langle N_{t,3} , \hat{\vf}\rangle| \nonumber \\
    &\leq \frac{C\log(\tfrac{1}{\delta})}{\|\hat{\vf}\|}\sqrt{\tfrac{1}{m}\left[1+\tfrac{\Gamma_t^4}{\bar{\mu}^4}\log^2 \tfrac{m}{\delta}\right]} + C\sqrt{\frac{\|\Delta\|^2\log^2(\tfrac{2m}{\delta})}{m\bar{\mu}^2\|\hat{\vf}\|^2}\left(1+\tfrac{\Gamma^2_t\log\tfrac{2m}{\delta}}{\bar{\mu}^2}\right)} \nonumber \\
    &\quad + C\frac{\|\Delta\|^2}{\bar{\mu}\|\hat{\vf}\|^2}  \nonumber \\
 &\leq C\left[\frac{\log(\tfrac{1}{\delta})}{\|\hat{\vf}\|\sqrt{m}} + \frac{\|\Delta\|^2}{\bar{\mu}\|\hat{\vf}\|^2}\right]
\end{align}
In the last step we have used the bound $C\Gamma_t \leq \bar{\mu}\log(\tfrac{m}{\delta})$ and $C\|\Delta\|\log(\tfrac{m}{\delta}) \leq \bar{\mu}$. 
\end{enumerate}

\end{proof}

\subsection{Proof of Lemma~\ref{lem:Q_contraction}}
\label{subsec:Q_contr_proof}

\begin{proof}[Proof of Lemma~\ref{lem:Q_contraction}]
\begin{enumerate}
    \item Notice that $\|Q(I-\alpha_1H_t(\vf_t))Qu_t\|^2 = 1 - 2\alpha_1 \langle u_t,H_t(\vf_t) u_t \rangle + \alpha_1^{2} \|QH_t(\vf_t)u_t\|^2$
    
    From item 2 of Lemma~\ref{lem:irritating_bounds}, we have with probability at-least $1-\tfrac{\delta}{2}$:
    $\frac{C_0}{\|\hat{\vf}\|\bar{\mu}} \leq \langle u_t, H_t(\vf_t)u_t\rangle \leq \frac{C_1}{\|\hat{\vf}\|\bar{\mu}}$. By item 6 of Lemma~\ref{lem:irritating_bounds}, we have with probability at-least $1-\tfrac{\delta}{2}$:
    
    \begin{align}
        \|QH_t(\vf_t)u_t\|^2 &\leq \left(1+\tfrac{\Gamma_t}{\|\hat{\vf}\|}\right)^2\frac{C}{\|\hat{\vf}\|^2\bar{\mu}^2} + \frac{Cd\log^2(\tfrac{4md}{\delta})}{m\bar{\mu}^3\|\hat{\vf}\|} + \frac{Cd\Gamma^2_t \log(\tfrac{4d}{\delta})\log^3(\tfrac{8m}{\delta})}{m\bar{\mu}^5\|\hat{\vf}\|} \nonumber \\
        &\leq \frac{C}{\|\hat{\vf}\|^2\bar{\mu}^2}
    \end{align}
    In the second step, we have used the bounds on $m$ and $\Gamma_t$. 
    Now, picking $\alpha_1 \leq c \bar{\mu}\|\hat{\vf}\|$ for small enough $c$, we conclude the result. 
    \item Consider $\frac{1}{\|\hat{\vf}\|^2}\|QH_t(\vf_t)\hat{\vf}\|^2$.
    By item 5 of Lemma~\ref{lem:irritating_bounds}, we conclude the result and using the bounds on the parameters in the assumptions.  
    \item $$ |\bar{b}_t| \leq \|QN_{t,1}\| + \|QN_{t,2}\| + \|QN_{t,3}\| $$
    
    Consulting Lemmata~\ref{lem:N1_bound},~\ref{lem:N2_bound} and~\ref{lem:N3_bound}, we conclude $\|QN_{t,1}\| \leq C\log(\tfrac{d}{\delta})\sqrt{\frac{d}{m\bar{\mu}\|f\|}}$. $\|QN_{t,2}\| \leq C\sqrt{\tfrac{d\log(\tfrac{1}{\delta})}{m\bar{\mu}\|\bar{f}\|}}$. $\|QN_{t,3}\| \leq \frac{C\|\Delta\|^2}{\bar{\mu}\|\hat{\vf}\|^2} + \sqrt{\frac{C d}{m\bar{\mu}\|\hat{\vf}\|}} $, from which we conclude the result. 
\end{enumerate}
\end{proof}

\section{Proof of Concentration Lemmas}
\label{sec:concentration_lemmas}
\subsection{Proof of Lemma~\ref{lem:chi_squared_conc}}
\begin{proof}[Proof of Lemma~\ref{lem:chi_squared_conc}]
From \cite[Section 5.2.4]{vershynin2010introduction}, we conclude that $\mathbb{E}\exp(\lambda z_i) \leq \exp(c_0\lambda^2 )$ for every  $\lambda \in [-c,c]$ for some $c>0$. Therefore, whenever $|\lambda| \leq \frac{c}{\max_i|v_i|}$ we have:
$$\mathbb{E}\exp(\lambda\sum_{i=1}^{n} v_i z_i) \leq \exp(c_0\lambda^2 \sum_{i=1}^n v_i^2)$$

The Chernoff bound with $\lambda =  \frac{c}{\|v\|}$ implies:

$$\bP\left(|\sum_{i=1}^m z_i v_i| > t\right) \leq C_0\exp(-C_1\frac{t}{\|v\|})$$
\end{proof}

\subsection{Proof of Lemma~\ref{lem:conc_trunc_inv}}
\label{sec:conc_trunc_inv}

\begin{proof}[Proof of Lemma~\ref{lem:conc_trunc_inv}]
\begin{enumerate}
    \item   Note that $\langle X_i,\hat{\vf}\rangle \sim \mathcal{N}(0,\|\hat{\vf}\|^2)$. For integers $r-1 \geq i\geq 0$, define sets $A_i = \{x\in \mathbb{R}: 2^{i}\bar{\mu}\leq |x| < 2^{i+1}\bar{\mu}\}$ and $A_r = \{x \in \mathbb{R}: |x|\geq 2^{r}\bar{\mu}\}$. In Lemma~\ref{lem:partition_sets}, we take $\xi_1,\dots,\xi_n$ to be $\langle X_1,\hat{\vf}\rangle, \dots, \langle X_m,\hat{\vf}\rangle$. Taking $H_i,\mathcal{H}_i$ to be the frequency as defined in Lemma~\ref{lem:partition_sets}, with $ r = \lceil \log_2(\tfrac{\|\hat{\vf}\|}{\bar{\mu}})\rceil$ we conclude that the following holds almost surely:
    \begin{align}
        &\frac{1}{m}\sum_{i=1}^{m} \mathbbm{1}(|\langle \hat{\vf},X_i\rangle|>\bar{\mu})\frac{ \langle X_i,u\rangle^2}{\langle \hat{\vf},X_i\rangle^{2l}}
        \leq \frac{1}{m}\sum_{i=0}^{r} \frac{ 2^{-2li}}{\bar{\mu}^{2l}}\sum_{j\in \mathcal{H}_i} \langle X_j,u\rangle^2  \label{eq:counting_bins}
    \end{align}

Note that since $u \perp \hat{\vf}$, we must have $\langle X_j,u\rangle$ to be independent of $\mathcal{H}_i$ and is distributed as $\cN(0,1)$.  By Lemma~\ref{lem:chi_squared_conc}, we conclude:
\begin{equation}
    \bP\left(\sum_{j\in \mathcal{H}_i}\langle X_j,u\rangle^2 > H_i + C_0 \sqrt{H_i}t\right) \leq \exp(-t)
\end{equation}
Therefore, with probability at-least $1-\tfrac{\delta}{2}$, we must have for every $i = 0,\dots,r$:
\begin{equation}\label{eq:sq_within_bins}
\sum_{j\in \mathcal{H}_i}\langle X_j,u\rangle^2 \leq  H_i + C_0 \sqrt{H_i}\log(\tfrac{2r}{\delta})\end{equation}

Therefore, combining Equations~\eqref{eq:counting_bins} and ~\eqref{eq:sq_within_bins}, we conclude that with probability at-least $1-\tfrac{\delta}{2}$, we have:

\begin{align}
    L_1(X)\leq \frac{1}{m}\sum_{i=0}^{r} \frac{ 2^{-2li}}{\bar{\mu}^{2l}}\left(H_i +  C_0\sqrt{H_i}\log(\tfrac{2r}{\delta})\right)\label{eq:variance_proxy_3}
\end{align}

Let $p_i$ be the probabilities as defined in Lemma~\ref{lem:partition_sets}. Then, we must have for $i=0,\dots,r$,  $ c 2^{i}\frac{\bar{\mu}}{\|\hat{\vf}\|} \leq p_i \leq C 2^{i}\frac{\bar{\mu}}{\|\hat{\vf}\|}$. Applying Lemma~\ref{lem:partition_sets}, we conclude that with probability at-least $1-\frac{\delta}{2}$, we must have $ i = 0,\dots,r$: 
\begin{equation}\label{eq:count_inequality}
H_i \leq C_1\max\left(\log(\tfrac{2r}{\delta}),\frac{2^{i}\bar{\mu}m}{\|\hat{\vf}\|}\right) = \frac{C_12^{i}\bar{\mu}m}{\|\hat{\vf}\|}
\end{equation}
Here, we have used the assumption that $m \geq 
C\frac{\|\hat{\vf}\|}{\bar{\mu}}\log^2(\tfrac{2r}{\delta})$. Under the same assumption and the same event in Equation~\eqref{eq:count_inequality}, we conclude that $\sqrt{H_i}\log(\tfrac{2r}{\delta}) + H_i \leq \frac{C2^{i}\bar{\mu}m}{\|\hat{\vf}\|}$ for every $i = 0,\dots,r$. Combining this with Equation~\eqref{eq:variance_proxy_3} using the union bound, we conclude that with probability at-least $1-\delta$, we must have: 

$$L_1(X) \leq \frac{C}{\bar{\mu}^{2l-1}\|\hat{\vf}\|}$$

\item The lower bounds proceed in a similar fashion as the upper bounds in item 1. In Equation~\eqref{eq:counting_bins}, we use $2^{-2l(i+1)}$ instead to get a lower bound. In Equation~\eqref{eq:sq_within_bins}, we use the high probability lower bound of $H_i - C_0 \sqrt{H}_i\log(\tfrac{2r}{\delta})$ instead of the upper bound. A similar argument as in the proof of item 1 using the high probability lower bounds for $H_i$ bounds from Lemma~\ref{lem:partition_sets}, instead of the upper bounds allows us to conclude the result.

\item This proof proceeds similar to the proof of item 1, but replacing $H_i + C_0\sqrt{H_i}\log(\tfrac{2r}{\delta})$ with just $H_i$. 

\item This proof proceeds similar to the proof of item 2, but replacing $H_i - C_0\sqrt{H_i}\log(\tfrac{2r}{\delta})$ with just $H_i$. 

\end{enumerate}

\end{proof}

\section{Proof of the Lower Bound}
\label{sec:lb_proof}
\subsection{Technical Lemmas}
For any $x \in \bR$, we write $x^{+} = \max\{x, 0\}$
\begin{lemma}[Lower Bounding the Expected Risk by a Coupling]
\label{lem:coupling-lecam}
Let $\theta_1, \theta_2 \in \Theta \subset \bR$ be two identically distributed continuous random variables. Let $y = g(\theta, \epsilon)$ represent a model parameterized by $\theta$, where $g$ is some measurable function and $\epsilon$ is some random variable drawn independently of $\theta$, such that $\Law(y | \theta)$ has a density with respect to the Lebesgue measure for any $\theta \in \Theta$. Furthermore, let $\htheta(y)$ be an estimator of $\theta$. Then, for any arbitrary coupling $\Pi(\theta_1, \theta_2)$ of $\theta_1$ and $\theta_2$
\begin{align*}
    \bE_{\theta, y}\left[(\htheta(y) - \theta)^2\right] \geq \tfrac{1}{4}\bE_{\theta_1, \theta_2 \sim \Pi(\theta_1, \theta_2)} \left[\left(\theta_1 - \theta_2\right)^2\left(1 - \sqrt{2 \KL{\Law(y \ | \ \theta_1)}{\Law(y \ | \ \theta_2)}}\right) \right]
\end{align*}
\end{lemma}
\begin{proof}
Since $\theta_1$ and $\theta_2$ are identically distributed,
\begin{align*}
    \bE_{\theta, y}\left[(\htheta(y) - \theta)^2\right] &= \frac{1}{2} \bE_{\theta_1, y}\left[\left(\htheta(y) - \theta_1\right)^2\right] + \frac{1}{2} \bE_{\theta_2, y}\left[\left(\htheta(y) - \theta_2\right)^2\right] \\
    &= \frac{1}{2} \bE_{\theta_1, \theta_2 \sim \Pi(\theta_1, \theta_2)} \left[ \bE_{y \ | \theta_1}\left[\left(\htheta(y) - \theta_1\right)^2\right] + \bE_{y \ | \theta_2}\left[\left(\htheta(y) - \theta_2\right)^2\right] \right]
\end{align*}
Let $p(y \ | \ \theta)$ denote the density of $\Law( y \ | \ \theta)$ for any $\theta \in \Theta$. It follows that,
\begin{align*}
    \bE_{\theta, y}\left[(\htheta(y) - \theta)^2\right] &= \frac{1}{2} \bE_{\theta_1, \theta_2 \sim \Pi(\theta_1, \theta_2)} \left[ \int \left(\htheta(y) - \theta_1\right)^2 p(y \ | \ \theta_1) \dd y + \int \left(\htheta(y) - \theta_2\right)^2 p(y \ | \ \theta_2) \dd y  \right]  \\
    &\geq \frac{1}{2} \bE_{\theta_1, \theta_2 \sim \Pi(\theta_1, \theta_2)} \left[ \int \left(\left(\htheta(y) - \theta_1\right)^2 + \left(\htheta(y) - \theta_2\right)^2\right) \min\{p(y \ | \ \theta_1), p(y \ | \ \theta_2) \} \dd y \right] \\
    &\geq \frac{1}{4}  \bE_{\theta_1, \theta_2 \sim \Pi(\theta_1, \theta_2)} \left[ (\theta_1 - \theta_2)^2 \int   \min\{p(y \ | \ \theta_1), p(y \ | \ \theta_2) \} \dd y \right] \\
    &\geq \frac{1}{4}  \bE_{\theta_1, \theta_2 \sim \Pi(\theta_1, \theta_2)} \left[ (\theta_1 - \theta_2)^2 \left( 1 - \TV\left( \Law(y \ | \ \theta_1), \Law(y \ | \ \theta_2)\right) \right)  \right] \\
    &\geq \frac{1}{4}  \bE_{\theta_1, \theta_2 \sim \Pi(\theta_1, \theta_2)} \left[(\theta_1 - \theta_2)^2 \left( 1 - \sqrt{2\KL {\Law(y \ | \ \theta_1)}{ \Law(y \ | \ \theta_2)}} \right)^{+} \right]
\end{align*}
where the second inequality follows from the identity $a^2 + b^2 \geq \nicefrac{(a-b)^2}{2}$, the third inequality follows from the fact that $\TV\left(P, Q\right) = 1 - \int \min \left\{ \frac{\dd P}{\dd \lambda}, \frac{\dd Q}{\dd \lambda} \right\} \dd \lambda$ for any two probability measures $P$ and $Q$ that are absolutely continuous with respect to some base measure $\lambda$, and the last is an application of Pinsker's inequality.
\end{proof}

\subsection{Proof of the Lower Bound}
The proof of our lower bound for the heteroscedastic regression problem involves fixing a noise model $\vf$ and analyzing the statistical indistinguishability between two instances of the heteroscedastic regression model parameterized by the regressors $\vw$ and $\vw + \vv$ respectively,  and sharing a common noise model $\vf$. To this end, we use $P_{\vw, \vf}$ to denote the heteroscedastic regression model parameterized by $\vw$ and  $\vf$, i.e., $P_{\vw, \vf}$ is a probability distribution supported on $\bR^d \times \bR$ such that $(\vx, y) \sim P_{\vw, \vf}$ implies $\vx \sim \cN(0, \vI)$ and $y \ | \  \vx \ \sim \cN(\dotp{\vw}{\vx}, {\dotp{\vf}{\vx}}^2)$. We first consider the simpler case when $\vv$ is parallel to $\vf$. This allows us to obtain a loose lower bound of $\Omega\left(\nicefrac{\norm{\vf}^2}{n}\right)$ via direct application of LeCam's two point method. To this end, 

\begin{lemma}[Lower Bound for $\vv$ parallel to $\vf$]
\label{lem:lower-bound-coarse}
Consider a fixed $\vf \in \bR^d$ and let $\vv$ be a vector parallel to $\vf$ such that $\norm{\vv} = \Delta$. Then, for any estimator $\hvw$ which estimates $\wstar$ with inputs $(\vx_i,y_i)_{i\in [n]}$: 
\begin{align*}
\inf_{\hvw} \sup_{\vw \in \bR^d} \bE_{(\vx_i, y_i)_{i \in [n]} \ \iidsim \ P_{\vw, \vf}}\left[\norm{\hvw - \vw}^2\right] \geq \Omega(\nicefrac{\norm{\vf}^2}{n})
\end{align*}
\end{lemma}
\begin{proof}
Consider any two instances of the heteroscedastic regression problem with a common noise model $\hvf$ and regressors $\vw$ and $\vw + \vv$ respectively. From a direct computation, it follows that,
\begin{align*}
    \KL{\Law\left(\vx, \vy | \vw, \vf \right)}{\Law\left(\vx, \vy | \vw + \vv, \vf \right)} &= \bE_{\vx \sim \cN(0, \vI)}\left[ \KL{\cN(\dotp{\vw}{\vx}, {\dotp{\vf}{\vx}}^2)}{\cN(\dotp{\vw + \vv}{\vx}, {\dotp{\vf}{\vx}}^2)} \right] \\
    &= \bE_{\vx \sim \cN(0, \vI)}\left[ \left(\frac{\dotp{\vv}{\vx}}{\dotp{\vf}{\vx}}\right)^2 \right] = \frac{\Delta^2}{\norm{\vf}^2}
\end{align*}
From LeCam's two point method, we obtain, the following for any $\Delta \geq 0$
\begin{align*}
\inf_{\hvw} \sup_{\vw \in \bR^d} \bE_{(\vx_i, y_i)_{i \in [n]} \ \iidsim \ P_{\vw, \vf}}\left[\norm{\hvw - \vw}^2\right] \geq \Delta^2 e^{-\frac{n \Delta^2}{\norm{\vf}^2}} 
\end{align*}
Setting $\Delta = \nicefrac{\norm{\hvf}}{\sqrt{n}}$, we obtain,
\begin{align*}
\inf_{\hvw} \sup_{\vw \in \bR^d} \bE_{(\vx_i, y_i)_{i \in [n]} \ \iidsim \ P_{\vw, \vf}}\left[\norm{\hvw - \vw}^2\right] \geq \Omega(\nicefrac{\norm{\vf}^2}{n})
\end{align*}
\end{proof}
We now present a finer lower bound of $\Omega\left(\nicefrac{d^2}{n^2}\right)$ by considering the case when $\vv$ is perpendicular to $\vf$. This case encapsulates the key technical challenge of our lower bound analysis. 
\begin{lemma}[Finer Lower Bound]
\label{lem:lower-bound-fine}
Assume $n\geq d$ and $d\geq 2$. For fixed $\vf \in \bR^d$ and any arbitrary estimator $\hvw$, 
\begin{align*}
\inf_{\hvw} \sup_{\vw \in \bR^d} \bE_{(\vx_i, y_i)_{i \in [n]} \ \iidsim \ P_{\vw, \vf}}\left[\norm{\hvw - \vw}^2\right] \geq \tilde{\Omega}(\nicefrac{\norm{\vf}^2 d^2}{n^2})
\end{align*} 

Here $\tilde{\Omega}$ hides a factor of $\left(\frac{1}{(\log \log n)^{3/2} + (\log \log n) \log d + \frac{(\log \log n)^2 (\log d)^2}{\sqrt{d}}}\right)^2$
\end{lemma}
\begin{proof}
Let $\ve_1, \dots, \ve_d$ denote an orthonormal basis of $\bR^d$ such that $\ve_d = \nicefrac{\vf}{\norm{\vf}}$ . Furthermore, let $S_{\alpha}(\ve_{1:d-1})$ denote uniform distribution over the sphere of radius $\alpha$ on the subspace spanned by $\ve_1, \dots, \ve_{d-1}$ centered around the origin. It follows that, for any estimator $\hvw$,
\begin{align*}
    \sup_{\vw \in \bR^d} \bE_{(\vx_i, y_i)_{i \in [n]} \ \iidsim \ P_{\vw, \vf}}\left[\norm{\hvw - \vw}^2\right] &\geq \bE_{\vw \sim S_{\alpha}(\ve_{1:d-1})}\left[ \bE_{(\vx_i, y_i)_{i \in [n]} \ \iidsim \ P_{\vw, \vf}}\left[\norm{\hvw - \vw}^2\right]\right] \\
    &= \bE_{\vw \sim S_{\alpha}(\ve_{1:d-1})}\left[ \bE_{(\vx_i)_{i \in [n]} \ \iidsim \ \cN(0, \vI)}\left[\bE_{\vy_i \sim \cN(\dotp{\vw}{\vx_i}, {\dotp{\vf}{\vx_i}}^2), i \in [n]}\left[\norm{\hvw - \vw}^2\right]\right]\right]
\end{align*}
For ease of notation, we shall denote $\vx_{1:n}$ to be the shorthand for $(\vx_i)_{i \in [n]} \ \iidsim \ \cN(0, \vI)$. Similarly, we shall  write $y_i \sim \cN(\dotp{\vw}{\vx_i}, {\dotp{\vf}{\vx_i}}^2), i \in [n]$ as $y_{1:n} | \vx_{1:n}$.  Thus,
\begin{align*}
    \sup_{\vw \in \bR^d} \bE_{(\vx_i, y_i)_{i \in [n]} \ \iidsim \ P_{\vw, \vf}}\left[\norm{\hvw - \vw}^2\right] &\geq \bE_{\vw \sim S_{\alpha}(\ve_{1:d-1})}\left[ \bE_{\vx_{1:n}}\left[\bE_{y_{1:n} | \vx_{1:n}}\left[\norm{\hvw - \vw}^2\right]\right]\right] \\
    &= \bE_{\vx_{1:n}} \left[ \bE_{\vw \sim S_{\alpha}(\ve_{1:d-1})}\left[\bE_{y_{1:n} | \vx_{1:n}}\left[\norm{\hvw - \vw}^2\right]\right]\right]
\end{align*}
We now proceed by performing a \emph{data-dependent change of co-ordinates}. In particular, we let $\vvu_1, \dots, \vvu_d$ be an orthonormal basis of $\bR^d$ such that $\vvu_d = \ve_d = \nicefrac{\vf}{\norm{\vf}}$ and $\vvu_1, \dots, \vvu_{d-1}$ are measurable functions of $\vx_1, \dots, \vx_n$, to be specified later. We note that, $\textrm{Span}(\ve_1, \dots, \ve_{d-1}) = \textrm{Span}(\vvu_1, \dots, \vvu_{d-1})$, and hence, $S_{\alpha}(\ve_{1:d-1})$ is equal to $S_{\alpha}(\vvu_{1:d-1})$ as they represent the uniform spherical distribution on the same subspace. Hence,
\begin{align}
\label{eqn:lower-bound-risk-coupling}
    \sup_{\vw \in \bR^d} \bE_{(\vx_i, y_i)_{i \in [n]} \ \iidsim \ P_{\vw, \vf}}\left[\norm{\hvw - \vw}^2\right] &\geq \bE_{\vx_{1:n}} \left[ \bE_{\vw \sim S_{\alpha}(\vvu_{1:d-1})}\left[\bE_{y_{1:n} | \vx_{1:n}}\left[\norm{\hvw - \vw}^2\right]\right]\right] \nonumber \\
    &\geq \bE_{\vx_{1:n}} \left[ \bE_{\vw \sim S_{\alpha}(\vvu_{1:d-1})}\left[\bE_{y_{1:n} | \vx_{1:n}}\left[\sum_{i=1}^{d} {\dotp{\hvw - \vw}{\vvu_i}}^2\right]\right]\right]
\end{align}
We now construct a family of \emph{co-ordinate flip couplings}, i.e., for every $i \in [d]$, let $(\vw, \vwc{i})$ be a coupling such that
\begin{align*}
    \dotp{\vwc{i}}{\vvu_j} &= \begin{cases}
        \dotp{\vw}{\vvu_j}, & \textrm{if}\ j \neq i \\
        -\dotp{\vw}{\vvu_j}, & \textrm{otherwise}
        \end{cases} 
\end{align*}
By the symmetry of the uniform spherical distribution, both $\vw$ and $\vwc{i}$ are distributed identically as $S_{\alpha}(\vvu_{1:d-1})$. Applying Lemma \ref{lem:coupling-lecam} to each summand in 
\eqref{eqn:lower-bound-risk-coupling}, we conclude that,
\begin{align*}
    4R^{\hvw}_{d,n} \geq \bE_{\vx_{1:n}} \left[  \sum_{i=1}^{d-1} \bE_{\vw, \vwc{i}}\left[(\vw - \vwc{i})^2 \left(1 - \sqrt{2 \KL{\Law\left(y_{1:n} \ | \vx_{1:n}, \vw, \vf\right)}{\Law\left(y_{1:n} \ | \vx_{1:n}, \vwc{i}, \vf \right)}}\right)^{+}\right]\right]
\end{align*}
where $R^{\hvw}_{d,n}$ denotes the worst-case risk of the estimator $\hvw$ defined as follows:
\begin{align*}
    \sup_{\vw \in \bR^d} \bE_{(\vx_i, y_i)_{i \in [n]} \ \iidsim \ P_{\vw, \vf}}\left[\norm{\hvw - \vw}^2\right]
\end{align*}
Furthermore, since $\Law\left(y_{1:n} \ | \ \vx_{1:n}, \vw, \vf\right) = \prod_{j=1}^{n} \cN(\vy_i \ | \ \dotp{\vw}{\vx_i}, {\dotp{\vf}{\vx_i}}^2)$, it follows that,
\begin{align*}
    \KL{\Law\left(y_{1:n} \ | \vx_{1:n}, \vw, \vf\right)}{\Law\left(y_{1:n} \ | \vx_{1:n}, \vwc{i}, \vf \right)} &= \frac{{\dotp{\vw - \vwc{i}}{\vvu_i}}^2}{\norm{\vf}^2} \sum_{j=1}^{n} \frac{{\dotp{\vvu_i}{\vx_j}}^2}{{\dotp{\vvu_d}{\vx_j}}^2},
\end{align*}
where we use the fact that $\vf = \norm{\vf}\vvu_d$. We also note that $\dotp{\vw}{\vvu_i} - \dotp{\vwc{i}}{\vvu_i} = 2\dotp{\vw}{\vvu_i}$ almost surely. Furthermore, for ease of notation, denote $\delta_i = \dotp{\vw}{\vvu_i}$. It follows that, 
\begin{align}
    R^{\hvw}_{d,n} \geq \bE_{\vx_{1:n}} \left[ \bE_{\vw \sim S_{\alpha}(\vvu_{1:d-1})} \left[ \sum_{i=1}^{d-1} \delta^2_i \left(1 - \sqrt{4 \frac{\delta^2_i}{\norm{\vf}^2} \sum_{j=1}^{n} \left(\frac{\dotp{\vvu_i}{\vx_j}}{\dotp{\vvu_d}{\vx_j}}\right)^2}\right)^{+} \right] \right] \label{eq:lb_master}
\end{align}

Note that $u_d$ is deterministic and independent of $\vx_{1:n}$. We divide the indices $[n]$ in to the following buckets given $1 \geq \gamma > 0$ and $k\geq 0$. 

$$B_k = \{j \in [n]: \dotp{\vvu_d}{\vx_j}^2 \in [2^{k}\gamma,2^{k+1}\gamma)\}$$
Let $K = \inf \{k : \sum_{l=0}^{k}|B_l|\geq d-1\}$ and $\kmax$ be any integer. Define $\mathcal{E}$ to be the event that the following hold simultaneously:
\begin{enumerate}
    \item $\inf_{j} \dotp{\vvu_d}{\vx_j}^2\geq \gamma$
    \item $\sup_j \dotp{\vvu_d}{\vx_j}^2 \leq \gamma 2^{\kmax}$
    \item $K < \min(\kmax, \bar{C} +2\log_2 d) $
    \item $|B_k| \leq C_B n\sqrt{\gamma}2^{k/2}\log(\kmax) \quad \forall \quad k \leq \kmax$ for some constant $C_B$. 
\end{enumerate}

We pick $\gamma = \frac{c_0}{n^2}$ and $\kmax = \max(C_1 \log n,C_2 + \log_2 d)$ for small enough $c_0$ and large enough constants $C_1,C_2$. Using the fact that the Gaussian density is bounded above by a constant and that $\langle \vvu_d,\vx_j\rangle$ are i.i.d standard Gaussians we conclude that with probability at-least $9/10$, we must have $\bP(\inf_{n}|\dotp{\vvu_d}{\vx_j}|^2\geq \gamma) \geq 9/10$. By an application of Lemma~\ref{lem:close_conc} we conclude that $\bP(\sup_j\langle\vvu_d,\vx_j\rangle^2 \leq \gamma 2^{\kmax}) \geq 9/10$. Furthermore, note that $|B_k| \sim \mathsf{Ber}(n,p_k)$ where $p_k \leq C_2 \sqrt{\gamma}2^{k/2}$. Lemma~\ref{lem:partition_sets} ensures that $\bP(K < \min(\kmax,\bar{C}+2\log_2 d) ) \geq 9/10$ since $\bP(|B_{\min(\kmax,\bar{C}+2\log_2 d)}| \geq d ) \geq \frac{9}{10}$. By an application of Bernstein's inequality, we conclude that $\bP(|B_k| \leq C_B n \sqrt{\gamma}2^{k/2}\log(\kmax)) \geq 1-\frac{1}{10 (\kmax+1)}$ by taking $C_B$ to be a large enough universal constant.

Applying union bound on the above events, we conclude that $\bP(\mathcal{E}) \geq \frac{6}{10}$.  Under the event $\mathcal{E}$, the following holds almost surely
\begin{equation}\label{eq:as_ineq}
    \sum_{j=1}^{n} \left(\frac{\dotp{\vvu_i}{\vx_j}}{\dotp{\vvu_d}{\vx_j}}\right)^2 \leq \sum_{k = 0}^{\kmax} \sum_{j\in B_k}2^{-k} \frac{\dotp{\vvu_i}{\vx_j}^2}{\gamma}
\end{equation}

Let $Q$ be the projector onto the space orthogonal to $\vvu_d$. Then, arguing by the properties of Gaussians, $(Q\vx_j)_{j \in [n]}$ is independent of $(B_k)_{k\geq 0}$.  Let $\tvx_j := Q \vx_j$. We must have have $\dotp{\vvu_i}{\vx_j} = \dotp{\vvu_i}{\tvx_j}$ almost surely whenever $i \neq d$. Therefore, using this in Equation~\eqref{eq:as_ineq}, we conclude that under the event $\mathcal{E}$:
\begin{align}
        \sum_{j=1}^{n} \left(\frac{\dotp{\vvu_i}{\vx_j}}{\dotp{\vvu_d}{\vx_j}}\right)^2 &\leq \sum_{k=0}^{\kmax} \sum_{j\in B_k}2^{-k} \frac{\dotp{\vvu_i}{\tvx_j}^2}{\gamma}
        \nonumber \\
        &= \sum_{k= 0}^{\kmax}\frac{2^{-k}} {\gamma} \vvu_i^{\intercal}H_k \vvu_i \label{eq:binning}
\end{align}

Where $H_k = \sum_{j\in B_k}\tvx_j\tvx_j^{\intercal}$ and we interpret $H_k$ as an operator over $\mathsf{span}(Q)$. Notice that due to the properties of Gaussians, $H_k$ has rank $|B_k|$ almost surely. We are now ready to pick the vectors $\vvu_1,\dots,\vvu_{d-1}$. We define, only in this proof, $N_k := \sum_{l=0}^{k}|B_l|$.

Whenever $ d-1 \geq |B_{0}| > 0$, let
$\vvu_1,\dots,\vvu_{|B_0|}$ be the eigenvectors $H_0$ with non-zero eigenvalue. If $|B_0| > d-1$, then choose $\vvu_1,\dots,\vvu_{d-1}$ to be any arbitrary non-zero eigenvectors of $H_0$.

Similarly, for $K-1\geq k \geq 1$, whenever $|B_k|>0$, we consider $\vvu^{(k)}_1,\dots,\vvu^{(k)}_{|B_k|}$ to be the non-zero eigenvectors of $H_k$. We define $\vvu_{1+N_{k-1}},\dots,\vvu_{N_k}$ to be $\vvu^{(k)}_1,\dots,\vvu^{(k)}_{|B_k|}$, after Gram-Schmidt Orthonormalization with respect to $\vvu_1,\dots,\vvu_{N_k}$ (we skip Orthonormalization step in the event $N_{k-1} = 0$). Note that such vectors exist almost surely since the rank of $\sum_{l=0}^{k}H_l$ is $N_k$ almost surely for every $k < K $. If $N_K = d-1$, we do the same procedure as for $1 \leq k \leq K-1$ as described above. If $ N_K > d-1 $, we take $\vvu^{(K)}_{1},\dots,\vvu^{(K)}_{d-1-N_{K-1}}$ to be any non-zero eigenvectors of $H_K$. We ortho-normalize them as above with respect to $\vvu_1,\dots,\vvu_{N_{K-1}}$ to obtain $\vvu_{N_{K-1} +1},\dots,\vvu_{d-1}$.

Notice that whenever $ 0 \leq k < K$, almost surely 
$\vvu_i^{\intercal}H_k\vvu_i = 0$ whenever $i > N_k$ because we pick such $\vvu_i$ to be orthogonal to $\mathsf{span}(H_k)$. We also note that whenever $i \leq N_{k-1}$, $\vvu_i$ is independent of $H_k$ when conditioned on $(B_l)_{l\geq 0}$. We collect these observations in the following claim:

\begin{claim}
\begin{enumerate}

\item
For $ k < K$ and $ d-1 \geq i > N_k$, $$\vvu_i^{\intercal}H_k \vvu_i = 0 \,.$$
\item For $k \leq K$ and $\min(d-1,N_k) \geq i > N_{k-1}$,  
$$\vvu_i^{\intercal}H_k \vvu_i \leq \|H_k\|\,.$$

\item When conditioned on $(B_l)_{l}$ and $\mathcal{E}$, when either
\begin{enumerate}
    \item  $k \leq K$ and $ 0<i \leq N_{k-1}$
    \item $k > K$ and $ 1\leq i \leq d-1$
\end{enumerate} 
Then, with probability at-least $1-\delta$:

$$ \vvu_i^{\intercal} H_k \vvu_i \leq |B_k| + C\sqrt{|B_k|}\log(\tfrac{1}{\delta})\,.$$

\item When conditioned on $(B_l)_{l}$ and $\mathcal{E}$, with probability at-least $1-\delta$, we must have:
$$\|H_k\| \leq C(|B_k| + d + \log(\tfrac{1}{\delta}))$$
\end{enumerate}
\end{claim}
\begin{proof}
  Items 1 and 2 are easy to show based on the prior discussion. Note that $(\tilde{\vx}_j)_j$ are independent of $(B_l)_l$ and $\mathcal{E}$. Item 3 follows from  Lemma~\ref{lem:chi_squared_conc} after noting that conditioned on $\mathcal{E}$ and $(B_l)_l$, we must have $\vvu_i^{\intercal} H_k \vvu_i$ is a sum of $|B_k|$ i.i.d. $\chi^2$ random variables (since $\vvu_i$ is independent of $H_k$). Item 4 follows from \cite[Theorem 5.29]{vershynin2010introduction}.
\end{proof}

In the discussion below, we condition on $(B_l)_l$ and $\mathcal{E}$. The probabilities are all conditional probabilities. We pick $\kmax > K$ to be a positive integer. By union bound and the claim above, the have with inequalities all hold simultaneously with probability at-least $1-\delta$: 
\begin{enumerate}
    \item For $0 \leq i \leq \min(N_0,d-1)$, $$\vvu_i^{\intercal}H_0\vvu_i \leq C(|B_0| + d + \log(\tfrac{\kmax d}{\delta})) \,.$$ 
    \item For $0 < k\leq K$, $N_{k-1} \leq i \leq \min(N_k,d-1)$, $l< k$ (if such a $k$ exists)
        $$\vvu_i^{\intercal}H_l\vvu_i = 0\,.$$
    \item  For $0 < k\leq K$, $N_{k-1} \leq i \leq \min(N_k,d-1)$ (if such a $k$ exists)
        $$\vvu_i^{\intercal}H_k\vvu_i = C(|B_k| + d + \log(\tfrac{\kmax d}{\delta})) \,.$$   
    \item $ 0< k \leq K$ and $ 0< i \leq N_{k-1}$ (whenever such $k,i$ exist)
    $$ \vvu_i^{\intercal} H_k \vvu_i \leq |B_k| + C\sqrt{|B_k|}\log(\tfrac{\kmax d}{\delta})\,.$$
    \item $ \kmax > k > K$ and $ 1\leq i \leq d-1$

        $$ \vvu_i^{\intercal} H_k \vvu_i \leq |B_k| + C\sqrt{|B_k|}\log(\tfrac{\kmax d}{\delta})\,.$$
\end{enumerate}

Therefore, from the above relationships, we conclude that when conditioned on $\mathcal{E}$, with probability at-least $1-\delta$, the following inequalities hold:

\begin{enumerate}
    \item $1\leq i \leq \min(N_0,d-1)$ whenever such an $i$ exists:
            \begin{align}&\sum_{k= 0}^{\kmax}\frac{2^{-k}}{\gamma}\vvu_i^{\intercal}H_k \vvu_i \lesssim \frac{1}{\gamma}\left[|B_0| + d + \log(\tfrac{\kmax d}{\delta}) + \sum_{k=1}^{\kmax} 2^{-k}(|B_k| + \sqrt{|B_k|}\log(\tfrac{\kmax d}{\delta}))\right]\nonumber \\
            &\lesssim \frac{1}{\gamma}\left[\sqrt{\gamma}n\log(\kmax) + d + \log(\tfrac{\kmax d}{\delta}) + \sum_{k=1}^{\kmax} 2^{-k}(\sqrt{\gamma}n 2^{k/2} \log(\kmax) + \log^2(\tfrac{\kmax d}{\delta}))\right] \nonumber \\
            &\lesssim \frac{n\log(\kmax)}{\sqrt{\gamma}} + \frac{d}{\gamma} + \frac{\log^2(\tfrac{\kmax d}{\delta})}{\gamma} \label{eq:bin_bound_1}
            \end{align}

    In the second step, we have used the bounds on $|B_k|$ in the definition of the event $\mathcal{E}$. 
    \item $N_{k-1} < i \leq \min(N_k,d-1)$, $K\geq k > 0$ whenever such $k,i$ exist:
            \begin{align}
            &\sum_{l= 0}^{\kmax}\frac{2^{-l}}{\gamma}\vvu_i^{\intercal}H_l \vvu_i \lesssim \frac{1}{\gamma}\left[2^{-k}\left(|B_k| + d + \log(\tfrac{\kmax d}{\delta})\right) + \sum_{l= k+1}^{\kmax} 2^{-l}(|B_l| + \sqrt{|B_l|}\log(\tfrac{\kmax d}{\delta}))\right]\nonumber\\
            &\lesssim 2^{-k/2}\frac{n\log(\kmax)}{\sqrt{\gamma}} + 2^{-k}\frac{(d+\log^2(\tfrac{\kmax d}{\delta}))}{\gamma} \label{eq:bin_bound_2}
            \end{align}
\end{enumerate}
Set $\delta = 1/6$. Let us call the event in the which the inequalities given in Equation~\eqref{eq:binning},~\eqref{eq:bin_bound_1} and~\eqref{eq:bin_bound_2} hold as $\mathcal{F}$. Now, $\bP(\mathcal{F}) \geq \bP(\mathcal{F}\cap \mathcal{E}) = \bP(\mathcal{F}|\mathcal{E})\bP(\mathcal{E}) \geq 5/6 \times 6/10 = 1/2$. Notice that the event $\mathcal{F}$ is measurable with respect to the sigma algebra $\sigma(\vx_{1:n})$. 

In Equation~\eqref{eq:lb_master}, consider the following term when $\vx_{1:n}$ is fixed and satisfies the event $\mathcal{F}$:
\begin{align}&\bE_{\vw \sim S_{\alpha}(\vvu_{1:d-1})}\left[ \sum_{i=1}^{d-1} \delta^2_i \left(1 - \sqrt{4 \frac{\delta^2_i}{\norm{\vf}^2} \sum_{j=1}^{n} \left(\frac{\dotp{\vvu_i}{\vx_j}}{\dotp{\vvu_d}{\vx_j}}\right)^2}\right)^{+} \right]\nonumber \\
&\geq \bE_{\vw \sim S_{\alpha}(\vvu_{1:d-1})}\left[ \sum_{i=1}^{d-1} \delta^2_i \left(1 - \sqrt{4 \frac{\delta^2_i}{\norm{\vf}^2} \sum_{j=1}^{n} \left(\frac{\dotp{\vvu_i}{\vx_j}}{\dotp{\vvu_d}{\vx_j}}\right)^2}\right) \right]\nonumber \\
&= \bE_{\vw \sim S_{\alpha}(\vvu_{1:d-1})}\left[ \sum_{i=1}^{d-1} \delta^2_i - \sum_{i=1}^{d-1}|\delta_i|^3\sqrt{\frac{4}{\norm{\vf}^2} \sum_{j=1}^{n} \left(\frac{\dotp{\vvu_i}{\vx_j}}{\dotp{\vvu_d}{\vx_j}}\right)^2} \right] \nonumber\\
&\geq \left[ \alpha^2 - C_3\frac{\alpha^3}{(d)^{\tfrac{3}{2}}}\sum_{i=1}^{d-1}\sqrt{\frac{4}{\norm{\vf}^2} \sum_{j=1}^{n} \left(\frac{\dotp{\vvu_i}{\vx_j}}{\dotp{\vvu_d}{\vx_j}}\right)^2} \right] \label{eq:inner_lb_1}
\end{align}
In the last step, we have used the fact that $\sum_{i=1}^{d-1}\delta_i^2 = \alpha^2$ almost surely and that $\mathbb{E}|\delta_i|^3 \leq \frac{C_3 \alpha^3}{d^{\tfrac{3}{2}}}$ for some universal constant $C_3$ (can be shown easily, see bounds in \cite{brennan2020phase}). Since Equation~\eqref{eq:binning},~\eqref{eq:bin_bound_1} and~\eqref{eq:bin_bound_2}  hold under the event $\mathcal{F}$ by definition, we proceed to bound:

\begin{align}
    &\sum_{i=1}^{d-1}\sqrt{\frac{4}{\norm{\vf}^2} \sum_{j=1}^{n} \left(\frac{\dotp{\vvu_i}{\vx_j}}{\dotp{\vvu_d}{\vx_j}}\right)^2} \leq \sum_{i=1}^{d-1}\sqrt{\frac{4}{\norm{\vf}^2}\sum_{k= 0}^{\kmax}\frac{2^{-k}} {\gamma} \vvu_i^{\intercal}H_k \vvu_i }\nonumber \\
 &\lesssim \min(N_0,d-1)\sqrt{\left[\frac{n\log(\kmax)}{\sqrt{\gamma}}  + \frac{d}{\gamma} + \frac{\log^2(\kmax d)}{\gamma}\right] } \nonumber \\
&\quad + \sum_{k=1}^{K}(\min(N_k,d-1) - N_{k-1}) \sqrt{\left[2^{-k/2}\frac{n\log(\kmax)}{\sqrt{\gamma}} + 2^{-k}\frac{(d+\log^2(\kmax d))}{\gamma}\right]} \nonumber \\
&\lesssim |B_0|\sqrt{\left[\frac{n\log(\kmax)}{\sqrt{\gamma}}  + \frac{d}{\gamma} + \frac{\log^2(\kmax d)}{\gamma}\right] } \nonumber \\
&\quad + \sum_{k=1}^{K}|B_k| \sqrt{\left[2^{-k/2}\frac{n\log(\kmax)}{\sqrt{\gamma}} + 2^{-k}\frac{(d+\log^2(\kmax d))}{\gamma}\right]} \nonumber \\
&\lesssim
n\sqrt{\gamma}\log(\kmax)\sqrt{\left[\frac{n\log(\kmax)}{\sqrt{\gamma}}  + \frac{d}{\gamma} + \frac{\log^2(\kmax d)}{\gamma}\right] } \nonumber \\
&\quad + \sum_{k=1}^{K}2^{k/2}n\sqrt{\gamma}\log(\kmax) \sqrt{\left[2^{-k/2}\frac{n\log(\kmax)}{\sqrt{\gamma}} + 2^{-k}\frac{(d+\log^2(\kmax d))}{\gamma}\right]}
\end{align}

Now, using the fact that $\gamma = \frac{c_0}{n^2}$ in the equation above, we conclude:
\begin{align}
    &\sum_{i=1}^{d-1}\sqrt{\frac{4}{\norm{\vf}^2} \sum_{j=1}^{n} \left(\frac{\dotp{\vvu_i}{\vx_j}}{\dotp{\vvu_d}{\vx_j}}\right)^2} \nonumber \\
&\lesssim n (\log(\kmax))^{\tfrac{3}{2}}2^{\tfrac{K}{4}} + \log(\kmax) (K+1)n \sqrt{d+\log^2(\kmax d)} \nonumber \\
&\lesssim n (\log(\kmax))^{\tfrac{3}{2}}\sqrt{d} + n \sqrt{d+\log^2(\kmax d)} \log(\kmax)\log(d)
\end{align}
In the last step, we have used the fact that $2^{K/4} \lesssim \sqrt{d}$. Plugging the bound above into Equation~\eqref{eq:inner_lb_1}, we conclude that whenever $\vx_{1:n}$ satisfies event $\mathcal{F}$, we must have:

\begin{align}
    &\bE_{\vw \sim S_{\alpha}(\vvu_{1:d-1})}\left[ \sum_{i=1}^{d-1} \delta^2_i \left(1 - \sqrt{4 \frac{\delta^2_i}{\norm{\vf}^2} \sum_{j=1}^{n} \left(\frac{\dotp{\vvu_i}{\vx_j}}{\dotp{\vvu_d}{\vx_j}}\right)^2}\right)^{+} \right] \nonumber \\
&\geq  \alpha^2 - C_4\frac{\alpha^3}{(d)^{\tfrac{3}{2}}}\left[ n\sqrt{d}(\log \kmax )^{\tfrac{3}{2}} + n\sqrt{d}\log (\kmax) \log (d) + n(\log(\kmax)\log(d))^2 \right]
\end{align}

Now, we pick $\alpha = \frac{cd}{n}\left[\frac{1}{(\log \log n)^{3/2} + (\log \log n) \log d + \frac{(\log \log n)^2 (\log d)^2}{\sqrt{d}}}\right]$ for some small enough constant $c$. Then, we conclude that under the event $\mathcal{F}$:

\begin{align}
    &\bE_{\vw \sim S_{\alpha}(\vvu_{1:d-1})}\left[ \sum_{i=1}^{d-1} \delta^2_i \left(1 - \sqrt{4 \frac{\delta^2_i}{\norm{\vf}^2} \sum_{j=1}^{n} \left(\frac{\dotp{\vvu_i}{\vx_j}}{\dotp{\vvu_d}{\vx_j}}\right)^2}\right)^{+} \right] \geq \frac{\alpha^2}{2} \label{eq:almost_final_lb}
\end{align}

Going back to Equation~\eqref{eq:lb_master}, we conclude:

\begin{align}
    &R^{\hvw}_{d,n} \geq  \bE_{\vx_{1:n}} \left[ \bE_{\vw \sim S_{\alpha}(\vvu_{1:d-1})} \left[ \sum_{i=1}^{d-1} \delta^2_i \left(1 - \sqrt{4 \frac{\delta^2_i}{\norm{\vf}^2} \sum_{j=1}^{n} \left(\frac{\dotp{\vvu_i}{\vx_j}}{\dotp{\vvu_d}{\vx_j}}\right)^2}\right)^{+} \right] \right] \nonumber \\
    &\geq  \bE_{\vx_{1:n}} \mathbbm{1}(\mathcal{F})\left[ \bE_{\vw \sim S_{\alpha}(\vvu_{1:d-1})} \left[ \sum_{i=1}^{d-1} \delta^2_i \left(1 - \sqrt{4 \frac{\delta^2_i}{\norm{\vf}^2} \sum_{j=1}^{n} \left(\frac{\dotp{\vvu_i}{\vx_j}}{\dotp{\vvu_d}{\vx_j}}\right)^2}\right)^{+} \right] \right] \nonumber \\
    &\geq \bE_{\vx_{1:n}} \mathbbm{1}(\mathcal{F})\frac{\alpha^2}{2} = \frac{1}{2}\alpha^2\bP(\mathcal{F}) \geq \nicefrac{\alpha^2}{4}
\end{align}

In the second step, we have used Equation~\eqref{eq:lb_master}. Using our choice of $\alpha$, we conclude the result.

\end{proof}
Equipped with Lemmas \ref{lem:lower-bound-coarse} and \ref{lem:lower-bound-fine}, we finally complete the proof of the lower bound as follows:

\subsection{Proof of Theorem~\ref{thm:heteroscedastic-minimax-lb}}
\begin{proof}[Proof of Theorem~\ref{thm:heteroscedastic-minimax-lb}]
Consider a fixed $\vf \in \bR^d$. By definition, for any estimator $\vw$, 
\begin{align*}
    \inf_{\hvw} \sup_{P_{\wstar, \fstar} \in \cP} \bE_{(\vx_i, y_i)_{i \in [n]} \ \iidsim \ P_{\wstar, \fstar}} \left[ \norm{\hvw - \wstar}^2 \right] \geq \inf_{\hvw} \sup_{\wstar \in \bR^d} \bE_{(\vx_i, y_i)_{i \in [n]} \ \iidsim \ P_{\wstar, \vf}} \left[ \norm{\hvw - \wstar}^2 \right]
\end{align*}
From Lemma \ref{lem:lower-bound-coarse} and \ref{lem:lower-bound-fine}, it follows that, 
\begin{align*}
    \inf_{\hvw} \sup_{\wstar \in \bR^d} \bE_{(\vx_i, y_i)_{i \in [n]} \ \iidsim \ P_{\wstar, \vf}} \left[ \norm{\hvw - \wstar}^2 \right] &\geq \Omega\left( \max\left\{ \nicefrac{\norm{\vf}^2}{n}, \nicefrac{\norm{\vf}^2d^2}{n^2} \right\} \right) \\
    &\geq \Omega\left(\nicefrac{\norm{\vf}^2}{n} + \nicefrac{\norm{\vf}^2d^2}{n^2}\right)
\end{align*}
Hence, we obtain the desired lower bound as follows,
\begin{align*}
    \inf_{\hvw} \sup_{P_{\wstar, \fstar} \in \cP} \bE_{(\vx_i, y_i)_{i \in [n]} \ \iidsim \ P_{\vw, \vf}} \left[ \norm{\hvw - \wstar}^2 \right] \geq \Omega\left(\nicefrac{\norm{\vf}^2}{n} + \nicefrac{\norm{\vf}^2 d^2}{n^2} \right) 
\end{align*}
\end{proof}
\end{document}